%% file: neurips_2026.tex
\documentclass{article}

\usepackage[preprint]{neurips_2026}
\usepackage{amsmath}
\usepackage{graphicx}
\usepackage{mathtools}

\usepackage[utf8]{inputenc} 
\usepackage[T1]{fontenc}    
\usepackage{hyperref}       
\usepackage{url}            
\usepackage{booktabs}       
\usepackage{amsfonts}       
\usepackage{nicefrac}       
\usepackage{microtype}      
\usepackage{xcolor}         
\usepackage{amsthm} 
\usepackage{wrapfig}
\usepackage{adjustbox}
\usepackage{algorithm}
\usepackage{algorithmic}

\newtheorem{theorem}{Theorem}[section]
\newtheorem{proposition}[theorem]{Proposition}

\theoremstyle{definition}

\newcommand{\ProblemDefinitionBox}{%
  \begingroup
  \setlength{\fboxsep}{4pt}%
  \setlength{\fboxrule}{0.6pt}%
  \noindent\hfill\fbox{%
    \begin{minipage}{0.97\columnwidth}
      \textbf{Problem setting.}\par
      \textbf{Provided:} models $(f_{\theta},h_{\theta})$, observation data $\{Y_{1:T}^{(n)}\}^{N}_{n=1}$.\par
  \textbf{Goal:} compute the smoother distribution $p_{\theta}(X_{1:T}\!\mid\!Y_{1:T})$ and the best parameters $\theta$ based on models $(f_{\theta},h_{\theta})$.\par
      The trajectory $\{X^{(n)}_{1:T}\}^{N}_{n=1}$ is \textbf{not} available for training.
    \end{minipage}%
  }\hfill\relax
  \endgroup
}

\title{PR-Smoother: Simulator-Preserving Non-Gaussian Smoothing for Data Assimilation}

\author{%
  Yuta Tarumi \\
  iTHEMS\\
  RIKEN\\
  Kobe, Japan \\
  \texttt{yuta.tarumi@riken.jp} \\
}

\begin{document}

\maketitle

\begin{abstract}
Many physical data assimilation (DA) workflows require smoothing methods that represent non-Gaussian posteriors over physical state variables, scale to high-dimensional simulators, train from observation windows alone, and remain compatible with calibration of the prescribed simulator.
We introduce PR-Smoother, a simulator-preserving amortized smoother designed for this prescribed-simulator DA regime.
Its key design principle is to keep the prescribed simulator explicit in both the evidence lower bound and the variational family: rather than learning replacement dynamics or a learned trajectory prior, PR-Smoother learns only future-conditioned corrections around the prescribed rollout.
This yields an explicit non-Gaussian smoothing distribution over physical trajectories and supports joint state, parameter, and sensor-bias learning from observations alone.
The variational family contains the exact smoother in deterministic and linear-Gaussian limits.
Empirically, PR-Smoother captures multimodal posteriors in 4-dimensional Lorenz–96, remains accurate under ambiguous nonlinear observations and process noise in 40-dimensional Lorenz–96, and scales to joint state-parameter-bias inference in 16,384-dimensional Kolmogorov flow.
\end{abstract}

\section{Introduction}

Data assimilation (DA) infers latent physical states from partial, noisy observations under a prescribed dynamical model. We focus on smoothing, which estimates a posterior distribution over a physical trajectory using all observations within a time window. We formulate this task as inference in a state-space model (SSM) with prescribed dynamics and observation operators \citep[e.g.][]{Carassi_DAreview}.
DA underlies reanalysis systems such as ERA5 \citep{ERA5}, which enabled the training of AI-based weather prediction models \citep{FourCastNet,PanguWeather,GraphCast}. DA is increasingly used in related domains, including Venusian atmospheric modeling \citep{VenusDA} and space weather \citep{SpaceWeather1, SpaceWeather2}.

In many physical DA settings, the dynamics and observation operators $(f_\theta,h_\theta)$ are prescribed by scientific knowledge, while their parameters $\theta$ may be unknown. We call inference \emph{simulator-preserving} when these operators remain explicit in the generative model and inference objective, allowing their physically meaningful parameters to be calibrated.
Many DA settings are observation-only: one has access to $Y_{1:T}$, but not to state trajectories $X_{1:T}$ for supervision. If $\theta$ is unknown, surrogate training requires choosing parameter values to generate synthetic trajectories. Inference remains difficult because nonlinear $f_\theta$ and $h_\theta$ induce smoothing posteriors that are often non-Gaussian and sometimes multimodal, while in high dimensions sequential Monte Carlo is often impractical. Operational methods therefore rely on Gaussian analyses or windowed MAP estimates \citep{2010Bocquet}. These approximations can work well in near-linear regimes, but become brittle under many-to-one or strongly nonlinear observations, such as saturating radiance operators or scatterometry \citep{VarQC1,VarQC2,2019Geer_AllSky}. This motivates a physics-preserving alternative that keeps $(f_\theta,h_\theta)$ explicit and differentiable in the objective while using a flexible variational posterior to represent the smoothing distribution.

These considerations define a practically important regime of prescribed-simulator DA. In this regime, inference should simultaneously \textbf{(i) represent non-Gaussian smoothing posteriors in physical space, (ii) scale to high-dimensional systems, (iii) learn from observation windows alone, and (iv) keep the prescribed simulator explicit} so that physically meaningful parameters can be estimated. Taken separately, these goals are familiar. Taken together, they exclude many otherwise successful approaches: methods that preserve the simulator typically rely on Gaussian analyses or windowed MAP estimates, whereas methods with richer learned posteriors often move inference to learned latent states or learned trajectory priors. This paper targets precisely this intersection.

We propose Physics Rollout (PR)-Smoother, a simulator-preserving non-Gaussian smoother for prescribed-simulator DA. Its key design choice is to build inference around the simulator rollout itself: rather than learning replacement dynamics or an auxiliary latent transition model, PR-Smoother learns only future-conditioned corrections around $f_\theta(X_{t-1})$. Concretely, it models $q_\phi(X_1 \mid Y_{1:T})$ with a conditional normalizing flow (NF) and parameterizes each transition $q_\phi(X_t \mid X_{t-1}, Y_{t:T})$ as a future-conditioned Gaussian correction around the simulator rollout. This mirrors the smoother factorization while keeping $(f_\theta,h_\theta)$ explicit in both the evidence lower bound (ELBO) and the variational family, yielding an explicit non-Gaussian density over physical trajectories. 
The observation-only ELBO jointly trains the inference network and calibrates unknown physical and sensor parameters. Inference is \emph{amortized}: after training, a new observation window is processed by an inference-network forward pass and simulator rollout.

\textbf{Our contributions are as follows.}
\begin{itemize}
    \item We introduce PR-Smoother, a simulator-preserving amortized smoother for prescribed-simulator DA. It defines a flexible non-Gaussian density over physical trajectories by combining a normalizing-flow initial-state posterior with simulator-rollout-centered transition kernels. The model is trained from observation windows alone using a variational-inference objective, while the prescribed models $(f_\theta,h_\theta)$ remain explicit in the ELBO, enabling calibration of generative-model parameters.
    \item We show that this variational family is well aligned with the target inference problem: it contains the exact smoother in (i) deterministic dynamics limit and (ii) linear-Gaussian limit with the standard universality assumption on NF.
    \item We validate this design across three experiments: controlled multimodality in 4-dimensional Lorenz--96 (Sec.~\ref{sec:Lorenz96_multimodal}), ambiguous nonlinear observations and process noise in 40-dimensional Lorenz--96 (Sec.~\ref{sec:Lorenz96}), and joint state--parameter--bias inference in a 16{,}384-dimensional Kolmogorov flow from observation windows alone (Sec.~\ref{sec:Kolmogorov}).
\end{itemize}

\section{Problem: Prescribed-Simulator Data Assimilation}
\label{sec:problem-setting}

We consider a Markovian SSM with prescribed, differentiable dynamics and observation operators:
\begin{align}
  X_{t+1} &= f_{\theta_\mathrm{dyn}}(X_t) + \eta_t, \label{eq:dynamics}\\
  Y_t &= h_{\theta_\mathrm{obs}}(X_t) + \xi_t, \label{eq:observation}
\end{align}
where $X_t \in \mathbb{R}^d$ are latent states, $Y_t \in \mathbb{R}^p$ are observations, and $\eta_t$ and $\xi_t$ are process
and observation noise, respectively.
We assume
$\eta_t \sim \mathcal{N}(0, Q)$ and $\xi_t \sim \mathcal{N}(0, R)$ unless stated otherwise.\footnote{More generally, non-Gaussian observation noise can be handled without change, as $p_{\theta_{obs}}(Y_t\!\mid\!X_t)$ appears only in the ELBO and does not constrain $q_{\phi}$.}
The parameters $\theta_{\mathrm{dyn}}$ and $\theta_{\mathrm{obs}}$ collect unknown physical or bias parameters.

With these models, the generative model factorizes as
\begin{align}
    p_{\theta}(X_{1:T},Y_{1:T})=p(X_1)\prod_{t=2}^Tp_{\theta_{dyn}}(X_t\!\mid\! X_{t-1})\prod_{t=1}^Tp_{\theta_{obs}}(Y_t\!\mid\! X_t),
    \label{eq:generative model}
\end{align}
where $p(X_1)$ is a prior over the initial state. DA seeks the filtering posterior $p_\theta(X_{T}\!\mid\! Y_{1:T})$ or smoothing posterior $p_\theta(X_{1:T}\!\mid\! Y_{1:T})$ and, when applicable, estimates of the model parameters $\theta = (\theta_{\mathrm{dyn}}, \theta_{\mathrm{obs}})$.
Because $f_{\theta_{\mathrm{dyn}}}$ and $h_{\theta_{\mathrm{obs}}}$ are nonlinear in general,
this posterior is not available in closed form.

\ProblemDefinitionBox

Within each dataset, physical parameters and persistent sensor biases are shared across observation windows, while each window has its own latent trajectory. Unknown shared parameters are optimized jointly with the inference network in a single training run. The trained smoother is reused across windows from that system; deployment in a different physical regime may require adaptation.

\section{Related Work}
\label{sec:related-methods}

\paragraph{Simulator-preserving DA in physical space.}
A first family keeps the prescribed dynamics and observation operators inside inference. This includes 4D-Var and weak-constraint variants, which solve windowed MAP problems \citep{4DVar1,4DVar2}; ensemble Kalman filters and smoothers, which propagate Gaussian analyses under the prescribed model \citep{EnKF_Evensen,IEnKS}; particle/sequential Monte Carlo methods, which target non-Gaussian posteriors using particle representations \citep{PFF07,2013reich,2019vanLeeuwen_PFReview}; and nonlinear ensemble-transport smoothers, which use transport-map ensemble updates for non-Gaussian smoothing in physical space \citep{Ramgraber23a,Ramgraber23b}. Recent learned model-based DA methods also keep the prescribed model in the loop while learning inverse maps, update rules, or solver components \citep{Fablet21,Beauchamp23,2021frerix,KalmanNet,RTSNet,EKFNet}. In the high-dimensional observation-only settings we study, however, scalable methods in this simulator-preserving regime typically rely on Gaussian analyses, MAP objectives, or learned update components rather than an amortized explicit density over physical trajectories. PR-Smoother instead learns a non-Gaussian smoothing density centered on the prescribed rollout from an observation-only ELBO.

\paragraph{Latent-space DA and latent SSM inference.}
A second family gains posterior flexibility by moving inference to a learned latent representation. Amortized VI methods for latent SSMs jointly learn transition, emission, and recognition models from observations \citep{DVAE_review,DKF,Krishnan2017,DVBF,KVAE,Doerr2018,RKN}. In DA, latent-space assimilation has used autoencoders, latent surrogates, and learned atmospheric representations \citep{Peyron2021,Fan2025LDA}. \citep{DBF} is closest in DA spirit: for nonlinear dynamics it constructs a latent space in addition to the physical space and computes the filtering posterior there. LEVDA \citep{LEVDA} performs per-window ensemble-variational optimization for joint state--parameter estimation using a differentiable latent-dynamics surrogate pretrained on physical state trajectories. PR-Smoother learns an amortized non-Gaussian smoother directly in physical space from observation windows alone, retaining the prescribed dynamics and likelihood in the ELBO. 

\paragraph{Learned-prior and generative DA.}
A third family learns a trajectory prior or transition model from simulated trajectories and conditions on observations at inference time. Score-based DA learns a score model of state trajectories from simulated segments, while FlowDAS learns stepwise transition dynamics using stochastic interpolants \citep{SDA,FlowDAS}. These methods are attractive when the goal is to learn a flexible generative prior from the simulator. Our goal is different: in prescribed-simulator DA, the simulator itself is the generative model to preserve. We therefore keep the simulator on-graph and place flexibility only in the smoothing posterior around the prescribed rollout, aligning it with simulator-preserving DA rather than learned-prior simulation-based inference \citep{SBI_review}.

\paragraph{Relation to generic variational SSM inference.}
At the inference level, PR-Smoother is related to variational filtering and smoothing for nonlinear SSMs and continuous-time models \citep{Frigola2014,Hirt2019,ryder18,ryder21}, sequential VI objectives—including particle-filter bounds, future-likelihood estimators, and twisted smoothing targets—such as FIVO, VSMC, VIFLE, and SIXO \citep{FIVO, Naesseth2018, VIFLE, SIXO}, and to Gaussian variational state / parameter estimation and system-identification methods \citep{Courts2023,Dutra2025}. Normalizing Kalman Filters also combine NFs with state-space structure to obtain flexible non-Gaussian time-series models with tractable Kalman-style inference, but target transformed linear-Gaussian forecasting models rather than simulator-preserving physical-space DA and parameter calibration \citep{NormalizingKalmanFilter}. PR-Smoother contributes a prescribed-simulator DA instantiation of this broader inference line, targeted at the requirements described in the previous section. Contemporaneously, AFSF \citep{AFSF} learns filtering/backward-kernel flows from simulated state--observation trajectories; PR-Smoother trains from observation windows via a prescribed-simulator ELBO.

More broadly, \citet{Campbell2021} develop online variational filtering and parameter learning using reverse-time decompositions $q(X_T|Y_{1:T})\prod_{t=1}^{T-1}q(X_t|X_{t+1},Y_{1:t})$, and \citet{ABVI} analyze additive smoothing error for this backward family. PR-Smoother uses the complementary forward/future-conditioned factorization $q(X_1|Y_{1:T})\prod_{t=1}^{T-1}q(X_{t+1}|X_t,Y_{t+1:T})$, which is tailored to prescribed-simulator DA: it allows an expressive NF posterior over $X_1$ while keeping every transition centered on the physical rollout $f_{\theta}$ and retaining $h_{\theta}$ inside the observation-only ELBO.

\paragraph{Positioning.} PR-Smoother targets prescribed-simulator DA, where $f_{\theta}$ and $h_{\theta}$ are the scientific model to be calibrated rather than a data generator to be replaced. Classical DA methods satisfy this requirement but typically use Gaussian analyses or windowed MAP estimates. Latent SSM and latent DA methods gain posterior flexibility by learning auxiliary latent dynamics, emissions, or recognition models from observations. This can improve tractability, but the calibrated quantities are no longer the prescribed physical operators $(f_{\theta}, h_{\theta})$ themselves. Learned-prior DA and SBI-style methods train trajectory priors or transition models from simulated states and condition on observations afterward. PR-Smoother contributes a prescribed-simulator DA instantiation of this broader inference line, providing flexible yet scalable smoothing posteriors via normalizing flows, without giving up the prescribed models $(f_\theta,h_\theta)$. 

\section{Method}
\label{sec:method in a nutshell}

PR-Smoother targets prescribed-simulator DA by keeping the original SSM explicit in both the ELBO and the variational family, and learning only future-conditioned corrections around the prescribed rollout. We model $q_{\phi}(X_1\!\mid\! Y_{1:T})$ with a conditional NF and $q_{\phi}(X_t\!\mid\!X_{t-1},Y_{t:T})$ with future-conditioned Gaussian transitions centered at $f_{\theta_{\mathrm{dyn}}}(X_{t-1})$, yielding a non-Gaussian smoother in physical state space and enabling learning of $\theta$ from observation windows alone.

\subsection{Objective}

For an observation sequence $Y_{1:T}$ we define the ELBO:
\begin{align}
    &\mathcal{L}(\theta,\phi;Y_{1:T}) =
    \int q_{\phi}(X_{1:T}\!\mid\! Y_{1:T})\log \biggl[\frac{p_{\theta}(X_{1:T},Y_{1:T})}{q_{\phi}(X_{1:T}\!\mid\! Y_{1:T})}\biggr]dX_{1:T},
    \label{eq:elbo-high-level}
\end{align}
which lower-bounds $\log p_\theta(Y_{1:T})$.
Given $\{Y_{1:T}^{(n)}\}_{n=1}^N$, we maximize $\frac{1}{N}\sum_{n=1}^N \mathcal{L}(\theta,\phi;Y_{1:T}^{(n)})$ using minibatches, training from observation windows alone under the prescribed SSM. Maximizing this ELBO jointly learns the amortized smoother and point
estimates of the unknown shared physical and sensor parameters. 

\subsection{Simulator-Preserving Smoother Construction and ELBO}
\label{sec: variational family}

We describe two nested simulator-preserving variational families. The deterministic case uses the prescribed rollout directly, so the transition terms cancel in the ELBO (Eq.~\ref{eq:elbo-det}). The noisy case keeps the same rollout-centered structure but adds future-conditioned Gaussian corrections, matching the smoother conditional factorization (Eqs.~\ref{eq:Markov-decompose}-\ref{eq:q-corr}).

\subsubsection{Initial-Time Posterior $q_{\phi}(X_1\!\mid\!Y_{1:T})$}
To enable flexible modeling of the variational distribution, we use a conditional NF for the initial-time posterior $q_{\phi}(X_1\!\mid\!Y_{1:T})$. An encoder $c=E_{\phi}(Y_{1:T})$ summarizes observations over $T$ steps. With $z\sim\mathcal{N}(0,I)$ and a conditional NF $g_{\phi}$,
\begin{align}
X_1&=g_{\phi}(z;\,c),\qquad
\log q_{\phi}(X_1\!\mid\!Y_{1:T})= \log \mathcal{N}(z;0,I) - \log\!\left|\det \frac{\partial g_{\phi}}{\partial z}\right|.
\label{eq:condflow}
\end{align}
Eq.~\eqref{eq:condflow} provides a computationally efficient way to calculate the KL term in Eq.~\eqref{eq:elbo-det}. We use a conditional NF (RealNVP; \citealt{RealNVP}) rather than diffusion or flow-matching models mainly because it gives direct access to $\log q_{\phi}$, which we need in the ELBO; see also Appendix~\ref{sec:anticausal encoder}.

\subsubsection{Deterministic Dynamics Model}
\label{sec:deterministic}
When process noise is negligible, we introduce a vanishing-noise transition kernel
\begin{align}
p_{\theta_{\mathrm{dyn}},Q}(X_{t}\!\mid\! X_{t-1})
=\mathcal{N}\!\big(X_{t};\,f_{\theta_{\mathrm{dyn}}}(X_{t-1}),\,Q\big),
\label{eq:det-transition-Q}
\end{align}
and construct the variational posterior by pairing this kernel with the distribution for the first step:
\begin{align}
q_{\phi}(X_{1:T}\!\mid\! Y_{1:T})
= q_{\phi}(X_1\!\mid\! Y_{1:T}) \prod_{t=2}^{T} p_{\theta_{\mathrm{dyn}},Q}(X_{t}\!\mid\! X_{t-1}).
\label{eq:q-factorization-Q}
\end{align}
Because the transition factors in \eqref{eq:det-transition-Q} appear identically in both the joint distribution $p_{\theta}(X_{1:T},Y_{1:T})$ (Eq.~\ref{eq:generative model}) and $q_{\phi}$, the log-transition terms cancel in the ELBO for any  $Q>0$. 
Taking $Q\!\to\!0$ yields the deterministic rollout
\begin{align}
X_1 \sim q_{\phi}(X_1 \!\mid\! Y_{1:T}),
X_{t+1}=f_{\theta_{\mathrm{dyn}}}(X_t)\; (1\leq t\leq T{-}1)
\label{eq:rollout-det}
\end{align}
so $q_\phi(X_{1:T}\mid Y_{1:T})$ is the pushforward of $q_\phi(X_1\mid Y_{1:T})$
through $f_{\theta_{\rm dyn}}$, and the limiting ELBO is
\begin{align}
\mathcal{L}_{\mathrm{det}}(Y_{1:T},\theta,\phi) = \int q_{\phi}(X_{1:T} \!\mid\! Y_{1:T})\Big[\sum_{t=1}^{T}\log p_{\theta_{\mathrm{obs}}}(Y_t \!\mid\! X_t)\Big]dX_{1:T} - \mathrm{KL}\!\big(q_{\phi}(X_1 \!\mid\! Y_{1:T})\,\|\,p(X_1)\big).
\label{eq:elbo-det}
\end{align}

\subsubsection{Noisy Dynamics Model}
\label{sec:noisy}
With non-negligible process noise, we use a Markov per-step extension.
We first factorize the posterior in the forward direction:
\begin{align}
    p_{\theta}(X_{1:T}\!\mid\! Y_{1:T}) = p_{\theta}(X_1\!\mid\!Y_{1:T})\prod_{t=2}^T p_{\theta}(X_t\!\mid\! X_{1:t-1},Y_{1:T}).
    \label{eq:forward-factorize}
\end{align}
By the Markov property of the SSM, 
\begin{align}
    p_{\theta}(X_t\!\mid\! X_{1:t-1},Y_{1:T})=p_{\theta}(X_t\!\mid\! X_{t-1},Y_{t:T}).
    \label{eq:Markov-decompose}
\end{align}
We mirror this factorization by augmenting $q_{\phi}$ with future-aware noise terms:
\begin{align}
q_{\phi}(X_{1:T}\!\mid\!Y_{1:T})
= q_{\phi}(X_1\!\mid\!Y_{1:T})\prod_{t=2}^{T} q_{\phi}\!\big(X_{t}\!\mid\!X_{t-1}, Y_{t:T}\big),
\label{eq:q-markov} \\
q_{\phi}\!\big(X_{t}\!\mid\!X_{t-1}, Y_{t:T}\big)
= \mathcal{N}\!\big(f_{\theta_{\mathrm{dyn}}}(X_{t-1}){+}m_t,\,S_t\big),
\label{eq:q-corr}
\end{align}
where corrections $(m_t,S_t)$ depend on the previous state $X_{t-1}$ and future observations $Y_{t:T}$. 
Transitions no longer cancel; the ELBO for this family is

\begin{align}
\mathcal{L}_{noisy}(Y_{1:T},\theta,\phi) &= \int q_{\phi}(X_{1:T} \!\mid\! Y_{1:T})\Big[\sum_{t=1}^{T}\log p_{\theta_{\mathrm{obs}}}(Y_t \!\mid\! X_t)\Big]dX_{1:T} - \mathrm{KL}\!\big(q_{\phi}(X_1 \!\mid\! Y_{1:T})\,\|\,p(X_1)\big) \nonumber\\
&+\sum_{t=2}^{T}\int q_{\phi}(X_{1:T}\!\mid\!Y_{1:T})\log\biggl[\frac{p_{\theta_{\mathrm{dyn}}}(X_{t}\!\mid\!X_{t-1})}{q_{\phi}(X_{t} \!\mid\! X_{t-1},Y_{t:T})}\biggl]dX_{1:T}.
\label{eq:elbo-noisy}
\end{align}

Setting $m_t=0$ and $S_t=Q$ recovers the deterministic variational family of Eq.~\eqref{eq:q-factorization-Q}, 
so the noisy-dynamics variant strictly generalizes the deterministic family by adding future-aware Gaussian corrections around $f_{\theta_{dyn}}(X_{t-1})$.
For Gaussian transitions Eqs.~\eqref{eq:det-transition-Q} and \eqref{eq:q-corr}, each per-step KL is closed-form, encouraging $(m_t,S_t)$ to match the process noise in the dynamics model. 
Importantly, these corrections parametrize only the variational posterior; the generative transition $p_{\theta_{\mathrm{dyn}}}(X_t\!\mid\!X_{t-1})$ and therefore the prescribed simulator are used unchanged inside the ELBO \eqref{eq:elbo-noisy}.

\subsection{Consistency Checks in Tractable Regimes}
\label{sec:standpoint}
Here, we verify that the proposed variational family is posterior-inclusive in two limiting cases.

\begin{proposition}[Posterior inclusiveness in two limits]
\label{prop:posterior-inclusive}
Assume (i) the conditional flow $q_{\phi}(X_1\!\mid\!Y_{1:T})$ can represent
any smooth density over $X_1$, and (ii) in the noisy-dynamics variant, the
Gaussian corrections $q_{\phi}(X_t\!\mid\!X_{t-1},Y_{t:T})$ can realize any
mean and covariance of the form in Eq.~\ref{eq:q-corr}. Then, for the SSM in
Eqs.~\eqref{eq:dynamics}-\eqref{eq:generative model}, the exact smoothing posterior
$p_{\theta}(X_{1:T}\!\mid\!Y_{1:T})$ lies in our variational family in two
cases: (a) deterministic dynamics $X_{t+1}=f_{\theta}(X_t)$ and (b)
linear-Gaussian SSMs. Hence there
exists $\phi^\star$ with
$q_{\phi^\star}(X_{1:T}\!\mid\!Y_{1:T}) = p_{\theta}(X_{1:T}\!\mid\!Y_{1:T})$ for these cases.
\end{proposition}

\noindent\textbf{Proof sketch.}
In the deterministic case the trajectory is a function of $X_1$, so the
smoother is the pushforward of $p_{\theta}(X_1\!\mid\!Y_{1:T})$; by (i) the
flow matches this marginal and, since both models roll out
$X_{t+1}=f_\theta(X_t)$, we obtain $q_{\phi}=p_{\theta}$. 
In the linear-Gaussian case, Kalman-RTS smoothing yields Gaussian conditionals $p_{\theta}(X_t \mid X_{t-1}, Y_{t:T})$ in the future-conditioned factorization of Eq.~\ref{eq:forward-factorize}, of the form in Eq.~\ref{eq:q-corr}; by (ii) the corrections in $q_{\phi}$ can reproduce these, yielding $q_{\phi}=p_{\theta}$. Full derivations appear in Appendix~\ref{sec: proof for variational family including true posterior}.

\section{Experiments}
\label{sec:experiments}

\paragraph{Experimental design.} We test three claims: expressive non-Gaussian smoothing in 4D Lorenz–96 (Sec.~\ref{sec:Lorenz96_multimodal}), robustness under ambiguous nonlinear observations and process noise in 40D Lorenz–96 (Sec.~\ref{sec:Lorenz96}), and scalable joint state–parameter–bias inference in 16,384-D Kolmogorov flow (Sec.~\ref{sec:Kolmogorov}). Training uses only $Y_{1:T}$; implementation and runtime details are in Appendix~\ref{sec:experimental_settings, appendix}. Unless otherwise stated, all reported $\pm$ values denote mean $\pm$ standard deviation over five independent random seeds. Offline training takes approximately 1 GPU-hour for 4D Lorenz–96, 11–18 GPU-hours for 40D Lorenz–96, and 19–20 GPU-hours for Kolmogorov flow. After training, our Flow implementation processes a new window in 0.37–0.38 s for 40D Lorenz–96 and 1.89 s for Kolmogorov flow; full timing protocols and baseline comparisons are given in Appendix~\ref{app:runtime}. Code is available at \url{https://github.com/Yuta-Tarumi/PRSmoother_Neurips2026}.

\subsection{Multimodality Test: 4-Dimensional Lorenz-96}
\label{sec:Lorenz96_multimodal}

\begin{wrapfigure}{r}[0pt]{0.25\linewidth}
\vspace{-5mm}
 \centering
   \includegraphics[width=\linewidth]{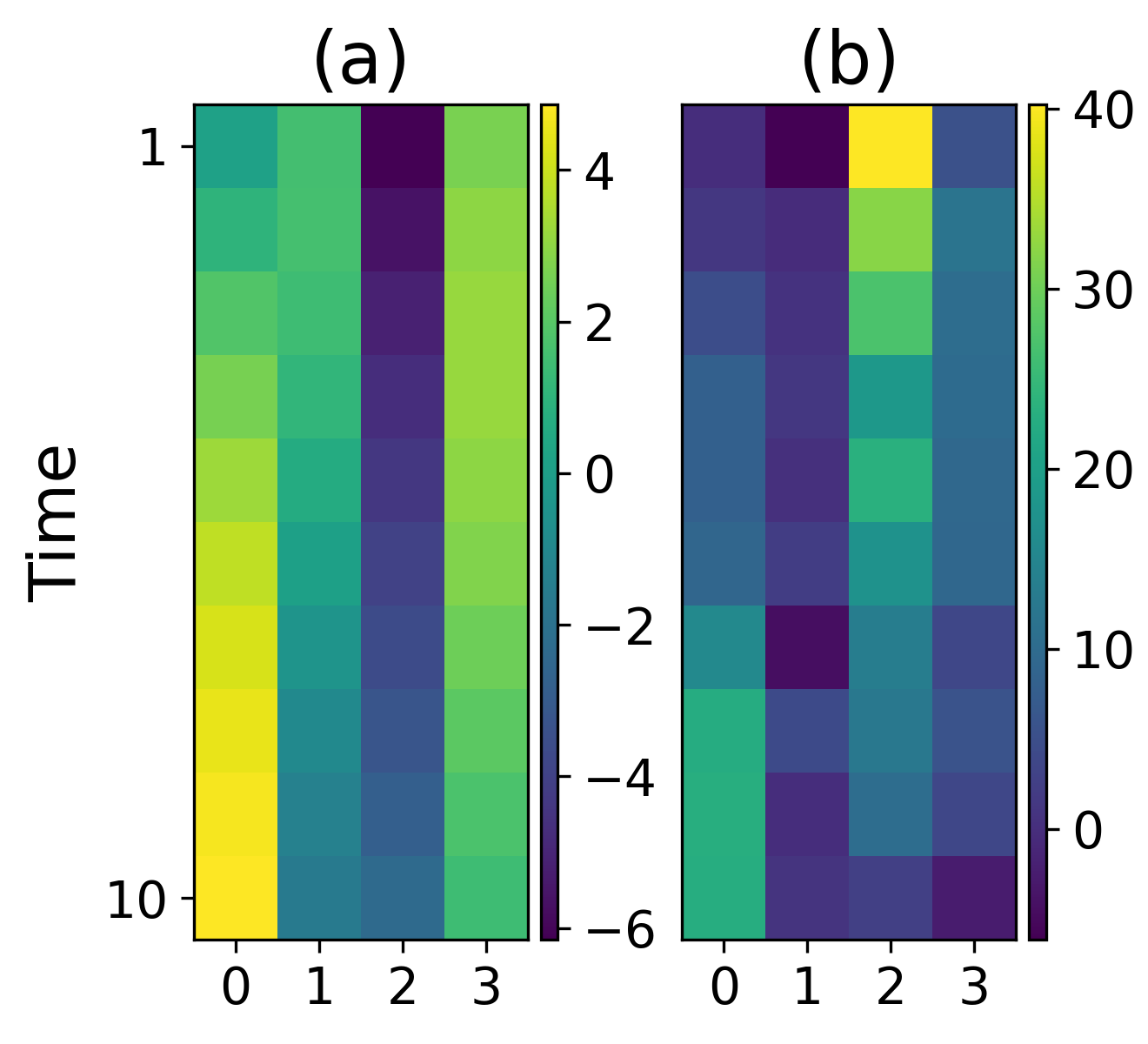}
 \caption{(a) True state $X_{1:T}$ (heatmap over time and space). (b) Observations under $h(x)=x^2$ with additive noise $\sigma=3$.}
 \label{fig:Lorenz96, 4D, example}
\end{wrapfigure}

The Lorenz-96 system is defined on a periodic ring of $N$ variables with time evolution
\begin{equation}
\frac{dx_i}{dt} = (x_{i+1}-x_{i-2})\,x_{i-1} - x_i + F, \ i=1,\dots,N
\label{eq:Lorenz96}
\end{equation}

with cyclic indexing. We take $(N,F)=(4,8)$ for this experiment. Observations are componentwise squares with additive Gaussian noise, 
\begin{align*}
    Y_t &= X_t^2 + \varepsilon_t,  \\
    \varepsilon_t &\sim \mathcal{N}(0,\sigma^2 I),\ \sigma=3,
\end{align*}
so that $x$ and $-x$ are indistinguishable at a single time point. As a result, the true posterior for one-step smoothing is at most $2^4$-modal. Using the same simulated paths, we study a 1‑step smoother $p(X_{10}\!\mid\! Y_{10})$ and a 10-step smoother $p(X_{10}\!\mid\! Y_{1:10})$, where the dynamics help resolve the sign ambiguity.

We instantiate $q_\phi(X_1 \!\mid\! Y_{1:T})$ with two models: a flow and a diagonal Gaussian. We further compare against a mean-field baseline (MF) that factorizes $q_{\phi}(X_{1:T}\!\mid\!Y_{1:T})$ over time and space (see Appendix~\ref{sec:experimental_settings, appendix} for the ELBO), and classical ensemble Kalman methods (EnKF; \citealt{EnKF_Evensen}, ETKF; \citealt{ETKF}). As a reference “truth,” we run a bootstrap PF with $5\times 10^6$ particles to represent $p(X_T \!\mid\! Y_{1:T})$. We use deterministic PR‑Smoother in this experiment; adding process noise would not change the multimodal structure qualitatively and would introduce extra complexity. Noisy variants are evaluated in Sections \ref{sec:Lorenz96}--\ref{sec:Kolmogorov}.

\begin{figure*}[htbp]
\centering
\makebox[\textwidth][c]{%
  \begin{minipage}[t]{0.23\textwidth}
    
  \centering

    \vspace{2pt}

    \includegraphics[width=0.99\linewidth]{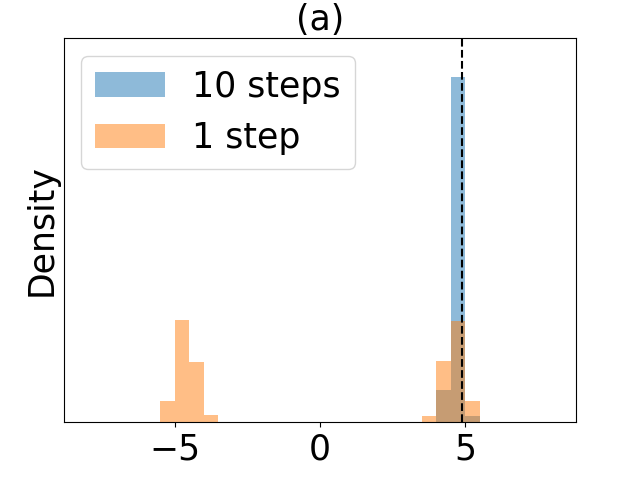}\hfill
    \includegraphics[width=0.99\linewidth]{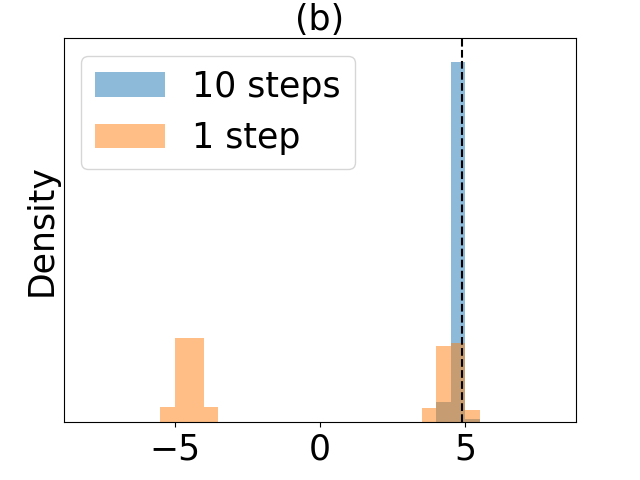}
  \end{minipage}%
  \hspace{0.02\textwidth}
  \begin{minipage}[t]{0.72\textwidth}
    \vspace{10mm}
    \centering
    \small
    \setlength{\tabcolsep}{4pt}\renewcommand{\arraystretch}{1.15}
    \begin{tabular}{lcccc}
      \toprule
      & \multicolumn{2}{c}{1-step} & \multicolumn{2}{c}{10-step} \\
      \cmidrule(lr){2-3}\cmidrule(lr){4-5}
      Method & ED & $W_1$ & ED & $W_1$ \\
      \midrule
      flow & \textbf{0.20 $\pm$ 0.04} & \textbf{0.90 $\pm$ 0.09} & \textbf{0.13 $\pm$ 0.04} & \textbf{0.22 $\pm$ 0.02} \\
      Gauss   & 1.38 $\pm$ 0.24 & 3.63 $\pm$ 0.67 & 0.19 $\pm$ 0.02 & 0.29 $\pm$ 0.01 \\
      MF      & 1.37 $\pm$ 0.24 & 3.62 $\pm$ 0.67 & 0.57 $\pm$ 0.05 & 0.68 $\pm$ 0.07 \\
      EnKF    & 0.70 $\pm$ 0.09 & 2.63 $\pm$ 0.31 & 1.42 $\pm$ 0.26 & 2.57 $\pm$ 0.58 \\
      ETKF    & 0.70 $\pm$ 0.09 & 2.62 $\pm$ 0.32 & 1.50 $\pm$ 0.25 & 2.90 $\pm$ 0.56 \\
      \bottomrule
    \end{tabular}
  \end{minipage}%
}
\caption{Results for the Lorenz-96 experiments with $(N,F)=(4,8)$ and $h(x)=x^2$, $\sigma=3$. Panel (a): posterior distributions $p_{true}(X_{10}\!\mid\! Y_{10})$ (``1 step'') and $p_{true}(X_{10}\!\mid\! Y_{1:10})$ (``10 steps'') obtained by PF. Panel (b): variational distributions $q_{\phi}(X_{10}\!\mid\! Y_{10})$ (``1 step'') and $q_{\phi}(X_{10}\!\mid\! Y_{1:10})$ (``10 steps'') by our flow model. The table on the right shows quantitative comparisons of algorithms.}
\label{fig:Lorenz96, 4D}
\end{figure*}

Fig.~\ref{fig:Lorenz96, 4D} compares PF posteriors with PR-Smoother approximations. The one-step PF posterior is bimodal due to the sign ambiguity of $h(x)=x^2$; conditioning on $Y_{1:10}$ resolves this ambiguity and yields a sharply concentrated distribution. The flow-based $q_\phi$ reproduces the same transition from multimodal to concentrated. Quantitatively, it shows the smallest energy distance (ED) and 1-Wasserstein ($W_1$) distance to the PF reference (Fig.~\ref{fig:Lorenz96, 4D}, right). Posterior histograms for each method are in Appendix~\ref{sec:Lorenz96, inference examples, appendix}.

\textbf{Takeaways.}
(1) When information is insufficient (one step), flow-based $q_\phi$ \emph{represents} the required bimodality instead of collapsing to a Gaussian; (2) as information accumulates (ten steps), PR-Smoother \emph{uses dynamics to correctly disambiguate} and narrows down the variational distribution.

\subsection{Nonlinear Observation Operator and Dynamics Noise Test: 40-Dimensional Lorenz-96}
\label{sec:Lorenz96}
We consider Lorenz-96 with $(N,F)=(40,8)$ in three settings: linear observations $h(x)=x$, nonlinear observations $h(x)=\min(x^2,5)$, and the same nonlinear observation operator with Gaussian dynamics noise $\sigma=0.1$ per step. The observation window is set to 50 steps. Observations have additive Gaussian noise $\sigma=1$. We compare four PR-Smoother variants (Flow/Gauss, with/without noisy-transition corrections) against 4D-Var (10 restarts), IEnKS (\citealt{IEnKS}; 512 members), and a MF VI baseline (see Appendix~\ref{sec:L96, 40D, appendix}). We report RMSE averaged over state variables and time steps. Noisy-transition variants use isotropic learned scales ($S_t=\sigma_S^2 I$, $Q=\sigma_Q^2 I$); see Appendix~\ref{sec:experimental_settings, appendix} for details.

\begin{figure*}[htbp]
\centering
\begin{minipage}[t]{0.65\textwidth}
  \centering

  \vspace{1.6\baselineskip}

  {\small
  \setlength{\tabcolsep}{4pt}
  \renewcommand{\arraystretch}{1.15}
  \begin{tabular}{lccc}
    \toprule
    Method & Linear & Nonlinear & Nonlinear, noisy dyn \\
    \midrule
    Flow & $0.516\pm0.006$ & $0.738\pm0.011$ & $0.909\pm0.015$ \\
    Noisyflow & $0.287\pm0.001$ & {\bf 0.549 $\pm$ 0.084} & {\bf 0.498 $\pm$ 0.003} \\
    Gauss & $0.597\pm0.027$ & $0.827\pm0.009$ & $1.020\pm0.014$ \\
    NoisyGauss & $0.343\pm0.002$ & $0.797\pm0.019$ & $0.958\pm0.006$ \\
    MF & $0.413\pm0.001$ & $3.170\pm0.004$ & $3.019\pm0.030$ \\
    4D-Var & {\bf 0.142 $\pm$ 0.004} & $6.488\pm0.202$ & $3.516\pm0.024$ \\
    IEnKS & $0.176\pm0.004$ & $5.675\pm0.175$ & $5.603\pm0.177$ \\
    \bottomrule
    \end{tabular}
  }
\end{minipage}\hfill
\begin{minipage}[t]{0.34\textwidth}
  \centering
  \adjustbox{valign=t}{\includegraphics[width=\linewidth]{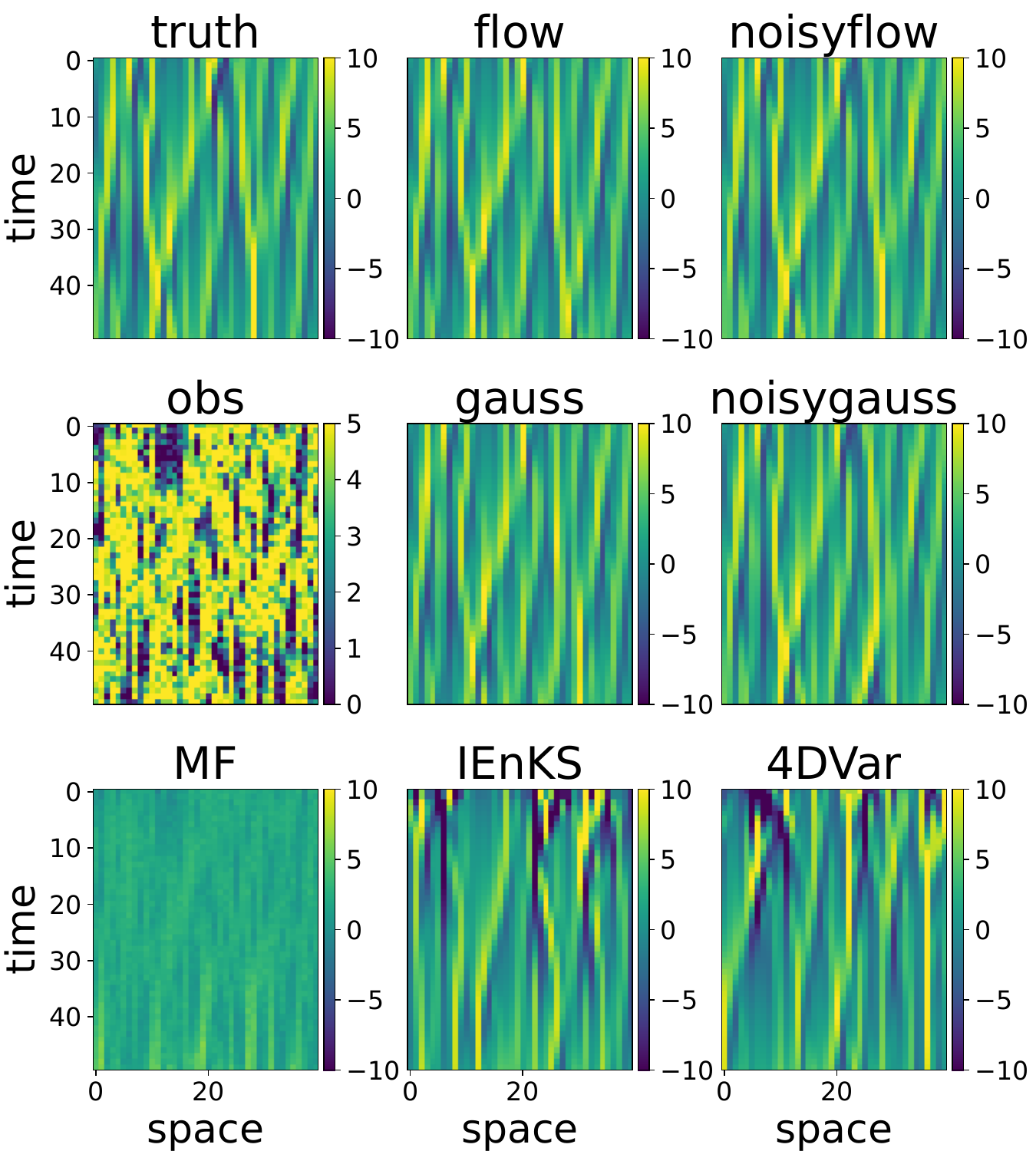}}
\end{minipage}

\caption{Lorenz-96 experiment results. The figure on the right shows reconstructions by methods. The table shows RMSE of samples.}
\label{fig:l96_capsule}
\end{figure*}

Fig.~\ref{fig:l96_capsule} reports RMSEs. With linear observations, all methods are stable and IEnKS/4D-Var perform best; the posterior is close to Gaussian, leaving limited headroom for PR-Smoother. Closing the remaining gap would likely require substantial additional training. With nonlinear $h(x)=\min(x^2,5)$, the observation map is saturating and many-to-one, so single-time data are ambiguous and information must be propagated through time. Gaussian and MF baselines then degrade: IEnKS' local-Gaussian approximation can lose effective temporal correlations, 4D-Var's MAP objective becomes highly nonconvex and initialization-sensitive \citep{ECMWF_4DVar_difficulty, 2018Bonavita}, and MF's factorization precludes trajectory-level dependence.

PR-Smoother couples time steps through the prescribed dynamics and performs well under the many-to-one observations. 
Among PR-Smoother variants, Noisyflow achieves the lowest RMSE across all settings.
In the noisy-dynamics dataset (true per-step process noise $\sigma_Q^{\text{True}}=0.1$), Noisyflow also recovers the noise level at data generation ($\ln\sigma_Q=-2.29\pm0.04$ vs.\ $\ln 0.1=-2.30$; Appendix~\ref{sec:sigmaq_appendix}, Table~\ref{tab:sigma_Q}).
Across regimes, Flow outperforms Gauss and Noisyflow outperforms NoisyGauss; moreover, noisy-transition variants improve on their noiseless counterparts even when the data are generated with deterministic dynamics, suggesting that the added flexibility can ease optimization.

\textbf{Takeaways.}
(i) With informative, near-linear observations, classical methods perform best; PR-Smoother offers limited gains without substantial training. (ii) Under ambiguous nonlinear observations, trajectory-level correlations matter: IEnKS and 4D-Var can diverge, whereas PR-Smoother's dynamics-coupled posterior remains accurate. (iii) Within PR-Smoother, a more flexible variational family (Flow vs.\ Gauss, especially Noisyflow) further reduces RMSE while still respecting the prescribed physics.

\subsection{High State Dimension and Model Parameter Training Test: Kolmogorov Flow}
\label{sec:Kolmogorov}

We consider a 16,384-dimensional Kolmogorov flow (2-D incompressible Navier–Stokes with sinusoidal forcing; Appendix~\ref{sec:Kolmogorov, appendix}), on a $128\times128$ grid. We test four observation operators: full (all grid points), half (upper half), sparse (subsampled points), and realistic (mixed dense/sparse coverage). The forcing amplitude $F$ and Reynolds number $Re$ are unknown (true $(F,Re)=(1.0,1000)$; init $(0.1,100)$). The observation window is set to 10 steps. Observations include Gaussian noise ($\sigma=3.0$) and an additive, time-independent bias field $b$—drawn once at data generation with $b_i\sim\mathcal{N}(0,3^2)$ and shared across all windows (sensor calibration)—yielding 16,386 parameters under “full” observations. 

\begin{figure*}[htbp]
\centering
\begin{minipage}[t]{0.4146\linewidth}
  \vspace{0pt}\centering
  \includegraphics[width=\linewidth]{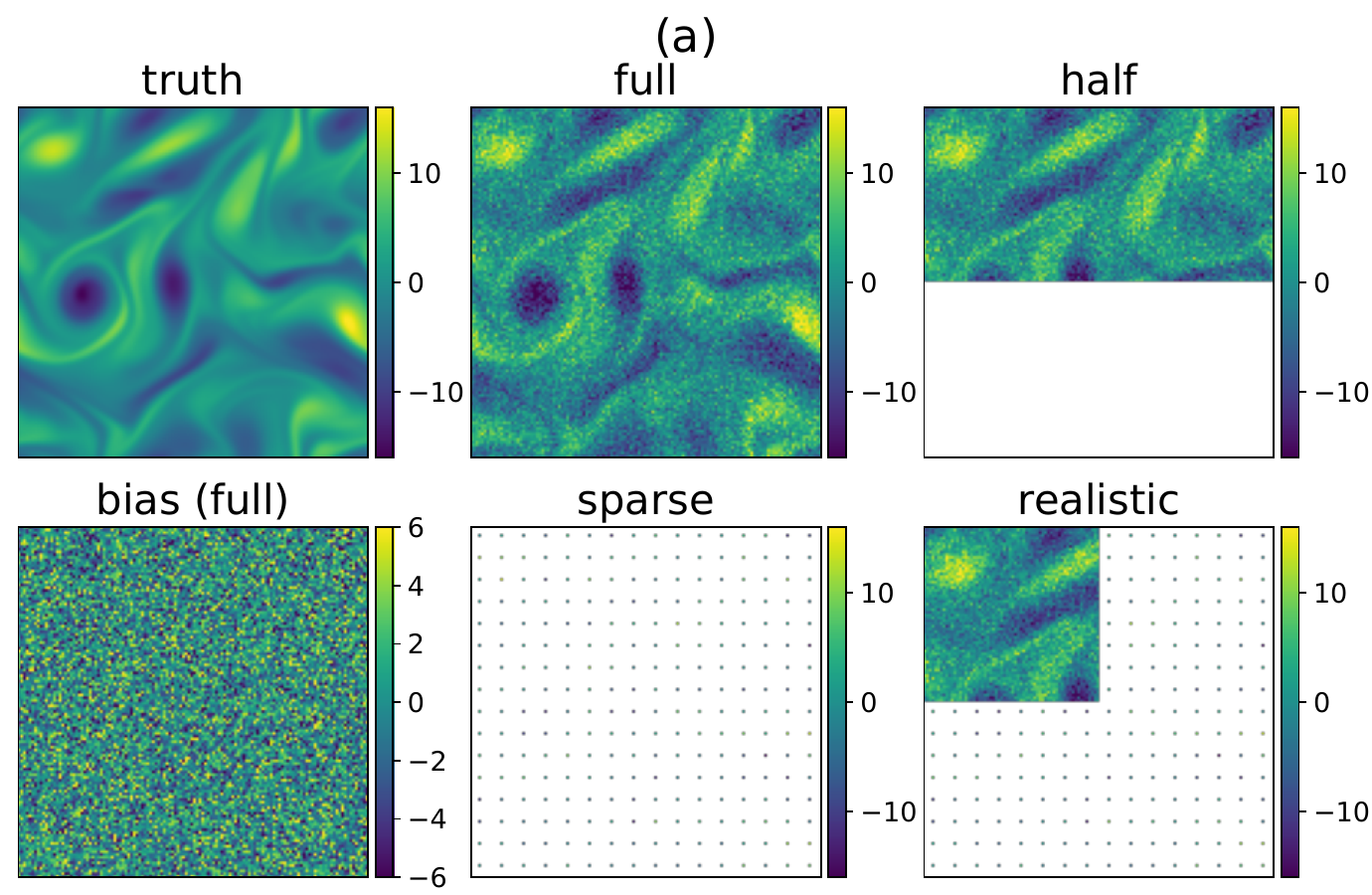}
\end{minipage}
\begin{minipage}[t]{0.27\linewidth}
  \vspace{0pt}\centering
  \includegraphics[width=\linewidth]{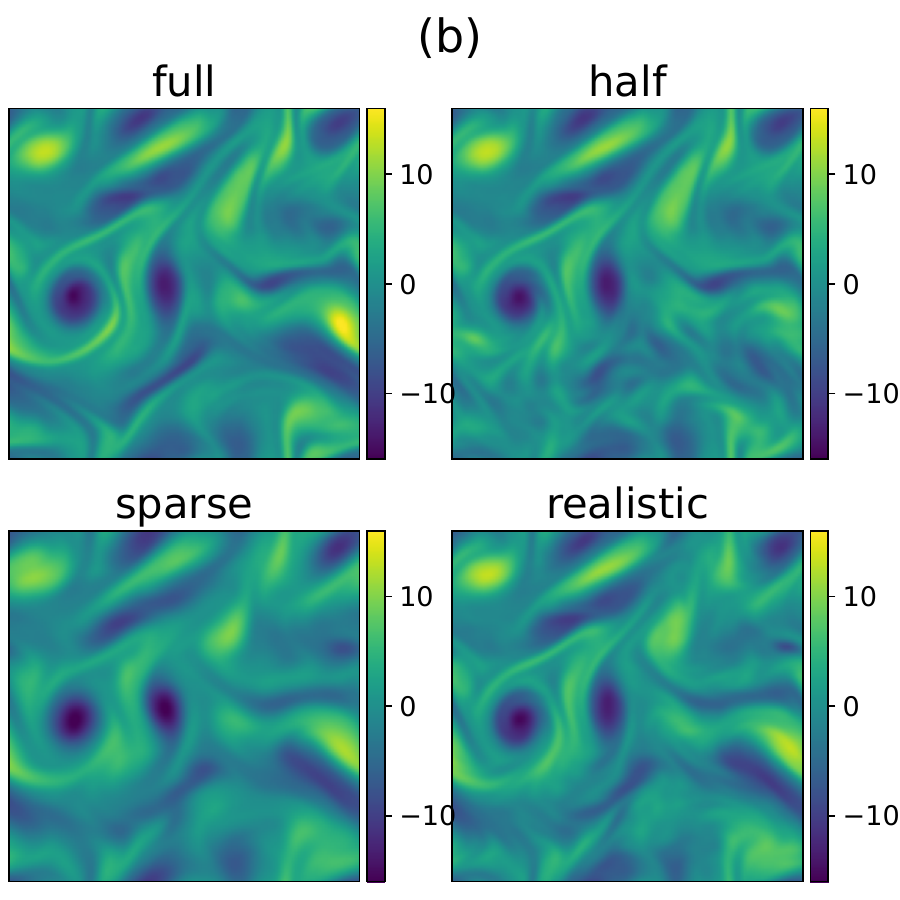}
\end{minipage}
\begin{minipage}[t]{0.27\linewidth}
  \vspace{0pt}\centering
  \includegraphics[width=\linewidth]{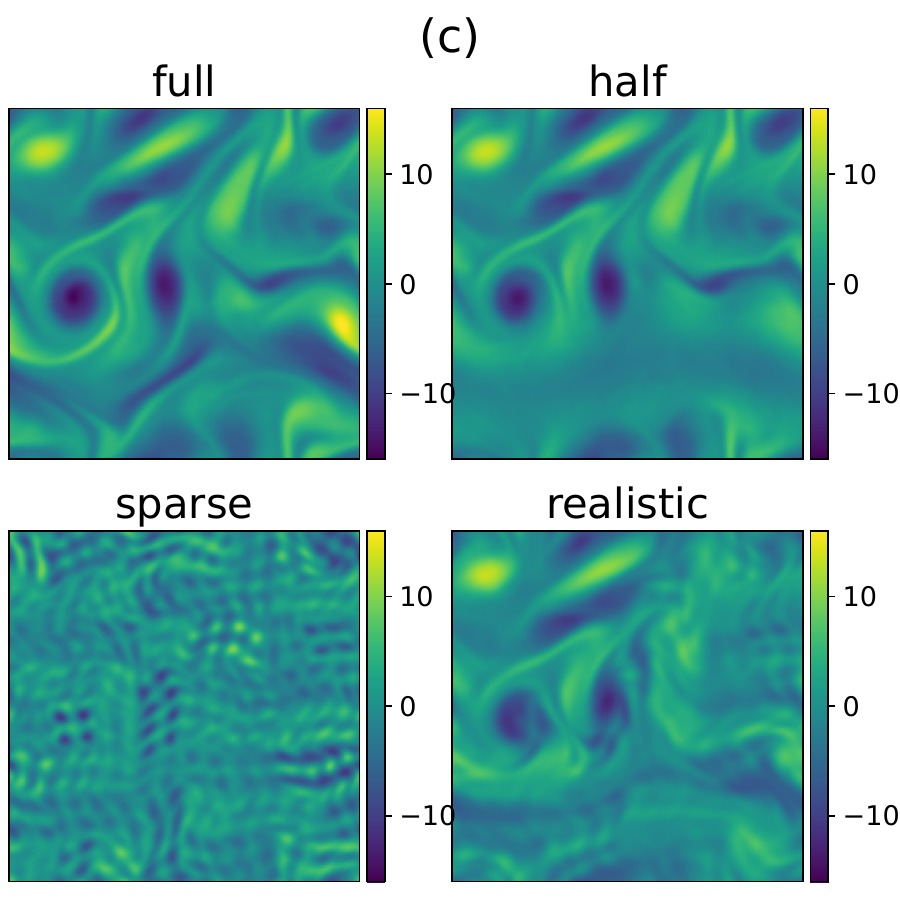}
\end{minipage}
\caption{Panel (a): observation operators in Kolmogorov experiment. Panel (b): inference results of the flow variational family. Panel (c): inference results of the Gauss variational family.}
\label{fig:Kolmogorov}
\end{figure*}

\begin{table*}[htbp]
\centering
\setlength{\tabcolsep}{4pt}
\renewcommand{\arraystretch}{1.10}
\caption{State RMSE across observation patterns (lower is better). Full tables with parameter metrics, including $|\Delta \log Re|$, are in Appendix~\ref{sec: kolmogorov, additional results}.}
\label{tab:Kolmogorov}
\begin{tabular}{lcccc}
\toprule
Method & Full & Half & Sparse & Realistic \\
\midrule
flow & \textbf{0.28 $\pm$ 0.00} & 2.31 $\pm$ 0.09 & \textbf{1.92 $\pm$ 0.12} & \textbf{1.60 $\pm$ 0.02} \\
Noisyflow & 0.65 $\pm$ 0.01 & \textbf{2.16 $\pm$ 0.10} & 2.18 $\pm$ 0.11 & 1.84 $\pm$ 0.07 \\
Gauss & 0.32 $\pm$ 0.00 & \textbf{2.16 $\pm$ 0.10} & 4.95 $\pm$ 0.10 & 2.99 $\pm$ 0.13 \\
NoisyGauss & 0.70 $\pm$ 0.01 & 2.40 $\pm$ 0.10 & 4.84 $\pm$ 0.10 & 2.91 $\pm$ 0.13 \\
MF & 0.55 $\pm$ 0.17 & 2.40 $\pm$ 0.02 & N/A & 3.63 $\pm$ 0.00 \\
IEnKS & 4.04 $\pm$ 0.12 & 4.13 $\pm$ 0.18 & 4.10 $\pm$ 0.18 & 4.02 $\pm$ 0.19 \\
4D-Var & N/A & N/A & 4.39 $\pm$ 0.15 & N/A \\
\midrule
IEnKS (oracle) & 1.54 $\pm$ 0.01 & 5.07 $\pm$ 0.30 & 6.13 $\pm$ 0.16 & 5.23 $\pm$ 0.19 \\
4D-Var (oracle) & 0.31 $\pm$ 0.00 & 2.42 $\pm$ 0.10 & 3.81 $\pm$ 0.14 & 3.06 $\pm$ 0.16 \\
\bottomrule
\end{tabular}
\end{table*}

\begin{table}[t]
\centering
\small
\setlength{\tabcolsep}{3.5pt}
\renewcommand{\arraystretch}{1.10}
\caption{Representative parameter-estimation errors in the ``realistic'' setting (lower is better). 
Full parameter tables across observation patterns are given in Appendix~G.}
\label{tab:Kolmogorov_small}
\begin{tabular}{@{}lccc@{}}
\toprule
Method & $|\Delta F|$ & $|\Delta \log Re|$ & bias \\
\midrule
flow       & \textbf{(2.3 $\pm$ 1.6)$\times 10^{-3}$} & \textbf{0.31 $\pm$ 0.01} & \textbf{0.57 $\pm$ 0.02} \\
Noisyflow  & $(1.1 \pm 0.7)\times 10^{-2}$         & $0.91 \pm 0.03$          & $0.98 \pm 0.02$ \\
Gauss      & $0.07 \pm 0.01$                       & $1.29 \pm 0.01$          & $0.66 \pm 0.01$ \\
NoisyGauss & $0.12 \pm 0.02$                       & $1.56 \pm 0.02$          & $1.06 \pm 0.01$ \\
MF         & $0.75 \pm 0.00$                       & $2.99 \pm 0.00$          & $0.66 \pm 0.02$ \\
IEnKS      & $0.89 \pm 0.00$                       & $2.31 \pm 0.00$          & $3.02 \pm 0.01$ \\
4D-Var     & N/A                                   & N/A                      & N/A \\
\bottomrule
\end{tabular}
\end{table}

Table~\ref{tab:Kolmogorov} reports state RMSEs (additional parameter/bias metrics are in Appendix~\ref{sec: kolmogorov, additional results}). Under “full” observations, PR-Smoother achieves state RMSE well below the observation noise level and matches oracle 4D-Var (true parameters and bias field; $10$ restarts). IEnKS ($16$ members; localization radius tuned over ${32,16,8,4}$ for each observation setting; Appendix~\ref{sec:Kolmogorov, appendix}; Fig.~\ref{fig:IEnKS_members} shows limited gains beyond $M{=}16$) is less competitive in this regime. When jointly estimating the full bias field and physical parameters, 4D-Var was less stable in our setup, and both baselines yielded larger errors.

Under “half” observations, both flow- and Gauss-based PR-Smoother recover the state accurately in the observed region, but neither recovers the unobserved half reliably (Gauss oversmooths, while flow remains sharper but inaccurate). In the “sparse” setting, flow-based models still reconstruct the state reasonably well, whereas Gauss-based models deteriorate and can produce extreme $\log Re$ values. Even for flow models, $\log Re$ remains biased, consistent with an identifiability issue: sparse point measurements do not resolve the small-scale structures controlled by the Reynolds number. In the “realistic” configuration, the flow-based family performs well on both state and parameters: the densely observed subregion identifies $F$ and $\log Re$, while the sparse network provides global coverage so the calibrated dynamics can propagate information across the domain. 

In the realistic observation setting, deterministic variants (Flow/Gauss) match or outperform noisy-transition variants under a matched wall-clock budget (Table~\ref{tab:Kolmogorov},~\ref{tab:Kolmogorov_small}; Appendix~\ref{sec: kolmogorov, additional results} for the full parameter learning results), unlike Sec.~\ref{sec:Lorenz96}. Because the ground-truth dynamics here are deterministic—and are contained in the noisy family as the zero-process-noise limit—we interpret this gap as slower convergence of the larger noisy search space when state, bias, and global parameters are learned jointly. Table~\ref{tab:Kolmogorov_small} reports recovery of $F$ and the bias field.

\textbf{Takeaways.}
(i) PR-Smoother can jointly learn high-dimensional sensor biases and global physical parameters at Kolmogorov scale.
(ii) With mixed coverage, dense regions support identifying $(F,\log Re)$ and calibrated dynamics transfer this information to sparsely observed regions.
(iii) Completely unobserved regions and parameters lacking small-scale observational support (e.g., $Re$ in the sparse case) remain difficult due to limited identifiability.

\section{Practical Scope and Limitations}
\label{sec:limitations}
PR-Smoother targets repeated-window DA with prescribed differentiable dynamics and observation models. Its offline training cost is amortized across windows from the same system. For legacy solvers, existing adjoint routines offer a route to gradient-based training while retaining the original forward model. Undertraining or distribution shift can reduce amortized inference accuracy. Appendix~\ref{app: number of training data} shows that, under a fixed training-compute budget, PR-Smoother remains competitive with oracle 4D-Var even with $N=288$ Kolmogorov-flow windows, though most gains occur by $N\simeq 10^5$--$10^6$. For applications requiring a point estimate, PR-Smoother could initialize per-window refinement with 4D-Var.

The reported state posteriors are conditional on point estimates of the shared physical and sensor parameters; a prior and a global variational factor $q_\psi(\theta)$ provide a natural extension for parameter uncertainty. Our Kolmogorov experiments highlight the importance of observation coverage for parameter identifiability. Changes in physical parameters or observation layouts may require adaptation of the trained smoother, with held-out predictive checks supporting deployment assessment. Evaluation on operational numerical weather prediction systems is an important next step.

\section{Conclusion}

We introduce PR-Smoother, a simulator-preserving non-Gaussian smoother for prescribed-simulator DA. By learning future-conditioned corrections around the prescribed rollout, PR-Smoother represents multimodal smoothing posteriors directly in physical space while keeping the prescribed dynamics and observation operator explicit inside a single observation-only ELBO for joint state, parameter, and bias learning. Across Lorenz–96 and 16,384-dimensional Kolmogorov flow, it captures multimodal posteriors and scales to high-dimensional nonlinear assimilation. Its broader impact is indirect: scientific deployment still requires domain-specific validation and calibration.

\begin{ack}
This work was supported by JST, CREST Grant Number JPMJCR24Q3 including AIP challenge program, Japan, and RIKEN TRIP initiative (RIKEN Prediction Science).
This research is also partially supported by Exploratory Collaborative Research Project, Information Technology Center, The University of Tokyo.
\end{ack}

\bibliography{example_paper}
\bibliographystyle{plainnat}


\appendix

\section{On the Variational Family Including the True Posterior in Two Limiting Cases}
\label{sec: proof for variational family including true posterior}

This appendix shows that our variational family contains the exact smoothing posterior in two important limits:
(i) deterministic dynamics and (ii) linear-Gaussian state-space models (LG-SSMs).
The arguments rely only on the Markov structure of the model and on the fact that our conditional NF
and ``shock head'' are rich enough to represent the relevant Gaussian conditionals. We use the notation and factorization
in Eqs.~\eqref{eq:dynamics}--\eqref{eq:generative model} and \eqref{eq:det-transition-Q}--\eqref{eq:q-markov} in the main text.

Throughout we assume:
\begin{itemize}
  \item the conditional flow $g_\phi$ used in $q_\phi(X_1 \mid Y_{1:T})$ is expressive enough to represent any smooth
        posterior density over $X_1$ (a standard universality assumption for NFs), and
  \item the shock head used in $q_\phi(X_t \mid X_{t-1}, Y_{t:T})$ can represent any mean and covariance in the
        Gaussian family of Eq.~\eqref{eq:q-corr}.
\end{itemize}
Under these assumptions we show that the true smoother $p_\theta(X_{1:T} \mid Y_{1:T})$ lies in our variational family
in the two limits.

\subsection{Deterministic Dynamics}
\label{sec:deterministic-limit}

Consider a SSM with deterministic dynamics
\begin{equation}
  X_{t+1} = f_\theta(X_t), \qquad t = 1,\dots,T-1,
\end{equation}
and observation model $Y_t \sim p_\theta(Y_t \mid X_t)$ as in Eq.~\eqref{eq:observation}.
The transition kernel is a Dirac measure
\begin{equation}
  p_\theta(X_t \mid X_{t-1}) = \delta\!\bigl(X_t - f_\theta(X_{t-1})\bigr),
\end{equation}
or, equivalently, the vanishing-noise limit $Q \to 0$ of Eq.~\eqref{eq:det-transition-Q}.

Let $F_t : \mathbb{R}^d \to \mathbb{R}^d$ denote the $t$-step deterministic rollout,
\begin{equation}
  F_1(x) = x, \qquad F_{t+1}(x) = f_\theta(F_t(x)), \qquad X_t = F_t(X_1).
\end{equation}
Thus the entire trajectory $X_{1:T}$ is a deterministic function of the initial state $X_1$.
The generative smoothing posterior therefore satisfies
\begin{equation}
  p_\theta(X_{1:T} \mid Y_{1:T})
  = p_\theta(X_1 \mid Y_{1:T}) \prod_{t=2}^T
      \delta\!\bigl(X_t - f_\theta(X_{t-1})\bigr),
\end{equation}
i.e., it is exactly the pushforward of the one-step posterior $p_\theta(X_1 \mid Y_{1:T})$ through the deterministic dynamics.

Our deterministic variational family (Eqs.~\ref{eq:q-factorization-Q}--\ref{eq:elbo-det}) is
\begin{equation}
  q_\phi(X_{1:T} \mid Y_{1:T})
  = q_\phi(X_1 \mid Y_{1:T}) \prod_{t=2}^T
      p_{\theta,\text{dyn},Q}(X_t \mid X_{t-1}),
\end{equation}
with
\begin{equation}
  p_{\theta,\text{dyn},Q}(X_t \mid X_{t-1})
  = \mathcal{N}\!\bigl(X_t; f_\theta(X_{t-1}), Q\bigr),
\end{equation}
and in the $Q \to 0$ limit this reduces to the sampling scheme
\begin{equation}
  X_1 \sim q_\phi(X_1 \mid Y_{1:T}), \qquad
  X_{t+1} = f_\theta(X_t), \quad t = 1,\dots,T-1.
  \label{eq:deterministic-variational-rollout-appendix}
\end{equation}

Because the $\delta$-factors coincide with those of the generative model, matching the full trajectory posterior reduces to
matching the initial-state posterior. By expressivity of the conditional flow, there exists a parameter setting $\phi^\star$ such that
\begin{equation}
  q_{\phi^\star}(X_1 \mid Y_{1:T}=y_{1:T}) = p_\theta(X_1 \mid Y_{1:T}=y_{1:T})
\end{equation}
for any $y_{1:T} \in {\rm supp}\bigl(p_{\theta}(Y_{1:T})\bigr)$. 
Plugging this into Eq.~\eqref{eq:deterministic-variational-rollout-appendix} shows
\begin{equation}
  q_{\phi^\star}(X_{1:T} \mid Y_{1:T})
  = p_\theta(X_{1:T} \mid Y_{1:T}),
\end{equation}
so the true deterministic smoother lies in our variational family.

In other words, in the deterministic (strong-constraint) limit the only nontrivial posterior degree of freedom is
$p_\theta(X_1 \mid Y_{1:T})$, and our conditional flow can represent it exactly; the rest of the trajectory is obtained by the
same deterministic rollout in both the generative and variational models.

\subsection{Linear-Gaussian SSMs}
\label{app:lgssm}

We now consider the classical LG-SSM
\begin{align}
  X_1 &\sim \mathcal{N}(m_0, P_0), \\
  X_t \mid X_{t-1} &\sim \mathcal{N}(A X_{t-1}, Q), \qquad t \ge 2, \\
  Y_t \mid X_t &\sim \mathcal{N}(C X_t, R), \qquad t \ge 1,
\end{align}
which is a special case of Eqs.~\eqref{eq:dynamics}--\eqref{eq:generative model} with $f_\theta(x) = A x$ and $h_\theta(x) = C x$.
It is well known that in this setting all filtering and smoothing distributions are Gaussian; the forward (Kalman) filter and
backward Rauch-Tung-Striebel (RTS) smoother can be derived as Gaussian message passing on the linear factor graph.

Our noisy variational family (Sec.~\ref{sec:noisy}) factorizes as
\begin{equation}
  q_\phi(X_{1:T} \mid Y_{1:T})
  = q_\phi(X_1 \mid Y_{1:T})
    \prod_{t=2}^T q_\phi(X_t \mid X_{t-1}, Y_{t:T}),
  \label{eq:variational-backward-factorization}
\end{equation}
with
\begin{equation}
  q_\phi(X_t \mid X_{t-1}, Y_{t:T})
  = \mathcal{N}\bigl(X_t; A X_{t-1} + m_t,\, S_t\bigr),
  \label{eq:variational-transition}
\end{equation}
where $m_t = m_t(X_{t-1}, Y_{t:T})$ and $S_t = S_t(X_{t-1}, Y_{t:T})$ are produced by the future-aware shock-head network
(Sec.~\ref{sec:anticausal encoder}). We show that, in an LG-SSM, the exact conditionals $p_\theta(X_t \mid X_{t-1}, Y_{t:T})$ are Gaussian and of the form
\eqref{eq:variational-transition}. Combined with the Gaussianity of $p_\theta(X_1 \mid Y_{1:T})$, this implies that the exact
smoother belongs to our variational family.

\subsubsection{Gaussian Backward Messages}

Define the \emph{backward message}
\begin{equation}
  \beta_t(x_t) \;\propto\; p_\theta(Y_{t+1:T}=y_{t+1:T} \mid X_t=x_t), \qquad t = 1,\dots,T-1,
\end{equation}
with terminal condition $\beta_T(x_T) \equiv 1$.
Using the Markov property one obtains the standard backward recursion
\begin{equation}
  \beta_t(x_t)
  \;\propto\;
  \int p_\theta(x_{t+1} \mid x_t)\,
       p_\theta(y_{t+1} \mid x_{t+1})\,
       \beta_{t+1}(x_{t+1})\, \mathrm{d}x_{t+1}.
  \label{eq:backward-recursion}
\end{equation}

We show by induction that every $\beta_t(x_t)$ is (unnormalized) Gaussian in $x_t$.

\paragraph{Base case.}
$\beta_T(x_T) \equiv 1$ is a constant function, which we can regard as a Gaussian with zero information (precision) matrix.

\paragraph{Inductive step.}
Assume
\begin{equation}
  \beta_{t+1}(x_{t+1})
  \;\propto\;
  \exp\Bigl(
    -\tfrac{1}{2}\, x_{t+1}^\top \Lambda_{t+1} x_{t+1}
    + \eta_{t+1}^\top x_{t+1}
  \Bigr)
\end{equation}
for some precision matrix $\Lambda_{t+1} \succeq 0$ and vector $\eta_{t+1}$.
In the LG-SSM, the two other factors in Eq.~\eqref{eq:backward-recursion} are
\begin{align}
  p_\theta(x_{t+1} \mid x_t)
  &\propto
    \exp\Bigl(
      -\tfrac{1}{2}\, (x_{t+1} - A x_t)^\top Q^{-1} (x_{t+1} - A x_t)
    \Bigr), \\
  p_\theta(y_{t+1} \mid x_{t+1})
  &\propto
    \exp\Bigl(
      -\tfrac{1}{2}\, (y_{t+1} - C x_{t+1})^\top R^{-1} (y_{t+1} - C x_{t+1})
    \Bigr).
\end{align}

The integrand in Eq.~\eqref{eq:backward-recursion} is thus the product of three Gaussians in the common variable $x_{t+1}$,
hence itself an (unnormalized) Gaussian in $x_{t+1}$. Integrating a joint Gaussian over $x_{t+1}$ produces an exponential
quadratic function in the remaining variable $x_t$; that is, there exist $\Lambda_t \succeq 0$ and $\eta_t$ such that
\begin{equation}
  \beta_t(x_t)
  \;\propto\;
  \exp\Bigl(
    -\tfrac{1}{2}\, x_t^\top \Lambda_t x_t
    + \eta_t^\top x_t
  \Bigr).
\end{equation}
Thus $\beta_t$ is Gaussian for all $t$ by induction.

\subsubsection{The Conditional $p(X_t \mid X_{t-1}, Y_{t:T})$ is Gaussian}

Using the backward message, the exact conditional appearing in Eq.~\eqref{eq:forward-factorize} can be written as
\begin{equation}
  p_\theta(x_t \mid x_{t-1}, y_{t:T})
  \;\propto\;
  p_\theta(x_t \mid x_{t-1})\,
  p_\theta(y_t \mid x_t)\,
  \beta_t(x_t).
  \label{eq:smoothing-conditional}
\end{equation}
We have just shown that all three factors are Gaussian in $x_t$. Their product is therefore an (unnormalized) Gaussian in $x_t$,
so
\begin{equation}
  p_\theta(x_t \mid x_{t-1}, y_{t:T})
  = \mathcal{N}\bigl(x_t; \mu_t(x_{t-1}), S_t\bigr),
  \label{eq:true-conditional-gaussian}
\end{equation}
with covariance
\begin{equation}
  S_t
  = \Bigl(Q^{-1} + C^\top R^{-1} C + \Lambda_t\Bigr)^{-1},
\end{equation}
and mean
\begin{equation}
  \mu_t(x_{t-1})
  = S_t \Bigl(Q^{-1} A x_{t-1} + C^\top R^{-1} y_t + \eta_t\Bigr).
  \label{eq:true-conditional-mean}
\end{equation}

The right-hand side of Eq.~\eqref{eq:true-conditional-mean} is an affine function of $x_{t-1}$.
Hence there exist a matrix $M_t$ and vector $b_t$ such that
\begin{equation}
  \mu_t(x_{t-1}) = M_t x_{t-1} + b_t.
\end{equation}
Writing $M_t x_{t-1} + b_t$ as
\begin{equation}
  M_t x_{t-1} + b_t
  = A x_{t-1}
    + \underbrace{\bigl((M_t - A) x_{t-1} + b_t\bigr)}_{=: m_t(x_{t-1})},
\end{equation}
we can put Eq.~\eqref{eq:true-conditional-gaussian} in the same form as our variational transition
Eq.~\eqref{eq:variational-transition}:
\begin{equation}
  p_\theta(x_t \mid x_{t-1}, y_{t:T})
  = \mathcal{N}\bigl(x_t; A x_{t-1} + m_t(x_{t-1}),\, S_t\bigr).
  \label{eq:true-conditional-variational-form}
\end{equation}

Because the shock head in Eq.~\eqref{eq:q-markov} is allowed to depend on $(X_{t-1}, Y_{t:T})$, 
there exists $\phi^\star$ such that its outputs $(m_t, S_t)$ coincide with those in
Eq.~\eqref{eq:true-conditional-variational-form} for all $t$. In other words, the exact conditional
$p_\theta(X_t \mid X_{t-1}, Y_{t:T})$ is a member of our Gaussian transition family.

\subsubsection{Putting the Pieces Together}

Finally, in an LG-SSM the smoothed marginal over the initial state is also Gaussian:
\begin{equation}
  p_\theta(X_1 \mid Y_{1:T}) = \mathcal{N}(m_{1 \mid T}, P_{1 \mid T}),
\end{equation}
with $(m_{1 \mid T}, P_{1 \mid T})$ given by the standard RTS backward recursion on top of the Kalman filter.

Since $q_\phi(X_1 \mid Y_{1:T})$ is parameterized by a conditional NF with Gaussian base distribution,
it can represent this Gaussian exactly for some $\phi^\star$. Choosing
\begin{align}
  q_{\phi^\star}(X_1 \mid Y_{1:T})
    &= p_\theta(X_1 \mid Y_{1:T}), \\
  q_{\phi^\star}(X_t \mid X_{t-1}, Y_{t:T})
    &= p_\theta(X_t \mid X_{t-1}, Y_{t:T})
       \quad \text{for } t \ge 2,
\end{align}
and plugging into the factorization \eqref{eq:variational-backward-factorization}, we obtain
\begin{equation}
  q_{\phi^\star}(X_{1:T} \mid Y_{1:T})
  = p_\theta(X_{1:T} \mid Y_{1:T}),
\end{equation}
because Eq.~\eqref{eq:forward-factorize} implies that the exact smoother admits the same
future-conditioned parameterization as our variational family. 

Thus, for linear-Gaussian SSMs, PR-Smoother's Markov-structured Gaussian variational family is \emph{posterior‑inclusive}:
the exact smoothing posterior lies within the family and can be recovered by choosing the shock-head parameters to reproduce
the Gaussian backward messages of the RTS smoother.

\section{Algorithm}
For an observation window $Y_{1:T}$, PR-Smoother proceeds as follows.
\begin{algorithm}[htb]
  \caption{PR-Smoother (training-time inference)}
  \label{alg:pr_smoother}
  \begin{algorithmic}
    \STATE {\bfseries Input:} dataset $\{Y_{1:T}^{(n)}\}_{n=1}^N$; dynamics $f_{\theta_{\mathrm{dyn}}}$; observation operator $h_{\theta_{\mathrm{obs}}}$; prior $p(X_1)$
    \STATE {\bfseries Output:} learned parameters $(\hat{\theta}, \hat{\phi})$ for amortized smoother $q_{\hat{\phi}}(X_{1:T} \!\mid\! Y_{1:T})$
    \STATE Initialize $\theta, \phi$
    \REPEAT
      \STATE Sample a minibatch $\{Y_{1:T}^{(b)}\}_{b=1}^B$
      \FOR{each window $Y_{1:T}^{(b)}$}
        \STATE Encode $c^{(b)} \leftarrow E_\phi(Y_{1:T}^{(b)})$
        \STATE Sample $X_1^{(b)} \sim q_\phi(\cdot \!\mid\! Y_{1:T}^{(b)})$ via conditional flow $g_\phi$
        \FOR{$t = 1$ {\bfseries to} $T-1$}
          \STATE Predict $\tilde{X}_{t+1}^{(b)} \leftarrow f_{\theta_{\mathrm{dyn}}}(X_t^{(b)})$
          \IF{deterministic (strong-constraint) variant}
            \STATE $X_{t+1}^{(b)} \leftarrow \tilde{X}_{t+1}^{(b)}$
          \ELSE \STATE Sample $X_{t+1}^{(b)} \sim q_\phi(\cdot \!\mid\! X_t^{(b)}, Y_{t+1:T}^{(b)})$ \hfill (future-aware noisy transition)
          \ENDIF
        \ENDFOR
      \ENDFOR
      \STATE Estimate ELBO $\widehat{\mathcal{L}}(\theta,\phi)$ on the minibatch and take a gradient-ascent step on $\widehat{\mathcal{L}}$
    \UNTIL{convergence}
  \end{algorithmic}
\end{algorithm}

\section{Architecture for the Anti-Causal Encoder}
\label{sec:anticausal encoder}
\paragraph{Future‑aware Transformer encoder.} We build per‑step ``anti‑causal'' features that summarize only future observations. Let $\tilde{Y}$ be the time‑reversed observation sequence. A standard causal Transformer encoder on $\tilde{Y}$ (masking each token to attend only to its past) yields hidden states $\tilde{H}$. Flipping back to original time, $H_t =\tilde{H}_{T+1-t}$ depends on $Y_{t:T}$ only. We take per‑step contexts $c_t =W_s H_{t}$  so that $c_t$ summarizes $Y_{t:T}$. The shock head then defines $q_{\phi} (X_{t} \!\mid\! X_{t-1} ,Y_{t:T})=N(f_{\theta} (X_{t-1})+m_t ,S_t )$, with $(m_t ,S_t)$ predicted from $[X_{t-1} ,c_t ,\tau_t ]$, where $\tau_t = t/T$. This preserves model‑based transparency---likelihoods are evaluated with $h_{\theta}$, transitions with $f_{\theta}$---while allowing future observations to explain stepwise corrections under a closed‑form per‑step KL to the physical noise $N(0,Q)$.

\section{Use of NF over DM or FM}
\label{sec:NF over DM/MF}
\paragraph{Diffusion models (DM).} Conditional score‑based diffusion models can be used to represent $q_{\phi}(X_{1} \!\mid\! Y_{1:T})$ and are often more permissive geometrically than NFs. However, our training objective requires sample‑wise 
log $q_{\phi}$. Standard DMs optimize a likelihood bound on data but do not provide tractable posterior density values for arbitrary $X_1$ at test time. One can recover densities through the probability‑flow ODE, but this entails solving an ODE and integrating the divergence (e.g., with Hutchinson estimators) per sample. This added NFE (number of function evaluations) materially increases wall‑clock cost.

\paragraph{Flow matching (FM) / continuous normalizing flows (CNF).} FM trains an ODE transport whose trajectories allow density tracking via the instantaneous change‑of‑variables formula. This makes FM more compatible with our ELBO than DMs. The practical trade‑off is that evaluating 
log $q_{\phi}$ still requires integrating a divergence along the path and sampling requires an ODE solve; both are substantially heavier than a single NF evaluation unless very small NFE integrators are used.

\section{Full ELBO Tables}
\label{sec:elbo_appendix}

\paragraph{Evaluation protocol.}
We report held-out ELBOs for all variational methods (flow / Gauss / Noisyflow / NoisyGauss / MF).
Each entry is computed on a full test sequence (smoothing window) and reported as mean $\pm$ std over five random seeds, evaluated on the test set.
The ELBO is estimated via Monte Carlo sampling from the variational posterior with one sample.

\paragraph{Reporting convention.}
We report the ELBO per sequence (i.e., the ELBO value for a full window, summed over time steps and observed components).
Therefore, ELBO magnitudes should be compared \emph{only within the same experimental setting} (same window length and observation pattern);
comparisons across columns with different $T$ or different observation densities are not meaningful.


\begin{table}[ht]
\centering
\vspace{1.5mm}
\caption{Lorenz96, 40D: ELBO for the test set (higher is better). 
}
\label{tab:Lorenz96_elbo_small}
{\small
\setlength{\tabcolsep}{4pt}
\renewcommand{\arraystretch}{1.10}
\begin{tabular}{lccc}
  \toprule
  Method & Linear & Nonlinear & Nonlinear, noisy \\
  \midrule
Flow       & -3287 $\pm$ 5 & -3603 $\pm$ 8 & -3908 $\pm$ 12 \\
Noisyflow  & \textbf{-3024 $\pm$ 0} & \textbf{-3439 $\pm$ 72} & \textbf{-3326 $\pm$ 9} \\
Gauss      & -3378 $\pm$ 32 & -3711 $\pm$ 6 & -4051 $\pm$ 14 \\
NoisyGauss & -3055 $\pm$ 1 & -3687 $\pm$ 16 & -3991 $\pm$ 9 \\
MF         & -3519 $\pm$ 1 & -4132 $\pm$ 4 & -4416 $\pm$ 3 \\
  \bottomrule
\end{tabular}
}
\end{table}

\begin{table}[ht]
\centering
\small
\setlength{\tabcolsep}{4pt}
\renewcommand{\arraystretch}{1.10}
\caption{Kolmogorov flow: ELBO for the test set (higher is better) for observation patterns. 
}
\label{tab:Kolmogorov_elbo_small}
\begin{tabular}{lcccc}
\toprule
Method/ELBO & Full & Half & Sparse & Realistic \\
\midrule
flow      & \textbf{-259171} $\pm$ \textbf{1214} & \textbf{-129962} $\pm$ \textbf{639} & \textbf{-8062} $\pm$ \textbf{71} & \textbf{-71863} $\pm$ \textbf{363} \\
Noisyflow & -295524 $\pm$ 4439 & -220524 $\pm$ 2060 & -8865 $\pm$ 61 & -84129 $\pm$ 664 \\
Gauss     & -263759 $\pm$ 628 & -134580 $\pm$ 630 & -15280 $\pm$ 65 & -87547 $\pm$ 320 \\
NoisyGauss& -303450 $\pm$ 2263 & -230927 $\pm$ 1609 & -15992 $\pm$ 66 & -95050 $\pm$ 594 \\
MF        & -266806 $\pm$ 398 & -140893 $\pm$ 343 & N/A & -91767 $\pm$ 86 \\
\bottomrule
\end{tabular}
\end{table}

\section{Trained $\sigma_Q$ Values in Noisyflow and Noisygauss Models}
\label{sec:sigmaq_appendix}
\paragraph{Learned process-noise scale in the noisy-transition variants.}
In the Lorenz-96 (40D) experiments of Sec.~\ref{sec:Lorenz96}, the noisy-transition variants (Noisyflow/NoisyGauss) treat the process-noise covariance in the generative dynamics as isotropic,
\(p(x_{t+1}\mid x_t)=\mathcal{N}(f(x_t), Q)\) with \(Q=\sigma_Q^2 I\),
and parameterize the variational transition as
\(q_\phi(x_{t+1}\mid x_t,y_{t+1:T})=\mathcal{N}(f(x_t)+m_t, S_t)\) with \(S_t=\sigma_S^2 I\) (shared across time).
Both \(\sigma_Q\) and \(\sigma_S\) are learned jointly with the inference-network parameters by maximizing the ELBO for each experimental regime.
Table~\ref{tab:sigma_Q} reports the resulting \(\ln\sigma_Q\) values (mean \(\pm\) standard deviation across successful random seeds).

\paragraph{Interpretation.}
For the ``linear'' and ``nonlinear'' columns, the data are generated with deterministic dynamics (no injected process noise), so the corresponding ground-truth value is \(\sigma_Q^{\text{True}}=0\) (reported as \(\ln\sigma_Q^{\text{True}}=-\infty\)).
In these deterministic regimes, \(\sigma_Q\) is not strictly identifiable: within the noisy-transition ELBO, \(\sigma_Q\) primarily sets the strength of the per-step KL regularization
\(\mathrm{KL}\!\left[q_\phi(x_{t+1}\mid x_t,y_{t+1:T})\,\|\,p(x_{t+1}\mid x_t)\right]\),
so the learned \(\sigma_Q\) should be read as an \emph{effective} model-error scale selected by the ELBO trade-off rather than as a physical noise estimate.
In contrast, in the ``noisy dynamics'' setting the data are generated with true per-step process noise \(\sigma_Q^{\text{True}}=0.1\) (\(\ln\sigma_Q^{\text{True}}=\ln 0.1=-2.30\)).
Here, Noisyflow recovers the planted noise level (\(\ln\sigma_Q=-2.29\pm 0.04\)), while NoisyGauss substantially underestimates it (\(\ln\sigma_Q=-4.61\pm 0.16\)).
This diagnostic complements the RMSE results in Fig.~\ref{fig:l96_capsule}: the more expressive Noisyflow posterior not only improves state accuracy but also yields more accurate calibration of the process-noise scale under variational-EM training.

\begin{table}[ht]
\centering
\small
\setlength{\tabcolsep}{4pt}
\renewcommand{\arraystretch}{1.10}
\caption{Trained $\sigma_Q$ values in the "noisyflow" and "noisygauss" models in Lorenz96 experiments.}
\label{tab:sigma_Q}
\begin{tabular}{lccc}
\toprule
Method & linear & nonlinear & noisy dynamics \\
\midrule
Noisyflow & $-2.58 \pm 0.00$ & $-3.59 \pm 0.53$ & \textbf{-2.29} $\pm$ \textbf{0.04} \\
NoisyGauss& $-2.12 \pm 0.01$ & $-5.06 \pm 0.08$ & $-4.61 \pm 0.16$ \\
True values & $-\infty$ & $-\infty$ & $-2.30$ \\
\bottomrule
\end{tabular}
\end{table}

\section{Additional Results for the Kolmogorov Experiments}
\label{sec: kolmogorov, additional results}

\begin{table*}[ht]
\centering
\small
\setlength{\tabcolsep}{3pt}
\renewcommand{\arraystretch}{1.15}
\caption{Parameter estimation errors across observation patterns (lower is better), for the \textbf{Full} and \textbf{Half} settings. Three metrics per pattern: $|\Delta F|$, $|\Delta \log Re|$, and observation bias. If no models out of five initial seeds succeed in training or optimization, the results are treated as N/A. The oracle models are removed because they do not estimate $F$, $\log Re$, nor the observation bias: they are prescribed.}
\label{tab:Kolmogorov_additional_full_half}
\begin{tabular}{l *{2}{ccc}}
\toprule
& \multicolumn{3}{c}{Full} & \multicolumn{3}{c}{Half} \\
\cmidrule(lr){2-4}\cmidrule(lr){5-7}
Method &
\makebox[0pt][c]{$|\Delta F|$} &
\makebox[0pt][c]{$|\Delta \log Re|$} &
\makebox[0pt][c]{bias} &
\makebox[0pt][c]{$|\Delta F|$} &
\makebox[0pt][c]{$|\Delta \log Re|$} &
\makebox[0pt][c]{bias} \\
\midrule
flow       & \textbf{(3.3 $\pm$ 3.6)$\times 10^{-4}$} & \textbf{0.14 $\pm$ 0.01} & \textbf{0.53 $\pm$ 0.01} & (1.1 $\pm$ 0.8)$\times 10^{-3}$ & \textbf{0.14 $\pm$ 0.01} & 0.53 $\pm$ 0.01 \\
Noisyflow  & (6.3 $\pm$ 1.3)$\times 10^{-3}$ & 0.72 $\pm$ 0.03 & 0.94 $\pm$ 0.03 & (1.0 $\pm$ 0.2)$\times 10^{-2}$ & 0.27 $\pm$ 0.00 & 0.96 $\pm$ 0.03 \\
Gauss      & (1.1 $\pm$ 0.3)$\times 10^{-3}$ & 0.23 $\pm$ 0.00 & 0.57 $\pm$ 0.01 & \textbf{(8.8 $\pm$ 6.4)$\times 10^{-4}$} & 0.27 $\pm$ 0.00 & 0.60 $\pm$ 0.01 \\
NoisyGauss & (7.1 $\pm$ 2.2)$\times 10^{-3}$ & 0.86 $\pm$ 0.02 & 0.86 $\pm$ 0.02 & (1.5 $\pm$ 0.2)$\times 10^{-3}$ & 0.99 $\pm$ 0.03 & 1.01 $\pm$ 0.02 \\
MF         & (3.4 $\pm$ 0.1)$\times 10^{-2}$ & 1.82 $\pm$ 0.01 & 0.40 $\pm$ 0.01 & 0.08 $\pm$ 0.00 & 2.44 $\pm$ 0.00 & \textbf{0.44 $\pm$ 0.01} \\
IEnKS      & 0.89 $\pm$ 0.00 & 2.31 $\pm$ 0.00 & 2.84 $\pm$ 0.01 & 0.89 $\pm$ 0.00 & 2.30 $\pm$ 0.02 & 2.96 $\pm$ 0.01 \\
4D-Var     & N/A & N/A & N/A & N/A & N/A & N/A \\
\bottomrule
\end{tabular}
\end{table*}

\begin{table*}[ht]
\centering
\small
\setlength{\tabcolsep}{3pt}
\renewcommand{\arraystretch}{1.15}
\caption{Same as Table~\ref{tab:Kolmogorov_additional_full_half}, but for the \textbf{Sparse} and \textbf{Realistic} settings.}
\label{tab:Kolmogorov_additional_sparse_realistic}
\begin{tabular}{l *{2}{ccc}}
\toprule
& \multicolumn{3}{c}{Sparse} & \multicolumn{3}{c}{Realistic} \\
\cmidrule(lr){2-4}\cmidrule(lr){5-7}
Method &
\makebox[0pt][c]{$|\Delta F|$} &
\makebox[0pt][c]{$|\Delta \log Re|$} &
\makebox[0pt][c]{bias} &
\makebox[0pt][c]{$|\Delta F|$} &
\makebox[0pt][c]{$|\Delta \log Re|$} &
\makebox[0pt][c]{bias} \\
\midrule
flow       & \textbf{0.10 $\pm$ 0.02} & 2.47 $\pm$ 0.02 & \textbf{0.62 $\pm$ 0.14} & \textbf{(2.3 $\pm$ 1.6)$\times 10^{-3}$} & \textbf{0.31 $\pm$ 0.01} & \textbf{0.57 $\pm$ 0.02} \\
Noisyflow  & 0.13 $\pm$ 0.05 & 2.80 $\pm$ 0.02 & 1.00 $\pm$ 0.12 & (1.1 $\pm$ 0.7)$\times 10^{-2}$ & 0.91 $\pm$ 0.03 & 0.98 $\pm$ 0.02 \\
Gauss      & 0.83 $\pm$ 0.03 & 18.18 $\pm$ 0.08 & 1.46 $\pm$ 0.10 & 0.07 $\pm$ 0.01 & 1.29 $\pm$ 0.01 & 0.66 $\pm$ 0.01 \\
NoisyGauss & 0.66 $\pm$ 0.07 & 12.76 $\pm$ 0.04 & 1.43 $\pm$ 0.15 & 0.12 $\pm$ 0.02 & 1.56 $\pm$ 0.02 & 1.06 $\pm$ 0.01 \\
MF         & N/A & N/A & N/A & 0.75 $\pm$ 0.00 & 2.99 $\pm$ 0.00 & 0.66 $\pm$ 0.02 \\
IEnKS      & 0.89 $\pm$ 0.01 & \textbf{2.30 $\pm$ 0.01} & 3.03 $\pm$ 0.01 & 0.89 $\pm$ 0.00 & 2.31 $\pm$ 0.00 & 3.02 $\pm$ 0.01 \\
4D-Var     & 1.35 $\pm$ 0.19 & 26.95 $\pm$ 1.66 & 3.61 $\pm$ 0.13 & N/A & N/A & N/A \\
\bottomrule
\end{tabular}
\end{table*}

\section{Compute-budgeted data scaling}
\label{app: number of training data}
\begin{figure}
    \centering
    \includegraphics[width=\linewidth]{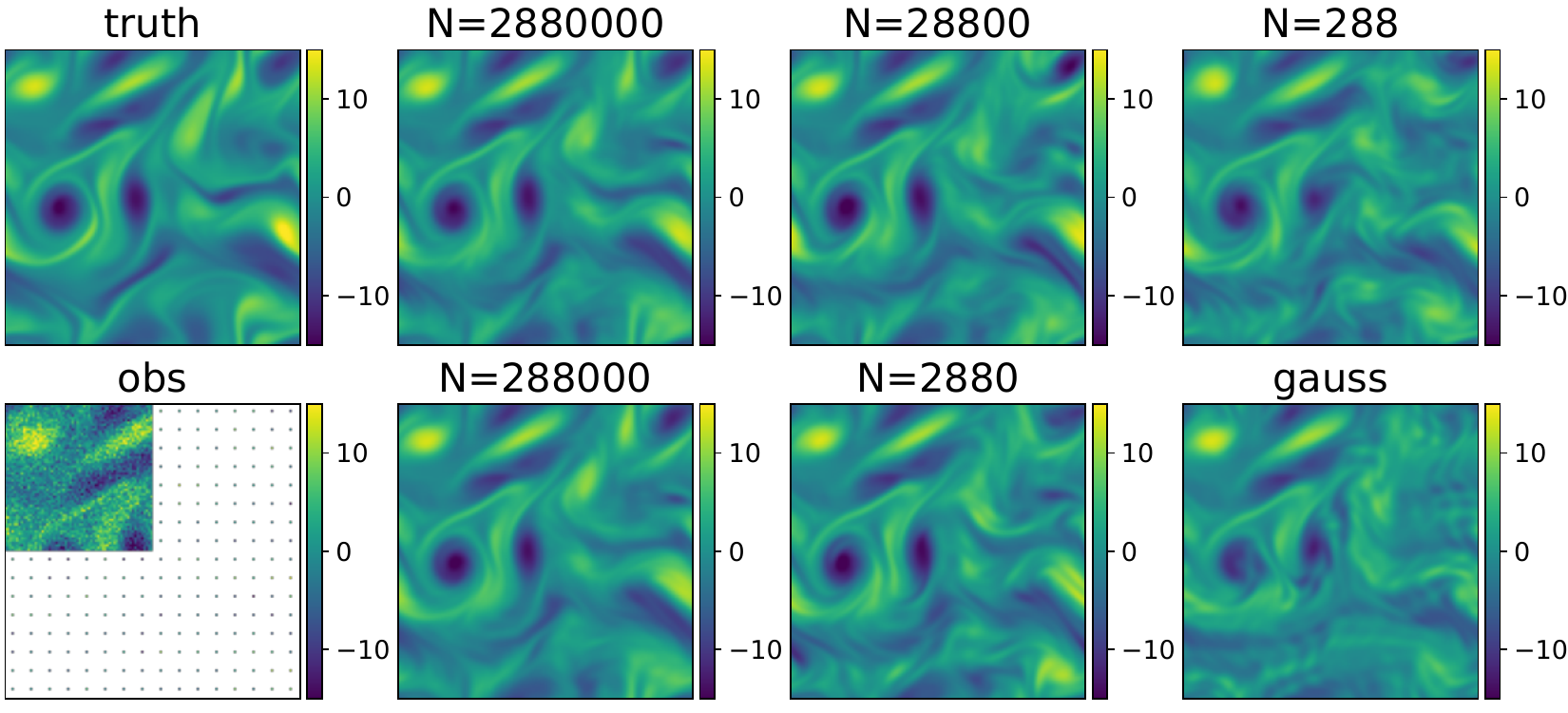}
    \caption{Kolmogorov flow (realistic observation network): qualitative compute-budgeted data scaling. Left column: ground-truth vorticity field (truth) and partial observations under the “realistic” network (obs). Other panels: PR-Smoother reconstructions (Flow variant) on the same held-out window when trained with $N \in \{2.88\times 10^6, 2.88\times 10^5, 2.88\times 10^4, 2.88\times 10^3, 2.88\times 10^2\}$ observation windows (panel titles), keeping the number of gradient updates fixed across N (compute-budgeted training). The final panel (“gauss”) shows the diagonal-Gaussian variant of $N=2.88\times 10^6$ for reference. Even at $N = 2.88\times 10^2 (288)$ windows, the Flow model preserves coherent vortical structure; quantitatively, its state RMSE remains below the oracle 4D-Var reference in Fig.~\ref{fig:ablation}(a).}
    \label{fig:ablation_visualize}
\end{figure}

\begin{figure}[htbp]
    \centering
    \includegraphics[width=0.33\linewidth]{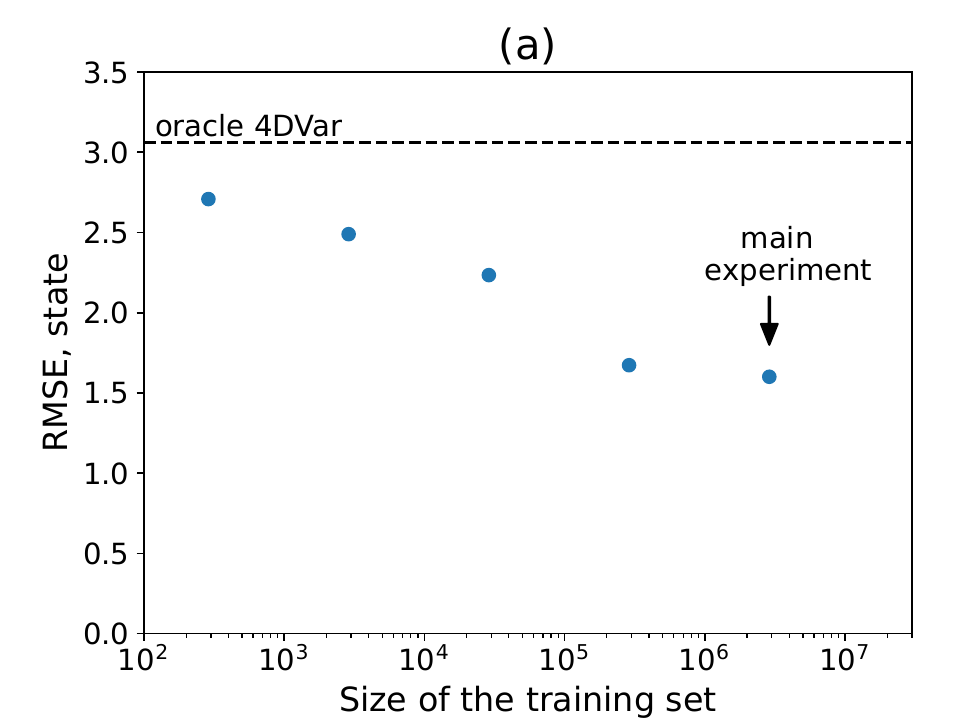}
    \includegraphics[width=0.33\linewidth]{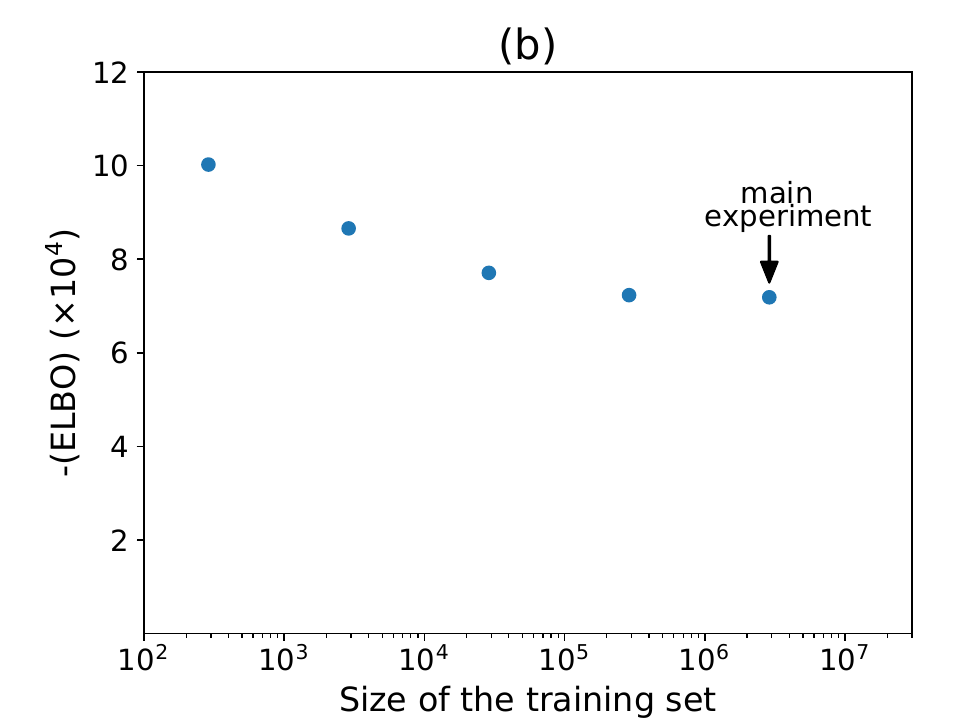}
    \includegraphics[width=0.33\linewidth]{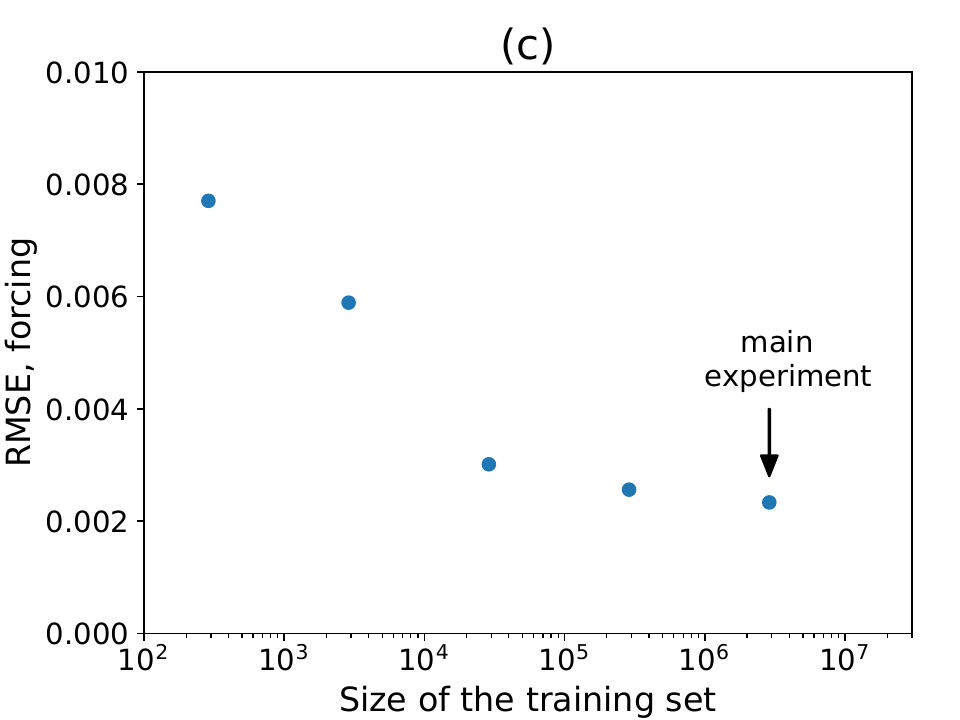}
    \includegraphics[width=0.33\linewidth]{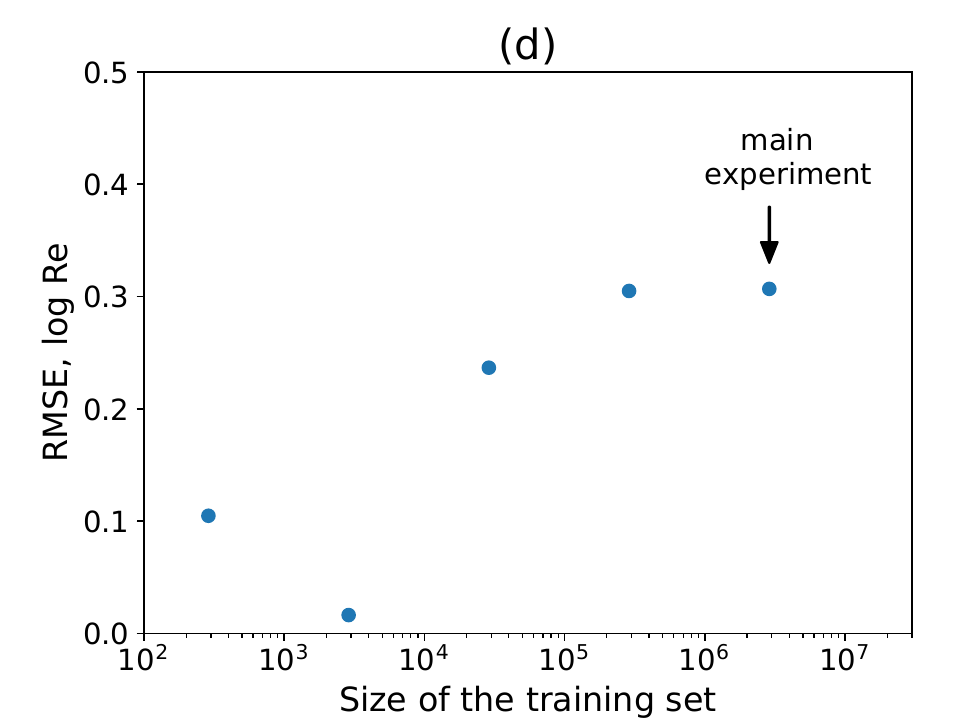}
    \includegraphics[width=0.33\linewidth]{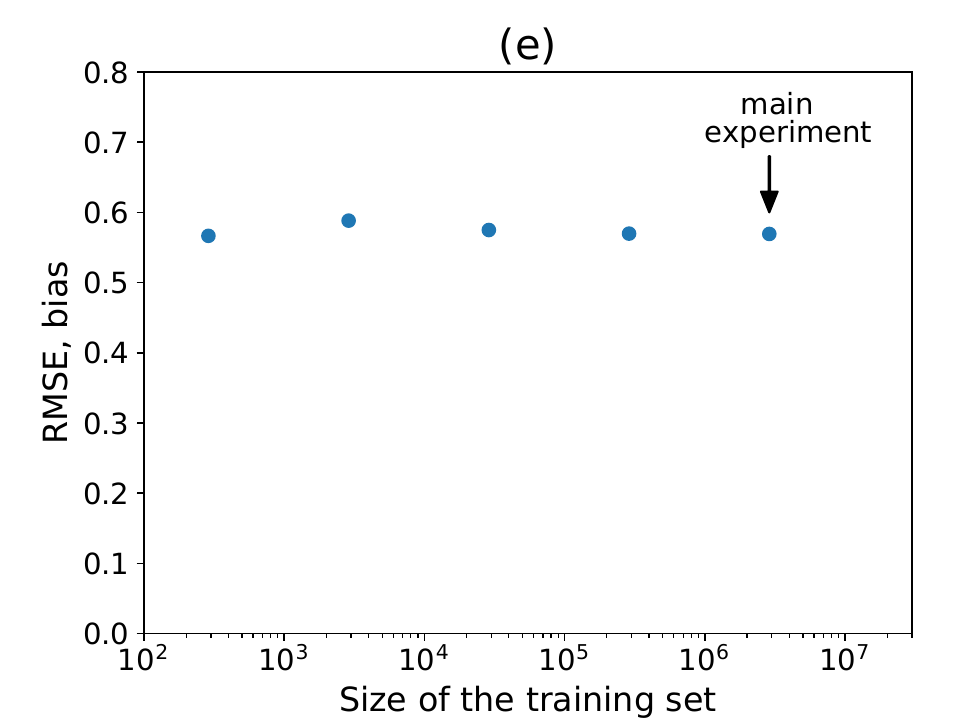}
    \caption{Ablation study results. Panel (a): RMSE of the state inference, (b): ELBO of the test set, (c) RMSE of the estimated forcing parameter $F$, (d) RMSE of the estimated log-Reynolds number $\log Re$, and (e): RMSE of the estimated bias field. The error bars for five initial conditions are smaller than the size of each marker.}
    \label{fig:ablation}
\end{figure}

A practical question for amortized smoothing is: \emph{how many observation windows are needed to obtain useful
performance under a fixed training compute budget?} In many deployments, collecting $10^{6}$--$10^{7}$ windows may be
impractical. We therefore repeat the Kolmogorov-flow experiment under the \emph{realistic} observation network while
subsampling the training set to
$N\in\{2.88\times10^6\ (100\%),\,2.88\times10^5\ (10\%),\,2.88\times10^4\ (1\%),\,2.88\times10^3\ (0.1\%),\,2.88\times10^2\ (0.01\%)\}$.
To isolate data effects from optimization cost, we keep the number of gradient updates fixed across all settings and increase
the number of epochs for smaller datasets. All metrics are computed on a held-out test set.

Fig.~\ref{fig:ablation} shows that the Flow variant degrades smoothly as N decreases under compute-budgeted training. Performance with $N \approx 10^5$ windows is already close to the main experiment, and the most noticeable
increase in state RMSE occurs when going below $\approx 10^4$ windows (Fig.~\ref{fig:ablation}(a)). Notably, even in the extreme low-data regime $N = 2.88\times 10^2 (288)$ windows, training remains stable and the Flow model’s state RMSE remains below the oracle 4D-Var reference in Fig.~\ref{fig:ablation}(a) (dotted line). Figure~\ref{fig:ablation_visualize} provides representative reconstructions across dataset sizes and illustrates that the Flow model still recovers coherent large-scale vortical structure at $N = 288$ (with the Gauss variant shown for comparison).

For parameter learning, the forcing amplitude $F$ is robustly estimated across dataset sizes (Fig.~\ref{fig:ablation} c). In contrast, the $\log Re$ estimate shows weaker sensitivity (and mild non-monotonicity) across dataset sizes (Fig.~\ref{fig:ablation} d). The bias-field RMSE varies only modestly across $N$ (Fig.~\ref{fig:ablation} e).
Overall, these results indicate that on the order of $10^5$ windows can suffice for near-main-experiment state accuracy in
this setting, while PR-Smoother remains competitive even with $10^2$--$10^3$ windows.

\section{Experimental Settings}
\label{sec:experimental_settings, appendix}
We summarize the experimental settings here. Code for reproducing the experiments is provided as supplementary material.

\paragraph{Computational resources.} We have used two clusters: a cluster of GH200 GPUs and another computer with an RTX5090. Each GPU has memory of 96GB (GH200) or 32GB (RTX5090).

\paragraph{Hyperparameters for training.} For all experiments, we have used Adam optimizer with default parameters. Hyperparameters employed in our experiments are provided in each section in the appendix. Each training is conducted with an RTX5090 (Lorenz96, 4D) or an GH200 GPU (Lorenz96, 40D and Kolmogorov flow).

\paragraph{Learning rate scheduling.}
For all the experiments, we warm-up the learning rate from $10^{-3}$ times the maximum lr, and it continuously increases up to the maximum lr for that experiment. The first 1\% (Lorenz96) / 5\% (Kolmogorov flow) of the total iterations is the warm-up period. Once it reaches the maximum, it continuously decreases exponentially, finishing the training with $0.1$ times the maximum lr at the final iteration. 

\paragraph{On the noise covariance matrix $S_t$ in Eq.~\eqref{eq:q-corr}}
For noisyFlow and noisyGauss we parameterize a residual correction $m_t(x_{t-1},y_{t:T})$ and use isotropic covariances with learned scales: $Q=\sigma_Q^2 I$ in the generative dynamics and $S_t=\sigma_S^2 I$ in the residual transition (shared across $t$). This preserves stability while eliminating manual noise calibration. The construction is nested: when $m_t\equiv 0$ and $S_t=Q$, the noisy variants reduce exactly to the corresponding flow/Gauss models.

\begin{table}[htbp]
    \centering
    \caption{Dataset summary.}
    \vspace{3mm}
    \begin{tabular}{cccc}
         & $X_{dim}$ & steps & $N_{data, train}$ \\ \hline
        Lorenz-96, 4D & $4$ & $1$ or $10$ & $1.0\times 10^7$ \\ 
        Lorenz-96, 40D & $40$ & $50$ & $1.0\times 10^7$ \\
        Kolmogorov flow & $16384$ & $10$ & $2.9\times 10^6$\\
    \end{tabular}
    \label{tab:dataset summary}
\end{table}

\subsection{Lorenz-96, 4D}
\label{sec:L96, 4D, appendix}

\begin{table}[htbp]
    \centering
    \caption{Training summary for Lorenz-96 (4D) experiment}
    \vspace{3mm}
    \begin{tabular}{ccccc}
         & lr & batch size & Epochs & train time per model\\ \hline
        Flow & $5\times 10^{-4}$ & 102400 & 100 & 1hr $\times$ 1GPU \\ 
        Gauss & $5\times 10^{-4}$ & 102400 & 100 & 1hr $\times$ 1GPU \\
        Mean Field & $5\times 10^{-4}$ & 102400 & 100 & 1hr $\times$ 1GPU \\
    \end{tabular}
    \label{tab:training summary for Lorenz96 4D}
\end{table}

\paragraph{Dataset:}
For the multimodality experiment in Sec.~\ref{sec:Lorenz96_multimodal} we use the four-dimensional Lorenz-96 system with periodic boundary
conditions. The continuous-time dynamics follow
\[
  \frac{dx_i}{dt}
  = (x_{i+1} - x_{i-2}) x_{i-1} - x_i + F,
  \quad i = 1,\dots,4,
\]
with cyclic indexing and forcing $F=8$.  We do not use any spin-up: the initial condition is drawn directly from a controllable Gaussian prior
\[
  X_1 \sim \mathcal{N}(0, (2.0)^2 I),
\]
so each coordinate has mean $0$ and standard deviation $2$. This prior is used both in the generative model and as the reference prior when evaluating approximate posteriors.

The dynamics are integrated deterministically with a fourth-order Runge-Kutta scheme.  We use an internal time step $\texttt{internal\_dt}=0.006$ and output states every five internal steps, giving an effective observation interval $\Delta t = 0.03$. For each draw of $X_1$ we run the simulation for $10$ observation steps, producing a latent path $X_{1:10} = (X_1,\dots,X_{10})$.

Observations are generated with a componentwise square and additive Gaussian noise,
\[
  Y_t = h(X_t) + \varepsilon_t,
  \qquad
  h(x) = x^{\!2},\quad
  \varepsilon_t \sim \mathcal{N}(0, \sigma^2 I_4),
\]
with observation noise level $\sigma = 3.0$.  To probe multimodality, we consider two conditioning problems built from
the same simulated paths: a $T=1$ setting that uses only the final time step $Y_{10}$, and a $T=10$ setting
that uses the full sequence $Y_{1:10}$.  In both cases the dynamics are strictly deterministic, so all
stochasticity comes from the Gaussian prior on $X_1$ and the observation noise.  For evaluation, we approximate the
“ground truth’’ posterior $p(X_{1:T}\!\mid\! Y_{1:T})$ with a very large bootstrap PF of $5\times 10^6$ particles and treat the resulting weighted particle cloud as an empirical reference distribution.
\paragraph{Models.}
For the 4-dimensional Lorenz-96 multimodality test in Sec.~\ref{sec:Lorenz96_multimodal} we instantiate PR-Smoother with three variational
families for the initial-state posterior $q_\phi(X_1 \!\mid\! Y_{1:T})$: a RealNVP flow (\texttt{Flow}), a diagonal Gaussian
(\texttt{Gauss}), and a MF variational smoother (\texttt{Mean-field}).  All three share the same deterministic
Lorenz-96 dynamics and squared observation operator described above; the only randomness in the model comes from the
Gaussian prior on $X_1$ and the Gaussian observation noise.

\paragraph{Observation encoder (periodic Conv1D).}
For the 4D Lorenz-96 experiment, we encode the observation window with a periodic 1-D convolutional network
$E_\phi$ (Conv1D variant). The input is a minibatch $y \in \mathbb{R}^{B \times T \times 4}$, where $T \in \{1,10\}$ is the
window length and the spatial index corresponds to the 4 cyclic Lorenz-96 coordinates.
We treat the time axis as channels and convolve along the (cyclic) state axis of length $4$ using circular padding.

\emph{Conv trunk.} Let $h^{(0)} = y$. We apply an 8-block residual Conv1D stack with channel schedule
\[
[T,\,128,\,128,\,128,\,128,\,128,\,128,\,128,\,C_{\mathrm{out}}],
\]
kernel size $k=3$, stride $1$, and circular padding (so the spatial length stays 4 throughout).
Each residual block has the form
\[
u = \mathrm{ReLU}(\mathrm{GN}(\mathrm{Conv1D}(h))), \qquad
h \leftarrow u + \mathrm{Skip}(h),
\]
where $\mathrm{GN}$ is GroupNorm with 1 group (equivalently, a per-location normalization over channels),
and $\mathrm{Skip}$ is identity when the channel counts match and otherwise a $1\times 1$ Conv1D projection.

\emph{ScaledTanh head.} The conv trunk output is passed through a $1\times 1$ Conv1D followed by a softly bounding
scaled $\tanh$ nonlinearity:
\[
\mathrm{Head}(v) \;=\; s \,\tanh(v/s),
\]
where $s>0$ is learnable and initialized to $s=10$. The $1\times 1$ head convolution is initialized at zero so the encoder
starts near zero output.

\paragraph{Flow model (RealNVP).}
For the flow model, we set $C_{\mathrm{out}}=64$, producing features $h \in \mathbb{R}^{B \times 64 \times 4}$.
We flatten $h$ to $\mathbb{R}^{B \times 256}$ and apply a 2-layer MLP to obtain a global context
$c \in \mathbb{R}^{256}$:
\[
c = \mathrm{MLP}( \mathrm{vec}(h) ), \qquad
\mathrm{MLP}:\; 256 \rightarrow 256 \xrightarrow{\mathrm{SiLU}} 256 .
\]
This context conditions all coupling layers of the conditional RealNVP flow
$q_\phi(x_1 \mid y_{1:T})$ on $\mathbb{R}^4$.

\paragraph{Gauss model.}
For the diagonal-Gaussian model, we set $C_{\mathrm{out}}=2$ and interpret the encoder output
$h \in \mathbb{R}^{B \times 2 \times 4}$ as $(\mu_\phi(y_{1:T}), \log \sigma_\phi(y_{1:T}))$ for each coordinate.
We sample $x_1 = \mu_\phi + \sigma_\phi \odot \varepsilon$ with $\varepsilon \sim \mathcal{N}(0,I_4)$ and evaluate the
closed-form diagonal-Gaussian log density.

\paragraph{Mean-field model.}
The \texttt{Mean-field} baseline removes the dynamical coupling from the variational family altogether and factorizes over
both time and coordinates:
\[
  q_\phi(X_{1:T} \!\mid\! Y_{1:T})
  = \prod_{t=1}^T \prod_{i=1}^4
      \mathcal{N}\!\bigl(x_{t,i};\,\mu_{t,i}(Y_{1:T}),
                                \sigma_{t,i}^2(Y_{1:T})\bigr).
\]
In implementation we reuse the same observation encoder as above, but instead of sampling $X_1$ and pushing it through
the dynamics, we attach lightweight MLP heads that take the encoder output and directly produce means and log-standard
deviations $\{\mu_{t,i}, \log \sigma_{t,i}\}$ for each $(t,i)$.  The Lorenz-96 model still enters the training objective
through the likelihood $p(Y_{1:T} \!\mid\! X_{1:T})$, but the variational posterior no longer enforces the Markov structure of
the SSM.  This makes the mean-field model a dynamics-agnostic reference for PR-Smoother in the low-dimensional,
multimodal Lorenz-96 test.
\paragraph{Ensemble Kalman filter (EnKF).}
As a classical Gaussian baseline we use an ensemble Kalman filter (EnKF) built on the same 4D Lorenz-96 dynamics and
squared observation operator as in the generative model.  We draw an ensemble of initial states
$\{X_1^{(n)}\}_{n=1}^{N_{\text{ens}}}$ i.i.d.\ from the Gaussian prior
$\mathcal{N}(0,4 I_4)$ and then propagate each member deterministically with the RK4 integrator, using the same
time-stepping scheme ($\texttt{internal\_dt}=0.006$, outputs every five internal steps).  At each observation time we
apply the nonlinear observation operator $h(x)=x^{\!2}$ to the forecast ensemble, compute the sample mean and covariance
of both the state and observation ensembles, and perform the standard EnKF analysis update in ensemble space with
observation noise covariance $R = 9 I_4$ (corresponding to $\sigma=3.0$).  For the $T=1$ setting this reduces to a
single forecast-analysis cycle; for $T=10$ we run EnKF for ten steps and treat the resulting ensemble of trajectories
$\{X_{1:T}^{(n)}\}$ as approximate samples from the posterior over the path.

\paragraph{Ensemble transform Kalman filter (ETKF).}
The ETKF baseline uses the same ensemble setup and dynamical/observation models as EnKF, but performs the analysis step
via a deterministic square-root transform instead of perturbing the observations.  In each window, we form the forecast
ensemble anomalies in state space, map them through the observation operator to obtain observation anomalies, and compute
the low-dimensional analysis transform in ensemble space that matches the Kalman posterior covariance.  Applying this
transform to the anomalies and updating the ensemble mean yields the analysis ensemble without additional observation
noise.  This typically reduces sampling noise and can give more stable posteriors in low-dimensional settings like the
4D Lorenz-96 system.  As with EnKF, we use the ETKF analysis ensemble over the window (either the final time only for
$T=1$ or the full sequence for $T=10$) as an empirical approximation to the posterior over the states.

\paragraph{Particle filter (PF).}
To obtain a high-accuracy reference distribution we run a bootstrap particle filter (PF) with the same prior, dynamics,
and observation model.  Particles are initialized from $X_1^{(n)} \sim \mathcal{N}(0,4 I_4)$, propagated deterministically
through the Lorenz-96 integrator, and reweighted at each observation time by the Gaussian likelihood
$p(Y_t \!\mid\! X_t^{(n)}) = \mathcal{N}(Y_t; h(X_t^{(n)}), 9 I_4)$ with $h(x)=x^{\!2}$.  We monitor the effective sample
size and resample adaptively to avoid severe weight degeneracy.  The resulting weighted particle cloud over $X_{1:T}$
(for $T=1$ and $T=10$) is treated as an empirical approximation to the true posterior $p(X_{1:T} \!\mid\! Y_{1:T})$ and is
used only for evaluation, i.e., as a “ground truth’’ reference when comparing the variational posteriors and ensemble
Kalman methods in this low-dimensional example.

\paragraph{Training variables:}
We do not assume any trainable model parameters in this experiment. Training variables are the parameters in the inference network, i.e., network parameters.

\subsection{Lorenz-96, 40D}
\label{sec:L96, 40D, appendix}

\begin{table}[htbp]
    \centering
    \caption{Training summary for Lorenz-96 experiment}
    \vspace{3mm}
    \begin{tabular}{ccccc}
         & lr & batch size & Epochs & train time per model\\ \hline
        Flow & $2\times 10^{-4}$ & 51200 & 400 & 18hr $\times$ 1GPU \\ 
        Noisyflow & $4\times 10^{-4}$ & 10240 & 50 & 16hr $\times$ 1GPU \\
        Gauss & $2\times 10^{-4}$ & 51200 & 400 & 16hr $\times$ 1GPU \\
        NoisyGauss & $4\times 10^{-4}$ & 10240 & 50 & 11hr $\times$ 1GPU \\
        Mean Field & $2\times 10^{-4}$ & 51200 & 400 & 14hr $\times$ 1GPU \\
    \end{tabular}
    \label{tab:training summary for Lorenz96}
\end{table}

\paragraph{Dataset:}
The dataset is synthesized by running simulations starting from a random field $X_1$:
\begin{align}
    X_{1,i} = 20\epsilon_i -10,\quad \epsilon_i\sim \mathcal{U}[0, 1],
\end{align}
where $i$ runs from 1 to 40. The time evolution is described by a set of ordinary differential equations:
\begin{equation}
\frac{dx_i}{dt} = (x_{i+1}-x_{i-2})\,x_{i-1} - x_i + F, \ i=1,\dots,N
\label{eq:L96, appendix}
\end{equation}
where we assume $F=8, N=40$ for our experiments. The time evolution is computed using the fourth-order Runge-Kutta (RK4) method with the step size of 0.006 and observations are produced every 0.03, both in the model time unit. We run the simulation for 100 observation steps (3.0 model time) and discard the first 50 observations (1.5 model time) as the spin-up period.

In the noisy dynamics case, we add a noise term at each step. The dynamics equation now reads:
\begin{align}
X_{t+1}=f(X_t)+\eta_t,
\end{align}
where $f$ represents the 5 steps of RK4 mapping function and $\eta_t\sim \mathcal{N}(\mu=0, \sigma^2=0.1^2)$ is the noise term added every $t=0.03$. This corresponds to the dynamics model of 
\begin{align}
    p(X_{t+1}\!\mid\! X_t)=\mathcal{N}[X_{t+1};f(X_t),Q],
\end{align}
where $Q=(0.1)^2I$. The dynamics noise is independent between time steps and the physical grids.

We adopt two observation operators to generate the observation data: (i) the linear operator $h(X_t) = X_t$, and (ii) the nonlinear operator $h(X_t) = min(X_t^2, 5)$. Both operators work componentwise. For either setting, we add a noise term $\epsilon_t\sim \mathcal{N}(\mu=0, \sigma^2=1.0^2)$ for the observation process. Noise added at an observation step is independent across steps. 

\paragraph{Models:} We have seven models: (i) flow model, (ii) noisy flow model, (iii) Gauss model, (iv) noisy Gauss model, (v) mean-field model, (vi) 4D-Var, and (vii) IEnKS.
\subparagraph{(i) Flow model}
We use a periodic 1-D conv encoder with global self-attention along the state axis $L$. Input is $y$ with shape $(B, steps, L)$, where $steps=50$ and $L=40$. Channels progress as $[steps, 128, 128, 128, 128, 128, 128, 128, out\_channels]$. Each hidden stage is a residual conv block: Conv1d(k=3, stride=1, circular padding) → LayerNorm(1) → ReLU; the skip is identity when channels match, otherwise a $1\times1$ projection. After selected 128-channel blocks ($attn\_at=(2, 5)$) we insert AttnBlock1D, which treats $L$ as the sequence: it transposes $(B, C, L) \rightarrow (B, L, C)$, applies LayerNorm, adds FourierPositionalEncoding1D(C) (sin/cos features tied to $L$ for circular domains), runs MultiheadAttention, adds a residual, applies another LayerNorm and a 2-layer MLP $(C \rightarrow 2C \rightarrow C)$ with GELU, then adds a second residual and transposes back to $(B, C, L)$. The stack ends with a $1\times 1$ Conv1d to $out\_channels$ and a ScaledTanh head initialized with bound 10, which softly bounds outputs while staying near-linear around zero. The design preserves spatial length, uses convolutions for local coupling, and uses attention to mix information globally across $L$ without downsampling.

\subparagraph{(ii) Noisy-flow model}

The noisy-flow variant uses exactly the same periodic Conv1D+attention encoder as the flow model to summarize the observations $Y_{1:T}$, and the same conditional RealNVP flow for the initial state $q_\phi(X_1 \!\mid\! Y_{1:T})$.  In addition, we instantiate a ``future-aware'' Transformer encoder over the observation sequence. This network takes the normalized observations $Y_{1:T}$ and returns per-step context vectors $c_t$ for $t=1,\dots,T-1$.  Given a rollout state $X_t$ from the physical dynamics and the context $(c_t, t/T)$, a small CNN/MLP “shock head” outputs a mean and diagonal log-variance for an additive correction $\delta_t$ to the dynamics noise. 

The full variational family is
\[
q_\phi(X_{1:T}\!\mid\! Y_{1:T}) = q_\phi(X_1\!\mid\! Y_{1:T}) \prod_{t=1}^{T-1}
q_\phi(\delta_t \!\mid\! X_t, Y_{t+1:T}),
\qquad
X_{t+1} = f_\theta(X_t) + \delta_t,
\]
with all log-densities and per-step KL terms evaluated in closed form.

\subparagraph{(iii) Gauss model}
The Gauss model keeps the same Conv1D+attention encoder but replaces the flow with a diagonal Gaussian posterior for the initial state.  The encoder outputs two channels of size $D=40$, which are interpreted as a mean and log standard deviation $(\mu(Y), \log\sigma(Y))$ for each grid point.  The variational family for the first time step is
\[
q_\phi(X_1\!\mid\! Y_{1:T}) = \mathcal{N}\!\bigl(X_1;\,\mu(Y_{1:T}),\operatorname{diag}(\sigma^2(Y_{1:T}))\bigr),
\]
implemented by a standard reparameterization $X_1=\mu+\sigma\odot\varepsilon$, $\varepsilon\sim\mathcal{N}(0,I)$.

\subparagraph{(iv) Noisy-gauss model}
The noisy-gauss model combines the diagonal-Gaussian initial posterior with the future-aware noisy transition family of Sec.~\ref{sec:noisy}.  As in the Gauss model, the encoder produces $(\mu,\log\sigma)$ and we sample $X_1\sim\mathcal{N}(\mu,\operatorname{diag}(\sigma^2))$.  In addition, the same Transformer future encoder used in the noisy-flow model provides per-step contexts $c_t$ from $Y_{t:T}$.  At each time step $t$, we form a feature vector from $(X_t, c_t, t/T)$ and feed it into a Gaussian shock head, which outputs a mean and diagonal log-variance for a correction $\delta_t$; this yields
\[
q_\phi(X_{1:T}\!\mid\! Y_{1:T})
= q_\phi(X_1\!\mid\! Y_{1:T})\prod_{t=1}^{T-1}
\mathcal{N}\!\bigl(\delta_t;\,m_t(X_t,c_t),\operatorname{diag}(s_t^2(X_t,c_t))\bigr),
\quad X_{t+1}=f_\theta(X_t)+\delta_t.
\]
The implementation shares the same Transformer and shock head as the noisy-flow model and accumulates the exact log-density of the path under this factorization.

\subparagraph{(v) Mean-field model}
The mean-field baseline removes the dynamical coupling from the variational family and factorizes over time:
\[
q_\phi(X_{1:T}\!\mid\! Y_{1:T}) = \prod_{t=1}^T q_{\phi,t}(X_t\!\mid\! Y_{1:T}).
\]
In code, this is implemented by a separate smoother class that uses the same observation encoder family as PR-Smoother but attaches lightweight heads that output $(\mu_t,\log\sigma_t)$ for each $(t,i)$ directly, without pushing samples through the dynamics inside $q_\phi$.  The physical model $f_\theta$ is still used in the likelihood and in the ELBO, but the variational posterior no longer enforces the Markov structure of the SSM, making it a useful “dynamics-agnostic’’ reference for PR-Smoother.

\subparagraph{(vi) 4D-Var model}
As a variational baseline we implement a windowed 4D-Var solver that keeps the Lorenz-96 dynamics and observation operator on-graph and performs MAP estimation of the latent trajectory. For a window of length T the method optimizes either only the initial state x1  (strong-constraint mode, used when the dynamics are noiseless) or the entire trajectory x1:T  (weak-constraint mode, used in the dynamics-noise setting), with the Lorenz-96 ODE integrated between observation times using an RK4 scheme with forcing F=8 and N=40 variables. The objective is the standard quadratic 4D-Var cost: a Gaussian background term on x1  with isotropic prior variance matched to the data generation, a Gaussian observation-misfit term with the same noise variance as in the generative model, and in the weak-constraint case an additional quadratic penalty on model-error increments xt+1-f(xt). Optimization is done in PyTorch via L-BFGS for the strong-constraint case and Adam for the weak-constraint case, unrolling the Lorenz-96 integrator inside the computation graph so that gradients are obtained by automatic differentiation. We use multiple random restarts per window and report the reconstructed trajectory corresponding to the lowest cost as the 4D-Var estimate.

\subparagraph{(vii) IEnKS model}
The IEnKS baseline follows the iterative ensemble Kalman smoother formulation of \citet{IEnKS}, using the same Lorenz-96 dynamics and observation operators as in the main experiment. For each assimilation window, we propagate an ensemble of trajectories through the RK4-discretized Lorenz-96 model and perform fixed-lag smoothing in ensemble space. The augmented state contains only the physical state xt (no bias or parameter components in this experiment); ensemble members are initialized from an isotropic Gaussian prior. At each IEnKS iteration we (i) roll the ensemble forward over the window, (ii) map each member to observation space through h (either identity or the nonlinear $h(x)=min(x^2,5))$, (iii) form anomalies and the associated design matrix in ensemble space, and (iv) solve the small linear system $(I+D^\top D)\delta w=D^\top d$ to obtain a Gauss-Newton-style update of the ensemble mean and anomalies. No localization is applied since the state dimension is moderate. We use 512 ensemble members for all Lorenz-96 experiments. Sliding this window along the sequence yields an analysis trajectory given by the ensemble mean at each time.

\paragraph{Training variables:}
We do not assume any trainable model parameters in this experiment. Training variables are the parameters in the inference network, i.e., network parameters.

\subsection{Kolmogorov Flow}
\label{sec:Kolmogorov, appendix}

\paragraph{Problem setting.}
We conduct an experiment on 16,384-dimensional Kolmogorov flow. Kolmogorov flow is a standard 2-D incompressible Navier-Stokes benchmark in a periodic box, driven by a spatially sinusoidal body force. In nondimensional form it reads 
\begin{align}
\partial_t \mathbf{u} + (\mathbf{u}\cdot\nabla)\mathbf{u}
&= -\nabla p + \frac{1}{\mathrm{Re}}\,\Delta \mathbf{u}
  + F \sin(k y)\,\mathbf{e}_x, \label{eq:Kolmogorov dynamics1}\\
\nabla \cdot \mathbf{u} &= 0.
\label{eq:Kolmogorov dynamics2}
\end{align}
where $F$ is the forcing amplitude, $k$ the forcing wavenumber, and $Re$ the Reynolds number. We use this problem in a $128\times128$ discretization (state dimension $16,384$) to evaluate high-dimensional data assimilation.

\begin{table}[htbp]
    \centering
    \caption{Training summary for Kolmogorov experiment}
    \vspace{3mm}
    \begin{tabular}{ccccc}
         & lr & batch size & Epochs & train time per model\\ \hline
        Flow & $3\times 10^{-4}$ & 288 & 4 & 19hr $\times$ 1GPU \\ 
        Noisyflow & $3\times 10^{-4}$ & 96 & 1 & 20hr $\times$ 1GPU \\
        Gauss & $3\times 10^{-4}$ & 384 & 5 & 19hr $\times$ 1GPU \\
        NoisyGauss & $3\times 10^{-4}$ & 96 & 1 & 20hr $\times$ 1GPU \\
        Mean Field & $3\times 10^{-4}$ & 384 & 6 & 20hr $\times$ 1GPU \\
    \end{tabular}
    \label{tab:training summary for Kolmogorov}
\end{table}

\paragraph{Dataset:}
The Kolmogorov experiment uses the 2‑D incompressible Navier-Stokes “Kolmogorov flow” benchmark on a periodic $128\times 128$ grid (state dimension $16,384$). The continuous equations are given in Eqs.~\eqref{eq:Kolmogorov dynamics1} and \eqref{eq:Kolmogorov dynamics2}, with sinusoidal body forcing of amplitude F and Reynolds number Re. We generate trajectories by numerically integrating this system and record windows of T=10 steps; the training set contains approximately $2.9\times 10^6$ windows. We consider four observation patterns---full, half, sparse, and mixed realistic coverage---and always add an unknown, per‑sensor additive bias as well as i.i.d. Gaussian measurement noise. The true forcing and Reynolds number are $F=1.0$ and $Re=1000$; all algorithms are initialized at $F=0.1$ and $Re=100$, and we learn these parameters jointly with the state.

\paragraph{Models.}
We have seven models: (i) flow model, (ii) noisy flow model, (iii) Gauss model, (iv) noisy gauss model, (v) mean-field model, (vi) 4D-Var, and (vii) IEnKS. All four PR-Smoother variants for Kolmogorov share the same
spectral Gaussian prior (see the paragraph below) on the initial state, the Kolmogorov
Navier-Stokes dynamics, and the observation likelihood and
bias parameterization described above. They differ only in the
choice of initial-state posterior $q_\phi(x_1 \!\mid\! y_{1:T})$ and
whether the future-aware noisy-dynamics corrections
$q_\phi(x_{t+1} \!\mid\! x_t, y_{t+1:T})$ are enabled. Training
hyperparameters for each variant are summarized in
Table~\ref{tab:training summary for Kolmogorov}.

\paragraph{Spectral Gaussian prior on the initial vorticity field.}
Let $x_1\in\mathbb{R}^{H\times W}$ denote the initial 2-D vorticity field on a periodic
$H\times W$ grid (in our experiments $H=W=128$). We use a zero-mean Gaussian prior
$p(x_1)=\mathcal{N}(0,C_0)$ whose precision is diagonal in the 2-D discrete Fourier basis.

Let $\mathcal{F}$ denote the unitary 2-D discrete Fourier transform (DFT), and define
$\hat{x}_1 \coloneqq \mathcal{F}\{x_1\}\in\mathbb{C}^{H\times W}$.
For each Fourier mode $k=(k_x,k_y)$ (using FFT frequency indexing), define the (negative)
discrete Laplacian eigenvalue
\begin{equation}
\lambda(k)\;=\;4\sin^2\!\left(\pi k_x/W\right)\;+\;4\sin^2\!\left(\pi k_y/H\right).
\end{equation}
We then define the mode-wise precision (``power'') as
\begin{equation}
q(k)\;=\;\big(\alpha + \beta\,\lambda(k) + \gamma\,\lambda(k)^2\big)\,\cdot\,b(k),
\qquad \alpha,\beta,\gamma>0,
\end{equation}
where $b(k)$ is an optional high-wavenumber ``spectral wall'':
\begin{equation}
b(k)\;=\;1 + (B-1)\,\sigma\!\big(\eta\,(r(k)-k_{\mathrm{cut}})\big), \qquad
r(k)=\sqrt{k_x^2+k_y^2},
\end{equation}
with $\sigma(\cdot)$ the logistic sigmoid, and constants $k_{\mathrm{cut}}=16$ and $\eta=12$.

With these definitions, the log-density is
\begin{equation}
\log p(x_1)\;=\;
-\frac{1}{2}\sum_{k} q(k)\,|\hat{x}_1(k)|^2
+\frac{1}{2}\sum_{k}\log q(k)
-\frac{HW}{2}\log(2\pi),
\end{equation}
where the sums range over all discrete Fourier modes.

\paragraph{Hyperparameters.}
We parameterize $\alpha,\beta,\gamma$ and $B$ in log-space to enforce positivity and initialize
$\alpha=10^{-3}$, $\beta=1$, $\gamma=0.1$, and $B=1+\exp(\log 5)\approx 6$; unless stated
otherwise, these scalars are optimized jointly with the other generative parameters.

\subparagraph{(i) Flow model.}
The flow model uses a conditional 2-D NF
for the initial state,
$q_\phi(x_1 \!\mid\! y_{1:T})$, instantiated as a multi-scale coupling
flow acting on the $128\times128$ latent field.
The Kolmogorov encoder produces a 2-channel feature map
$h \in \mathbb{R}^{B\times 2\times 128\times128}$, which is
projected by a $1\times1$ convolution to a 64-channel context
map that conditions all coupling layers. Starting from
$z \sim \mathcal{N}(0,I)$, the flow generates
$x_1 = g_\phi(z; c_{\rm global})$ and returns
$\log q_\phi(x_1 \!\mid\! y_{1:T})$ via the change-of-variables
formula. The dynamics in the variational family are
deterministic: trajectories are obtained by rolling $x_1$ forward
through the prescribed Kolmogorov integrator
$x_{t+1} = f_\theta(x_t)$ without additional per-step corrections.
Thus flow optimizes the deterministic ELBO, in which
transition terms cancel and only the spectral prior and
observation likelihood appear.

\paragraph{(ii) Noisyflow model.}
The noisyflow model uses the same conditional flow
$q_\phi(x_1\mid y_{1:T})$ as the flow model for the initial state,
and adds Gaussian corrections at subsequent steps:
\[
p_\theta(x_{t+1}\mid x_t)
= \mathcal{N}\!\left(f_\theta(x_t),\sigma_Q^2 I\right),
\qquad
q_\phi(x_{t+1}\mid x_t,y_{t+1:T})
= \mathcal{N}\!\left(f_\theta(x_t)+m_t,\sigma_S^2 I\right).
\]
The correction head predicts the mean field $m_t$ from the current
state, observation context, and time encoding. The process-noise
scale $\sigma_Q$ and variational correction scale $\sigma_S$ are
separate learned scalars shared across spatial coordinates, time
steps, and observation windows.

\subparagraph{(iii) Gauss model.}
The Gauss model keeps the deterministic Kolmogorov
rollout but replaces the flow-based initial-state posterior with a
diagonal Gaussian. The encoder output
$h \in \mathbb{R}^{B\times 2\times 128\times128}$ is interpreted
channelwise as per-pixel mean and log-standard deviation,
\[
  (\mu,\log\sigma)(i,j) = h_{:, :, i,j},
\]
so that
$q_\phi(x_1 \!\mid\! y_{1:T}) =
\mathcal{N}\big(x_1; \mu(y_{1:T}), \mathrm{diag}(\sigma^2(y_{1:T}))\big)$
factorizes over grid points. Samples $x_1$ are then propagated
deterministically with $x_{t+1}=f_\theta(x_t)$ as in
flow. This yields a mean-field style baseline in which
all non-Gaussian structure must be represented through the
encoder, while the variational family itself remains unimodal.

\subparagraph{(iv) NoisyGauss model.}
The noisyGauss model combines the diagonal-Gaussian
initial posterior of gauss with the noisy-dynamics
variational transitions of noisyflow. The first state
$x_1$ is drawn from the factorized Gaussian posterior defined
by the encoder, and subsequent states follow
\[
  q_\phi(x_{t+1} \!\mid\! x_t, y_{t+1:T})
  = \mathcal{N}\big(f_\theta(x_t) + m_t, \Sigma_t\big),
\]
with $(m_t,\Sigma_t)$ again parameterized by the
future-aware context maps and the 2-D Gaussian shock head.
The generative prior over transitions is the same Kolmogorov
dynamics with Gaussian process noise used in noisyflow,
and the ELBO includes the same per-step KL regularization
toward the physical noise model. Compared to noisyflow,
this variant tests the impact of restricting $q_\phi(x_1 \!\mid\!
y_{1:T})$ to a diagonal Gaussian while still allowing future
observations to correct the trajectory through noisy dynamics.

\subparagraph{(v) Mean-Field model}
The mean-field VI baseline keeps the same Kolmogorov dynamics, observation operators, and trainable parameters $(F,\mathrm{Re},\text{biases})$ as PR-Smoother, but uses a fully factorized diagonal-Gaussian posterior over the state trajectory. Concretely, it assumes
\[
q_\phi(X_{1:T}\!\mid\! Y_{1:T})
= \prod_{t=1}^T\prod_{i=1}^{H\times W}
\mathcal{N}\bigl(X_{t,i};\,m_{t,i}(Y_{1:T}),\,s_{t,i}^2(Y_{1:T})\bigr),
\]
so that all spatial and temporal correlations are ignored in $q_\phi$. In implementation, we reuse the Kolmogorov encoders but attach simple heads that output per-pixel means and log-standard deviations for each time step, without pushing samples through the dynamics inside the variational family. The physical Kolmogorov model is still used in the likelihood and in the Gaussian prior on $X_1$, but the posterior no longer enforces the Markov structure of the SSM, making this a “dynamics-agnostic’’ reference point for PR-Smoother in the high-dimensional Kolmogorov experiment (see Table~\ref{tab:training summary for Kolmogorov}).

\subparagraph{(vi) 4D-Var}
For the high-dimensional Kolmogorov-flow experiment we also include a 4D-Var baseline that performs joint state, bias, and parameter estimation over each window. The state $x_t$ is the 2-D vorticity field on a $128\times 128$ grid (flattened to dimension 16{,}384); the dynamics are given by the incompressible Navier-Stokes ``Kolmogorov flow'' model with unknown Reynolds number $Re$ and forcing amplitude $F$. For each window, we construct a quadratic 4D-Var cost consisting of (i) a Gaussian prior on the initial state $x_1$ in the same spectral basis as the PR-Smoother prior, (ii) Gaussian priors on the per-sensor additive biases (with variance set to the bias prior used when generating the data), (iii) Gaussian priors on the physical parameters $\theta=(\log Re, F)$, and (iv) a Gaussian observation-misfit term matching the noise variance used to synthesize the observations. In the strong-constraint variant (used in our main Kolmogorov results) we treat the dynamics as deterministic and optimize over $x_1$, the spatial bias field, and $\theta$; in the weak-constraint variant the trajectory $x_{1:T}$ is also optimized and a quadratic model-error penalty $\sum_t \|x_{t+1} - f_{\theta}(x_t)\|^2$ is added. Observations under the full / half / sparse / mixed operators are embedded into a full grid with masks, and only observed locations contribute to the likelihood term. All variables are optimized jointly by gradient-based methods (L-BFGS for the strong-constraint solver, Adam for the weak-constraint solver), with the Kolmogorov integrator unrolled inside the graph so that gradients with respect to both the state and $\theta$ are obtained by automatic differentiation.

We also report an ``oracle'' 4D-Var variant in which the true physical parameters and the bias parameters are provided and held fixed throughout optimization. Concretely, we remove $\theta$ and $b$ from the decision variables (equivalently omitting the parameter-prior term) and evaluate the unrolled Kolmogorov dynamics with $(\theta, b)=(\theta^\star, b^\star)$. The resulting oracle objective therefore optimizes only over the flow state (i.e., $x_1$).

\subparagraph{(vii) IEnKS}
The Kolmogorov IEnKS baseline uses the same implementation as in Appendix~\ref{sec:L96, 40D, appendix} but augments the state with both sensor biases and physical parameters. We use 16 ensemble members for all Kolmogorov experiments. The maximum number of iterations is $100$. In the non-oracle setting, the augmented vector is $[x_t,b,\theta]$, where $x_t$ is the 16,384-dimensional flow field, $b$ is a per-gridpoint additive bias field (restricted to observed locations under half / sparse / mixed coverage), and $\theta=(\log Re,F)$ are global parameters. For each window, we draw an ensemble of these augmented states from a factorized Gaussian prior: the initial flow state from the spectral prior, biases from a zero-mean Gaussian with variance matching the bias prior used in data generation, and $\theta$ from a broad Gaussian centered at the shared initialization $Re=100$, $F=0.1$. Ensemble members are then propagated through the Kolmogorov dynamics using member-specific parameters passed explicitly to the Kolmogorov solver (non-oracle mode). At each IEnKS iteration we build anomalies of the augmented state and of the corresponding modeled observations, and perform an SDA-style ensemble-space analysis that updates the full augmented state; this yields smoothed estimates of $x_t$, $b$, and $\theta$ over the window.

As with 4D-Var, we also include an “oracle” IEnKS variant in which the true parameters are provided (and, in the oracle setting we report here, the bias field is also provided), so the smoother only estimates the flow state over the window. To mitigate spurious long-range correlations in the high-dimensional Kolmogorov state, we use covariance localization within the ensemble-space analysis in both the oracle and non-oracle settings. We sweep the localization radius \texttt{kflow\_loc\_cutoff} over $\{4,8,16,32\}$ for both settings and report results using the best-performing value. Fig.~\ref{fig:IEnKS_localization} shows the calibration results for the localization radius. We chose the radius that yielded the lowest state RMSE. We additionally report a small ensemble-size sensitivity check $M\in \{16,32,48\}$ in Fig.~\ref{fig:IEnKS_members}. Given the substantial runtime growth with $M$, we use $M=16$ in all Kolmogorov IEnKS experiments.

\begin{figure}
    \centering
    \includegraphics[width=0.24\columnwidth]{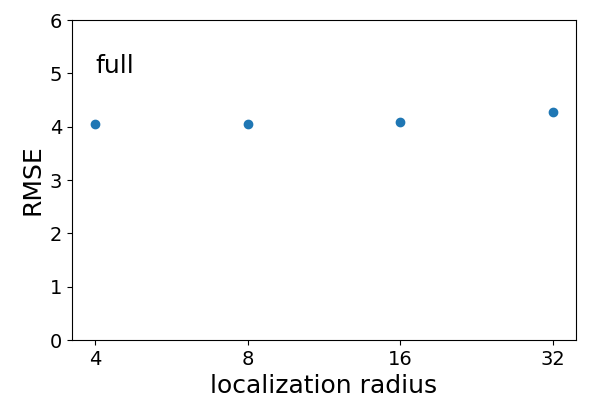}
    \includegraphics[width=0.24\columnwidth]{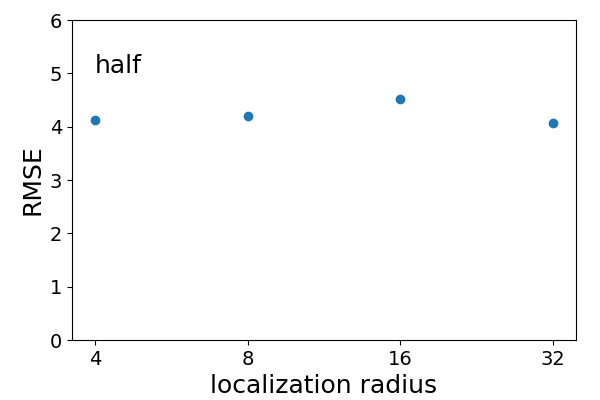}
    \includegraphics[width=0.24\columnwidth]{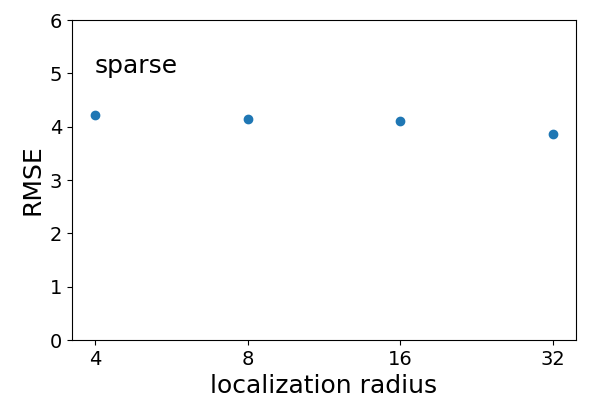}
    \includegraphics[width=0.24\columnwidth]{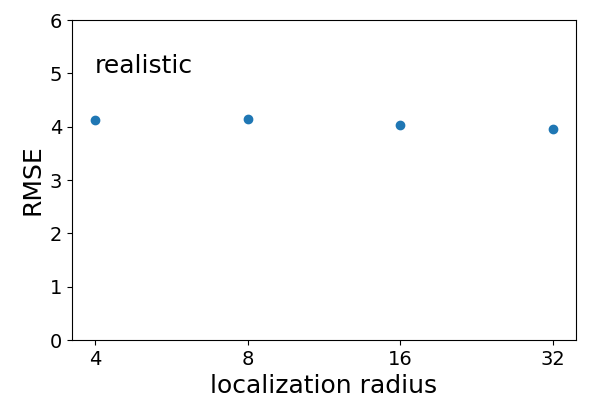}
    \includegraphics[width=0.24\columnwidth]{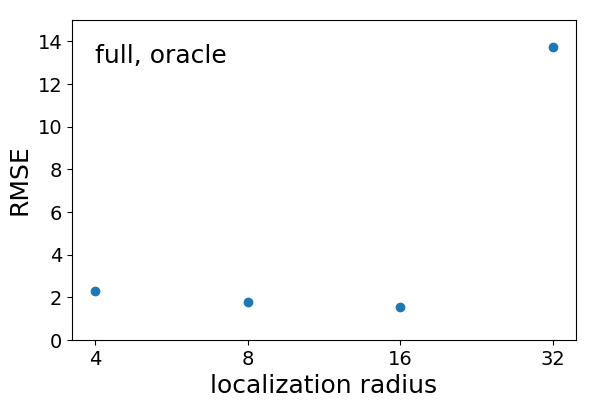}
    \includegraphics[width=0.24\columnwidth]{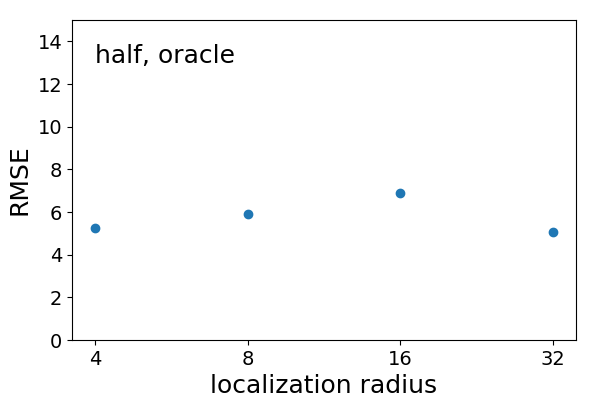}
    \includegraphics[width=0.24\columnwidth]{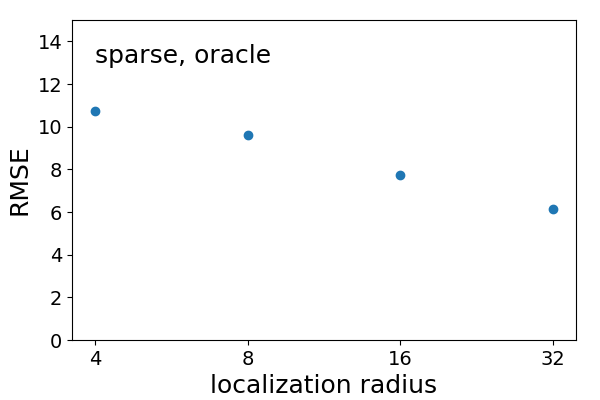}
    \includegraphics[width=0.24\columnwidth]{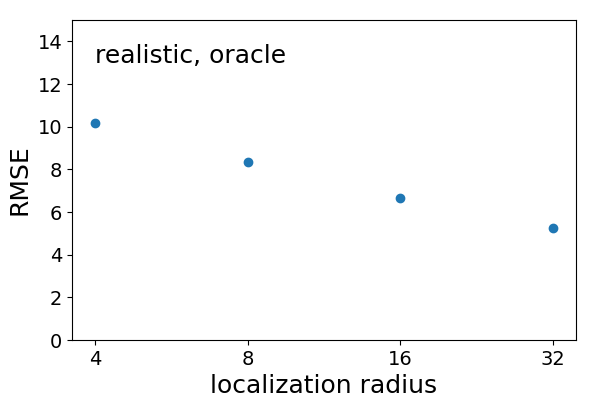}
    \caption{Localization radius calibration results for IEnKS.}
    \label{fig:IEnKS_localization}
\end{figure}

\begin{figure}
    \centering
    \includegraphics[width=0.49\columnwidth]{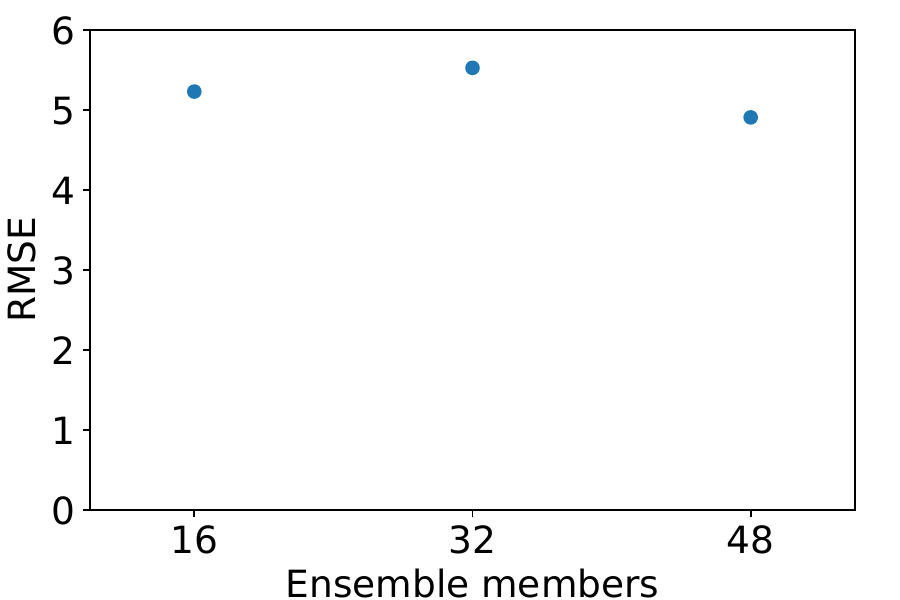}
    \includegraphics[width=0.49\columnwidth]{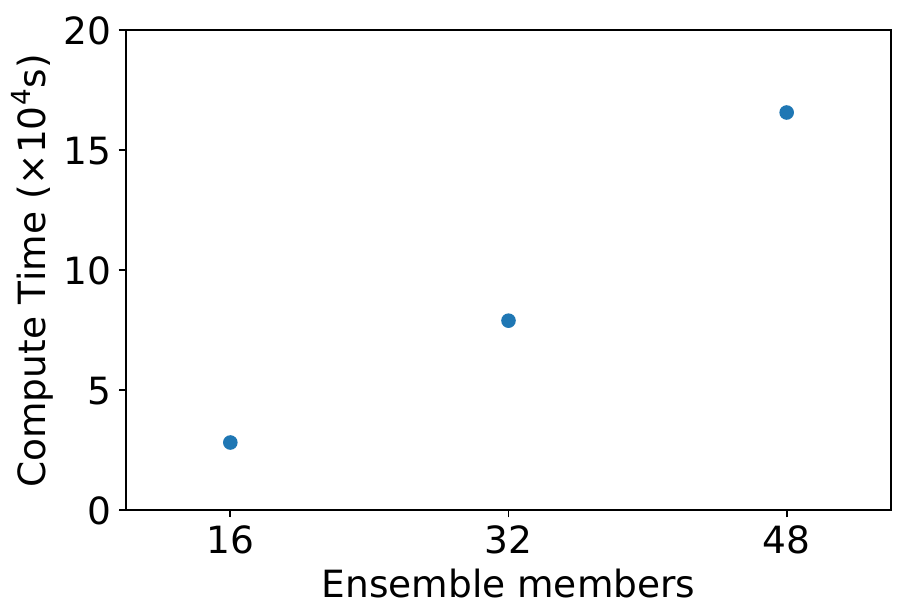}
    \caption{IEnKS ensemble-size sensitivity in the ``realistic'' observation network in the Kolmogorov-flow experiment (oracle setting). Left: mean state RMSE versus the number of ensemble members $M\in \{16,32,48\}$. Right: corresponding wall‑clock compute time to run IEnKS on the same evaluation set. Both metrics are averaged over four distinct initial conditions.}
    \label{fig:IEnKS_members}
\end{figure}

\paragraph{Training variables:}
Each model is able to model biases in the observations and the model parameters ($F$ and $\log Re$). Training variables are these model parameters in the generative model as well as the parameters in the inference network, i.e., network parameters.

\subsection{Runtime and Amortization}
\label{app:runtime}

PR-Smoother trades a one-time offline training phase for fast amortized test-time smoothing.
In this subsection we summarize the runtime characteristics of our implementation and compare
the per-window inference cost of PR-Smoother to non-amortized data-assimilation baselines.

\paragraph{Offline training cost.}
Tables~\ref{tab:training summary for Lorenz96 4D}--\ref{tab:training summary for Kolmogorov} list the training hyperparameters and wall-clock times per variant. Across our experiments, training a single PR-Smoother or baseline variant takes about $1$ GPU-hour on a RTX5090 for the 4D Lorenz-96 multimodality test, between $11$ and $18$ GPU-hours on a GH200 GPU for the 40D Lorenz-96 experiment, and between $19$ and $20$ GPU-hours on a GH200 GPU for the Kolmogorov-flow experiment. These costs are incurred once per trained model and can be amortized over a large number of future assimilation windows.

\paragraph{Per-window inference cost.}
After training, PR-Smoother performs smoothing on a new observation window with a single forward pass of the encoder and variational networks plus rollout of the physical model. In contrast, 4D-Var and IEnKS solve an optimization for each window separately. Table~\ref{tab:runtime-per-window} reports mean wall-clock time per assimilation window for PR-Smoother and for the 4D-Var and IEnKS baselines in representative settings.

\begin{table}[ht]
    \centering
    \caption{Wall-clock time per assimilation window (seconds). 
    Measurements use batch size $1$ and are averaged for five runs. PR-Smoother uses the trained models from Sec.~\ref{sec:experiments}.}
    \label{tab:runtime-per-window}
    \begin{tabular}{llc}
        \toprule
        Experiment & Method & Time / window (s) \\
        \midrule
        40D Lorenz-96, linear obs. 
            & PR-Smoother (Flow) & 0.37 \\
            & 4D-Var                  & 507.02 \\
            & IEnKS (512 members) & 14.53 \\
        \midrule
        40D Lorenz-96, nonlinear obs. 
            & PR-Smoother (Flow) & 0.38 \\
            & 4D-Var                  & 1480.55 \\
            & IEnKS (512 members) & 14.86 \\
        \midrule
        Kolmogorov flow, mixed coverage 
            & PR-Smoother (Flow) & 1.89 \\
            & 4D-Var (oracle)                        & 121.24 \\
            & IEnKS (oracle, 16 members, $r_{loc}$=16) & 1779.16 \\
        \bottomrule
    \end{tabular}
\end{table}

In all cases, the offline training cost is paid once, while the amortized per-window cost of PR-Smoother remains fixed as the number of assimilation windows grows, unlike the non-amortized 4D-Var and IEnKS baselines that must solve a new optimization or smoother problem for each window.


\section{Inference Examples in the Lorenz-96, 4D Experiment}
Here we show histograms for inference results obtained in Lorenz96, 4D experiment.
\label{sec:Lorenz96, inference examples, appendix}

\begin{figure}
    \centering
    \includegraphics[width=0.24\columnwidth]{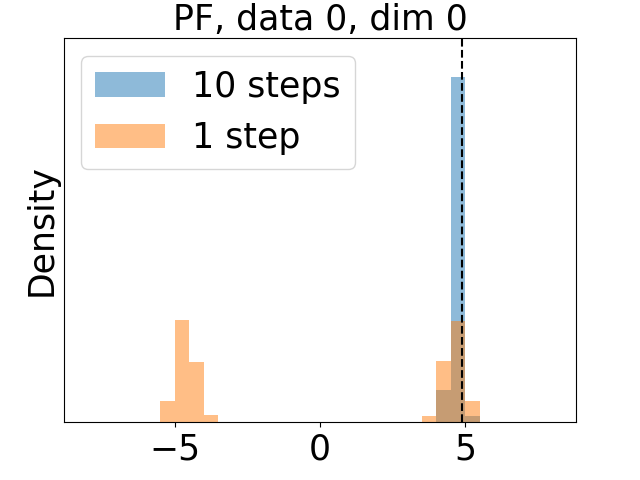}
    \includegraphics[width=0.24\columnwidth]{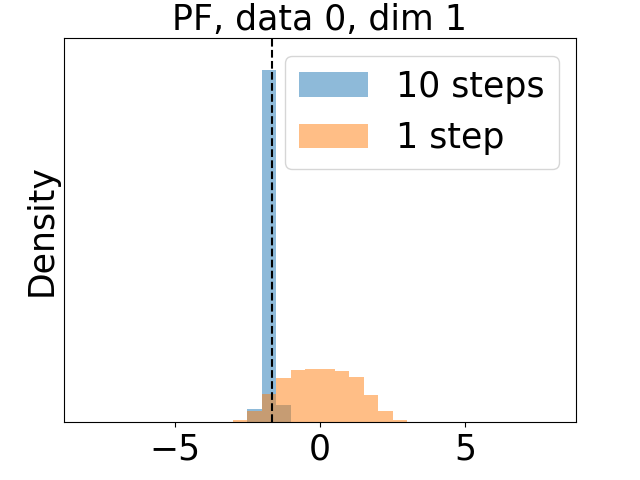}
    \includegraphics[width=0.24\columnwidth]{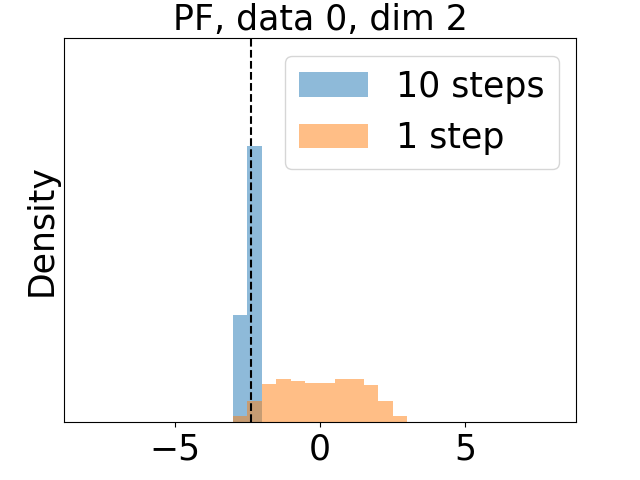}
    \includegraphics[width=0.24\columnwidth]{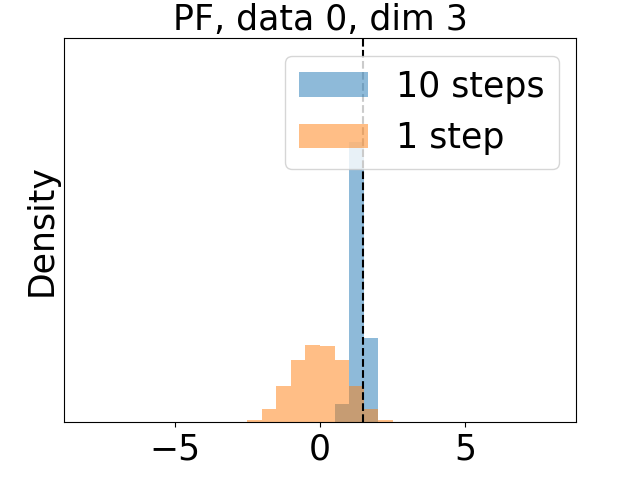}
    \includegraphics[width=0.24\columnwidth]{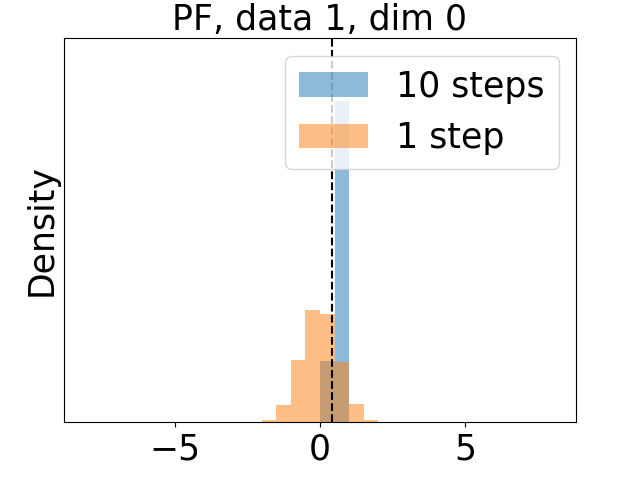}
    \includegraphics[width=0.24\columnwidth]{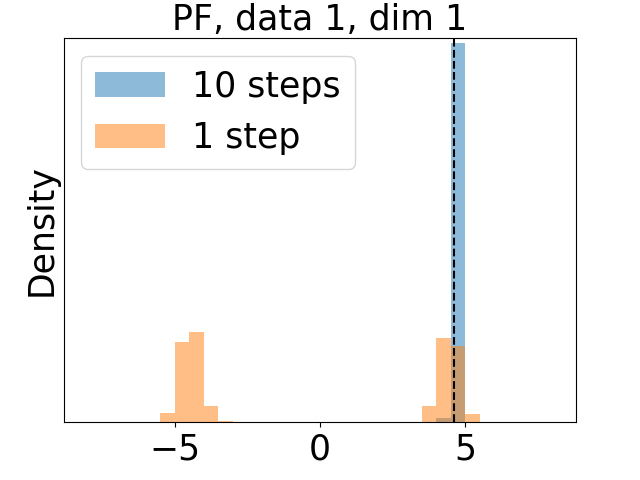}
    \includegraphics[width=0.24\columnwidth]{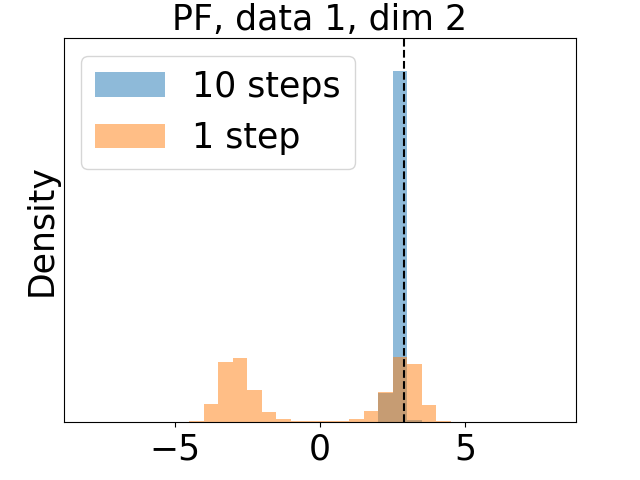}
    \includegraphics[width=0.24\columnwidth]{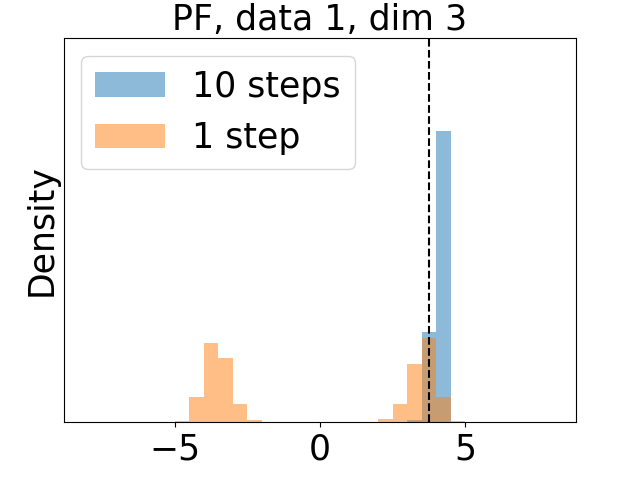}
    \includegraphics[width=0.24\columnwidth]{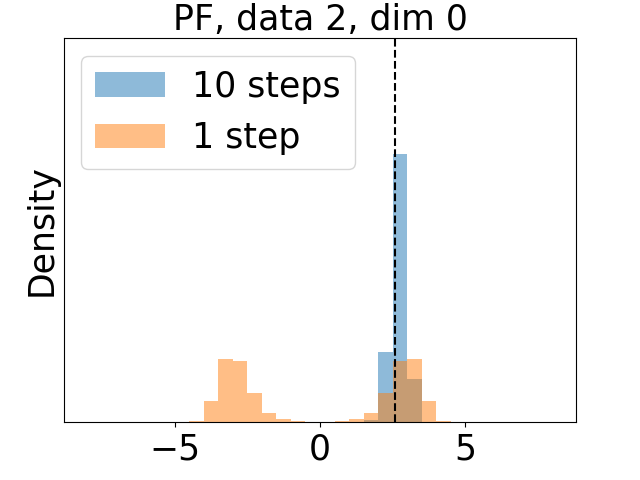}
    \includegraphics[width=0.24\columnwidth]{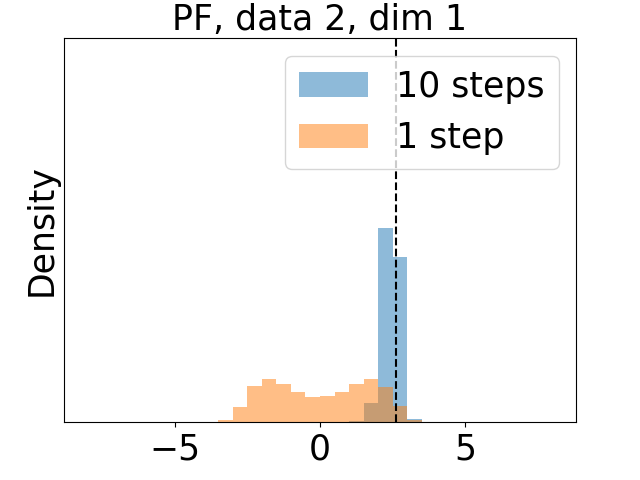}
    \includegraphics[width=0.24\columnwidth]{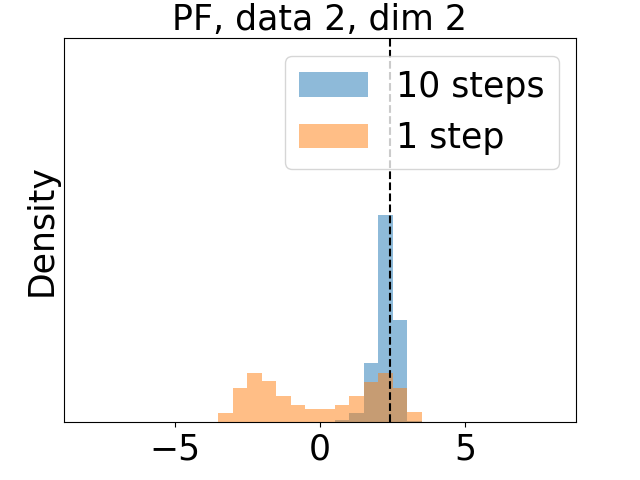}
    \includegraphics[width=0.24\columnwidth]{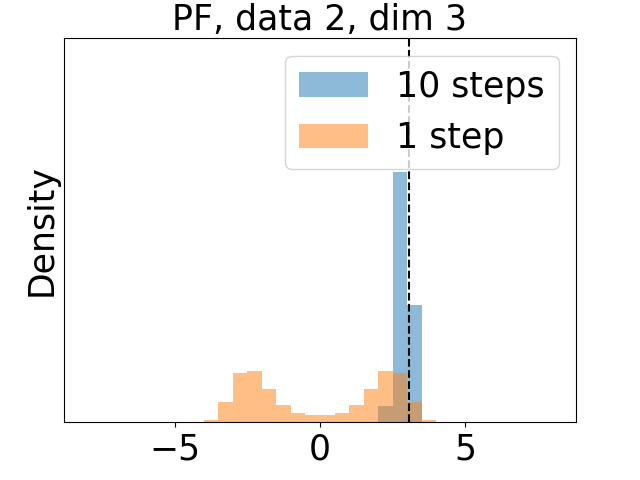}
    \includegraphics[width=0.24\columnwidth]{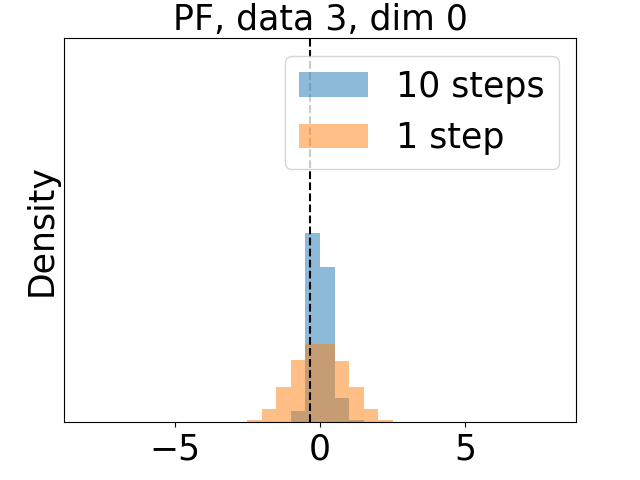}
    \includegraphics[width=0.24\columnwidth]{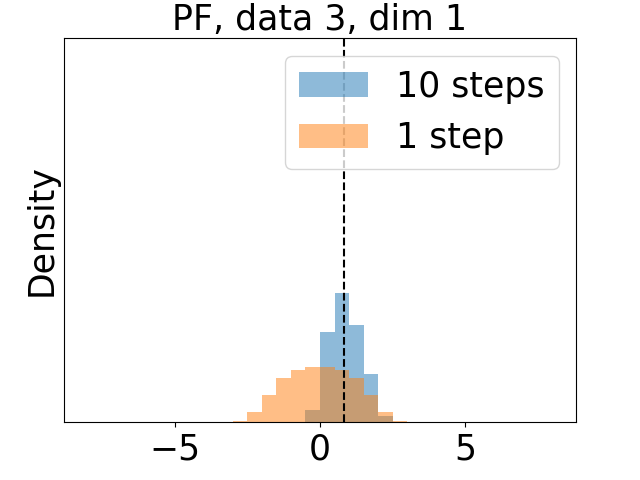}
    \includegraphics[width=0.24\columnwidth]{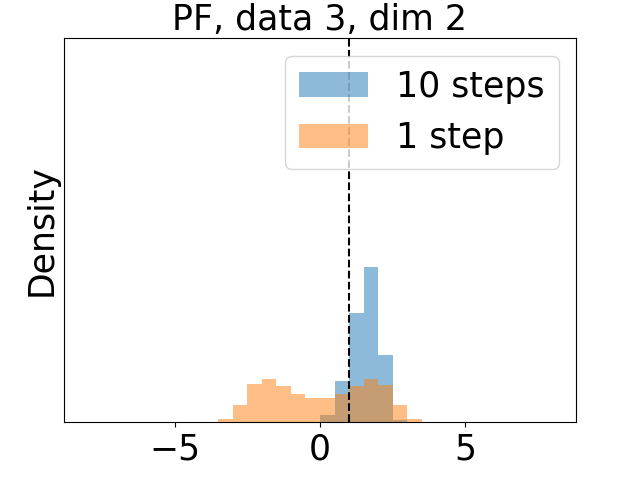}
    \includegraphics[width=0.24\columnwidth]{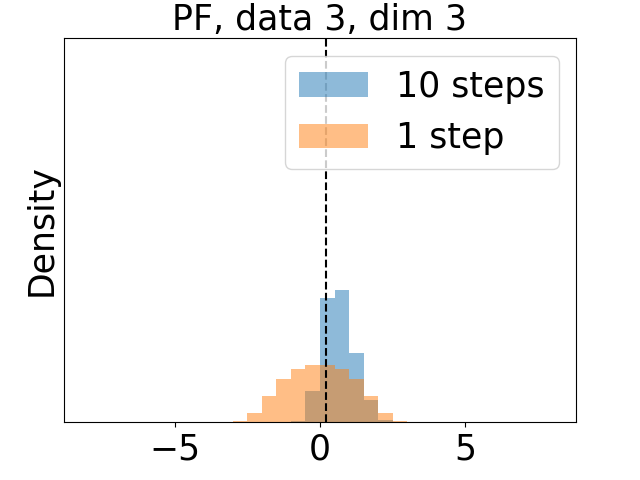}
    \includegraphics[width=0.24\columnwidth]{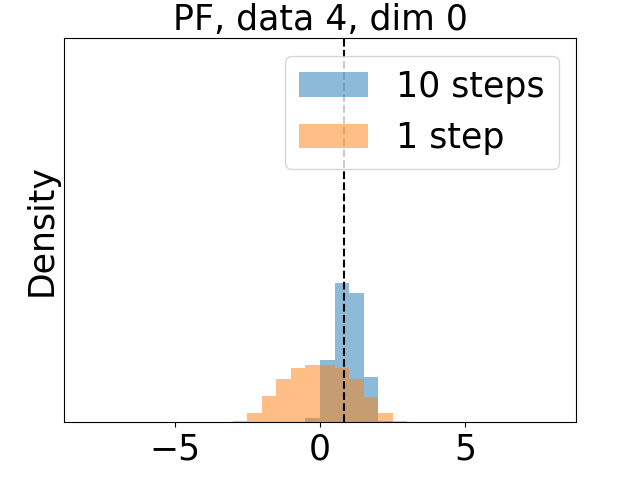}
    \includegraphics[width=0.24\columnwidth]{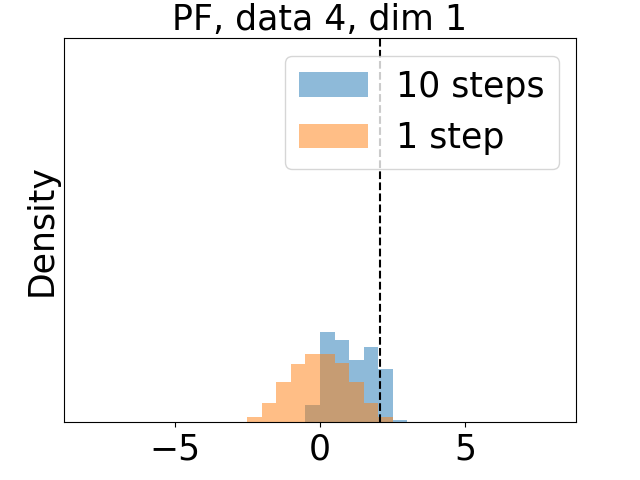}
    \includegraphics[width=0.24\columnwidth]{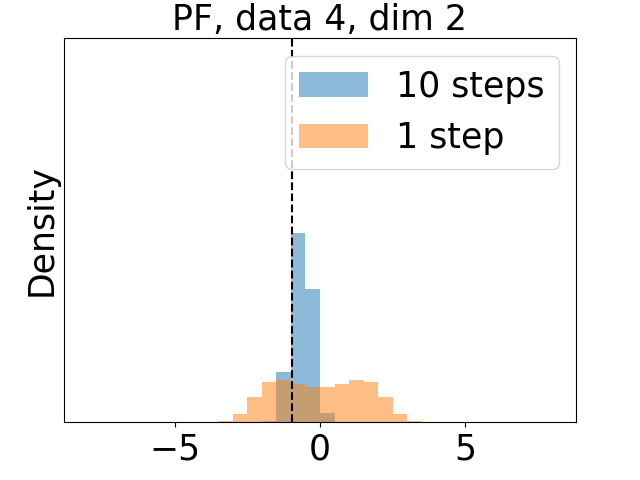}
    \includegraphics[width=0.24\columnwidth]{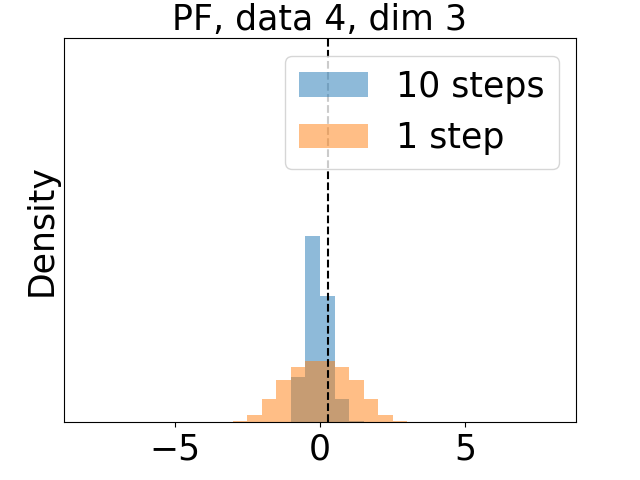}
    \includegraphics[width=0.24\columnwidth]{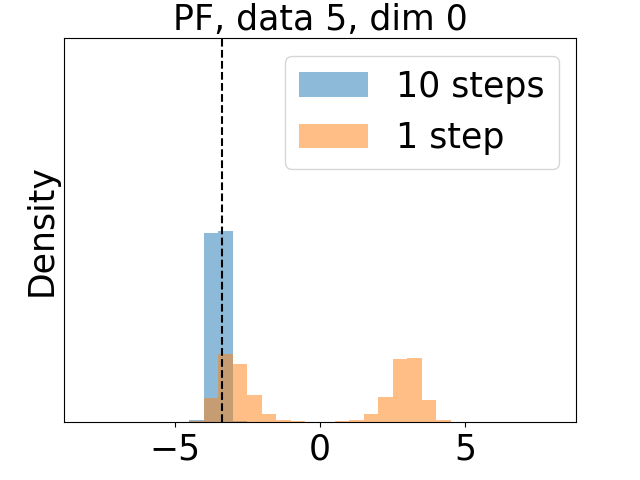}
    \includegraphics[width=0.24\columnwidth]{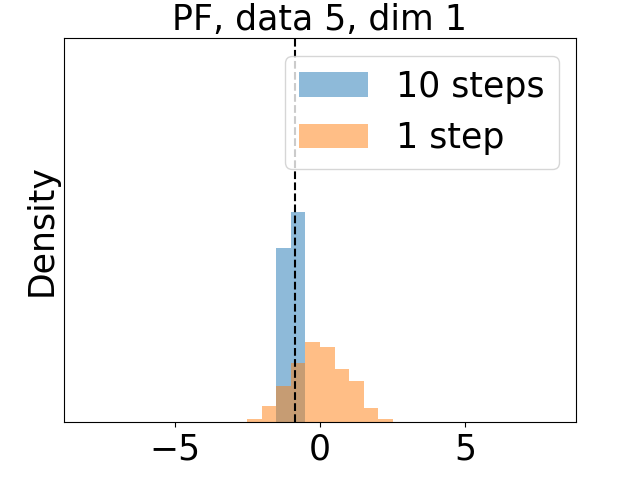}
    \includegraphics[width=0.24\columnwidth]{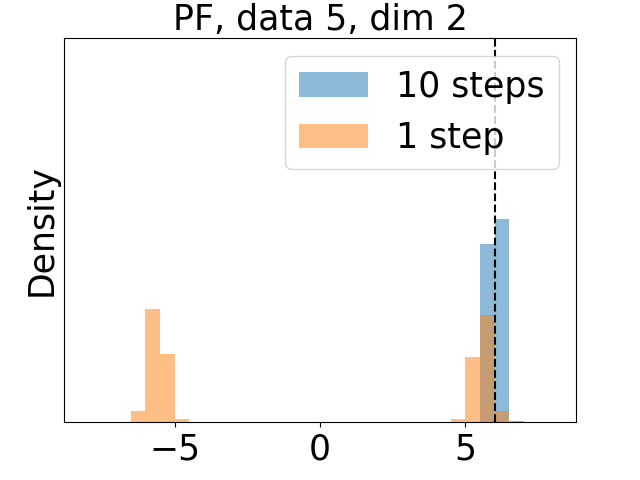}
    \includegraphics[width=0.24\columnwidth]{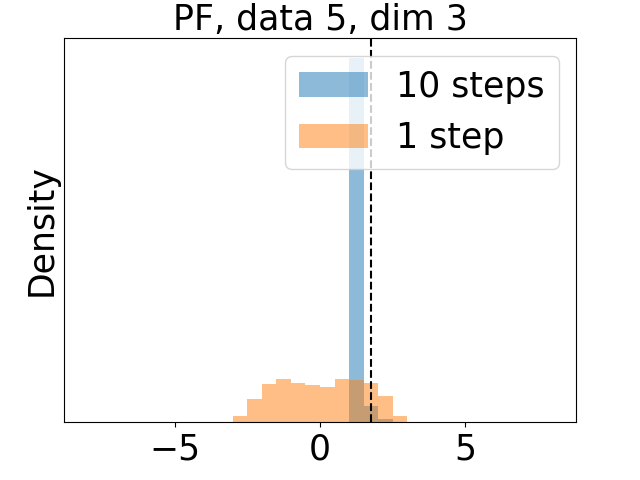}
    \caption{Histograms for the PF model.}
    \label{fig:PF_inference_example}
\end{figure}

\begin{figure}[htbp]
    \centering
    \includegraphics[width=0.24\columnwidth]{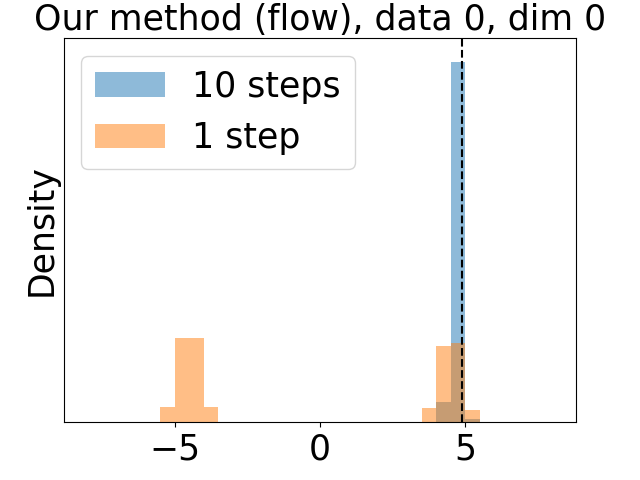}
    \includegraphics[width=0.24\columnwidth]{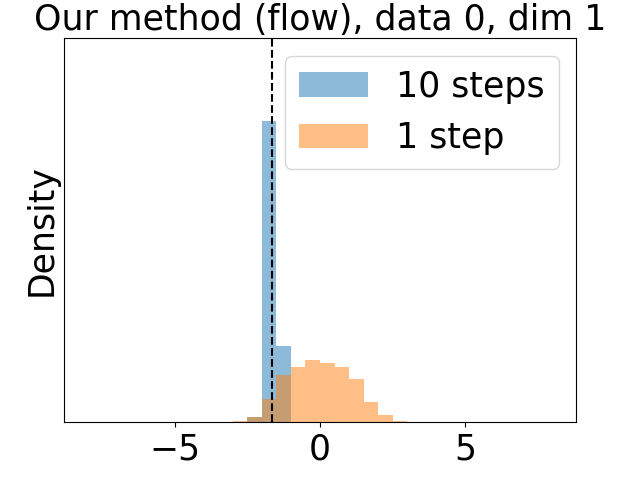}
    \includegraphics[width=0.24\columnwidth]{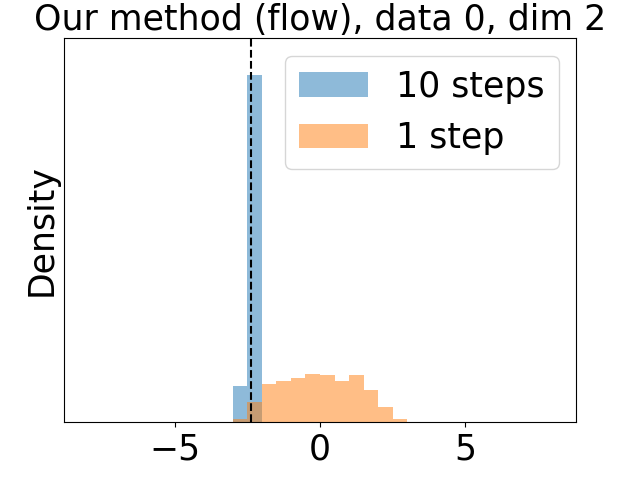}
    \includegraphics[width=0.24\columnwidth]{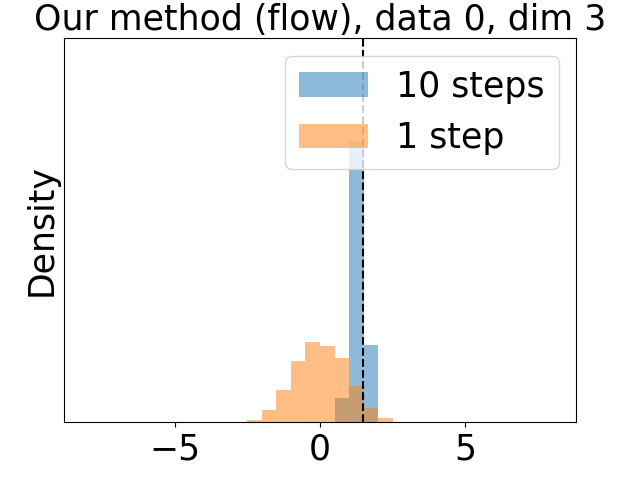}
    \includegraphics[width=0.24\columnwidth]{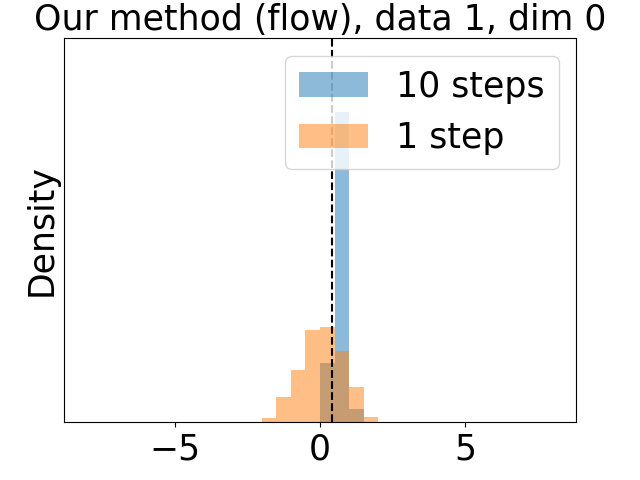}
    \includegraphics[width=0.24\columnwidth]{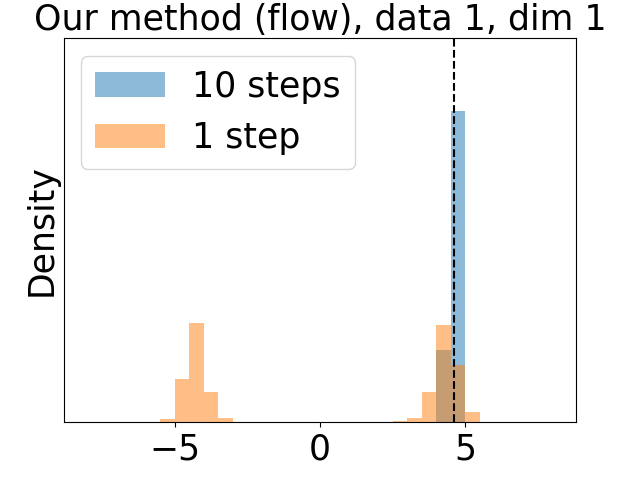}
    \includegraphics[width=0.24\columnwidth]{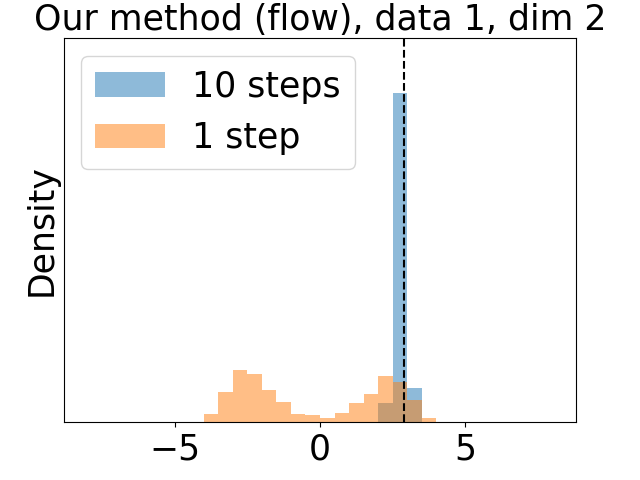}
    \includegraphics[width=0.24\columnwidth]{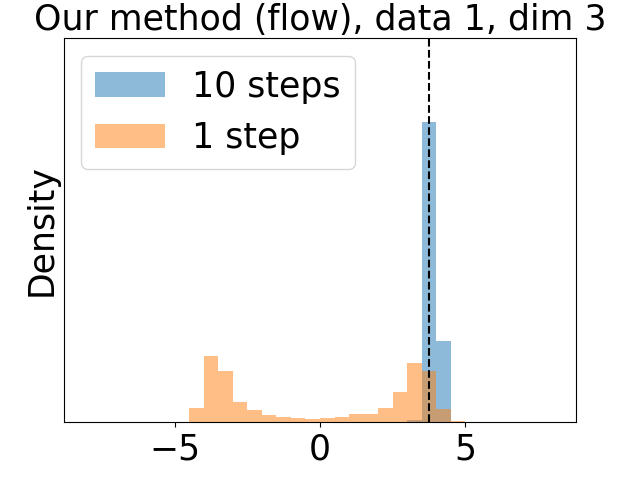}
    \includegraphics[width=0.24\columnwidth]{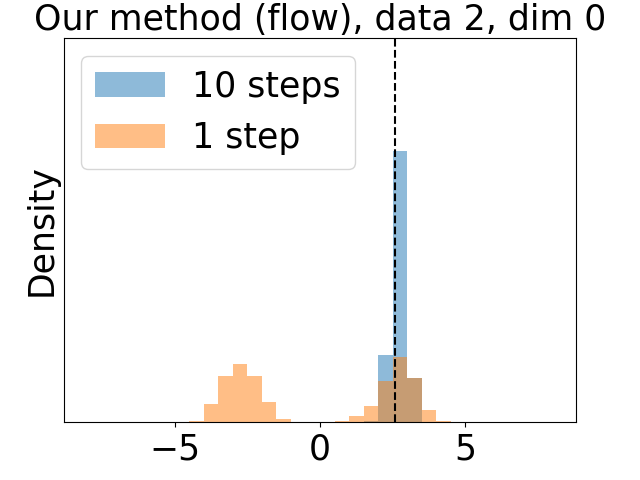}
    \includegraphics[width=0.24\columnwidth]{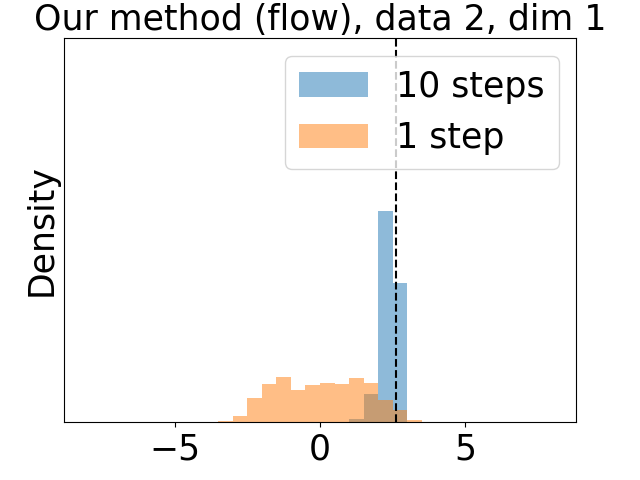}
    \includegraphics[width=0.24\columnwidth]{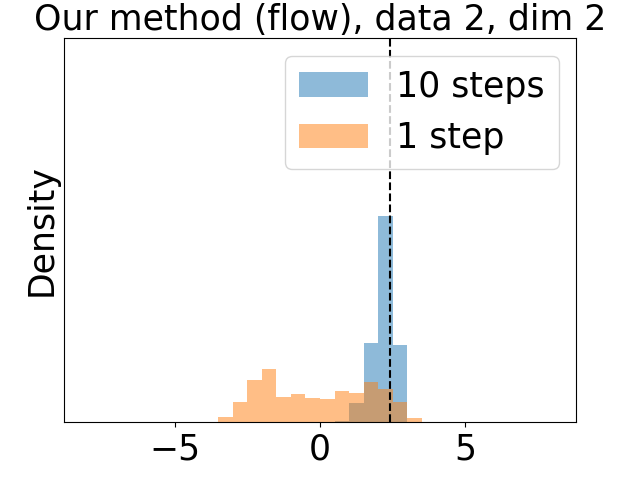}
    \includegraphics[width=0.24\columnwidth]{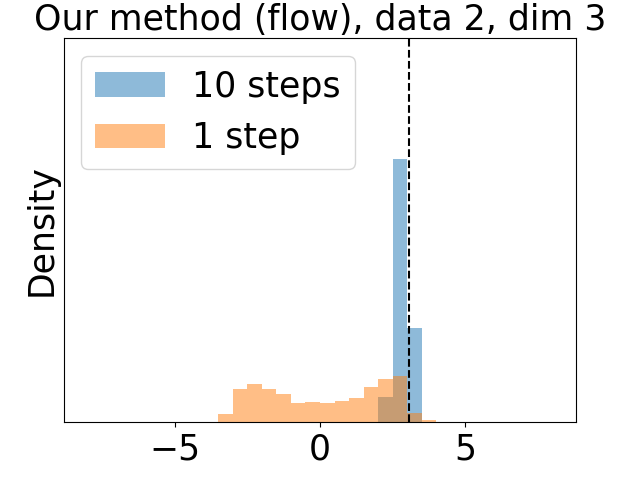}
    \includegraphics[width=0.24\columnwidth]{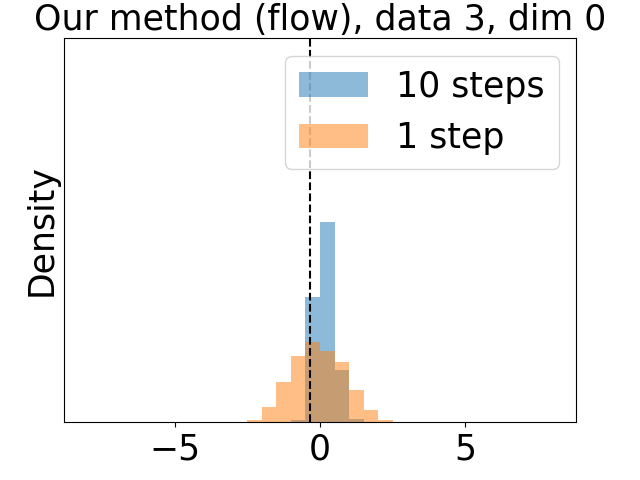}
    \includegraphics[width=0.24\columnwidth]{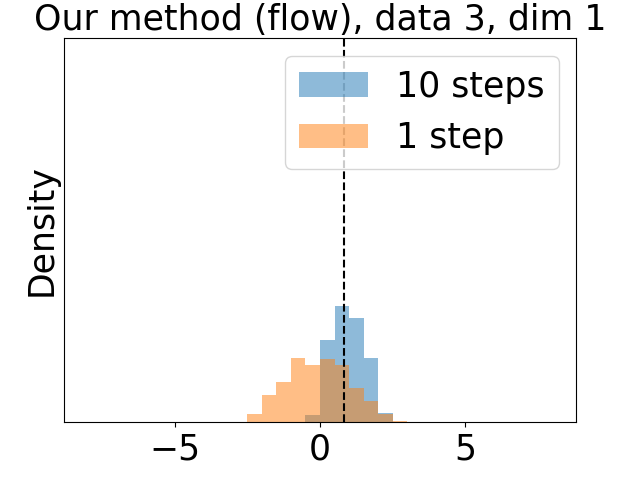}
    \includegraphics[width=0.24\columnwidth]{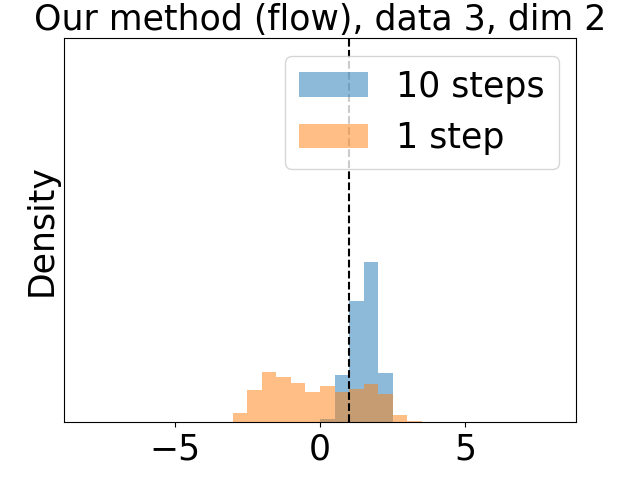}
    \includegraphics[width=0.24\columnwidth]{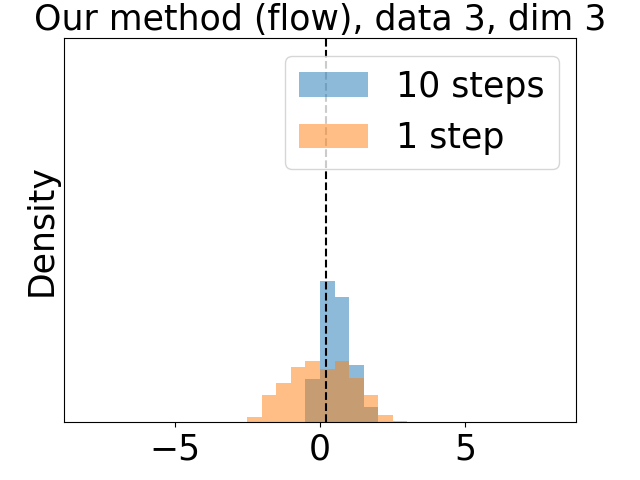}
    \includegraphics[width=0.24\columnwidth]{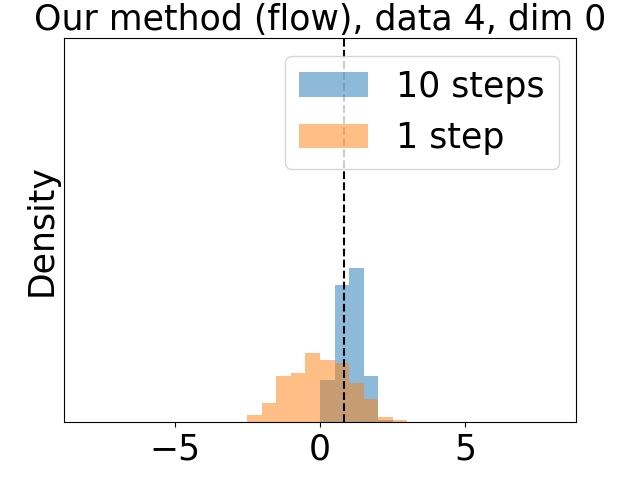}
    \includegraphics[width=0.24\columnwidth]{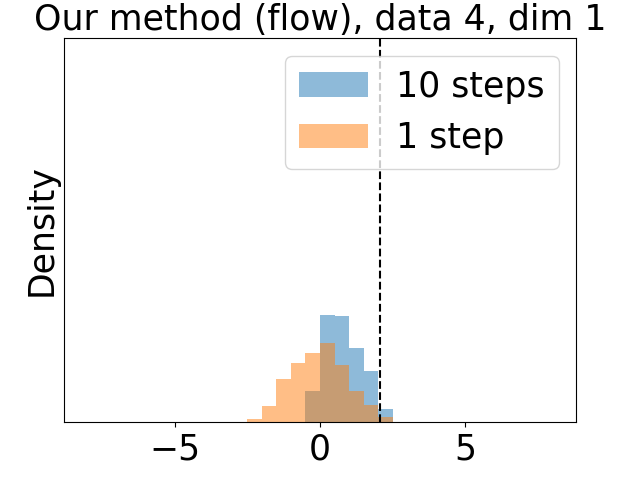}
    \includegraphics[width=0.24\columnwidth]{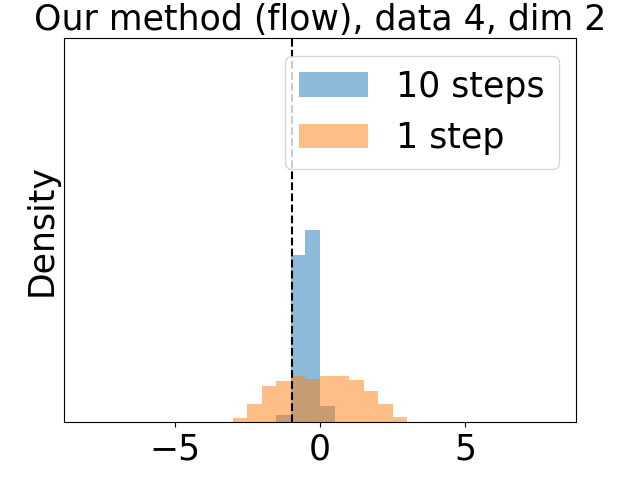}
    \includegraphics[width=0.24\columnwidth]{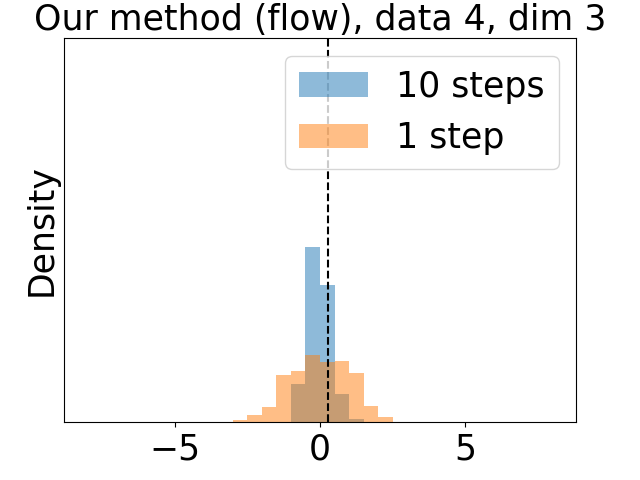}
    \includegraphics[width=0.24\columnwidth]{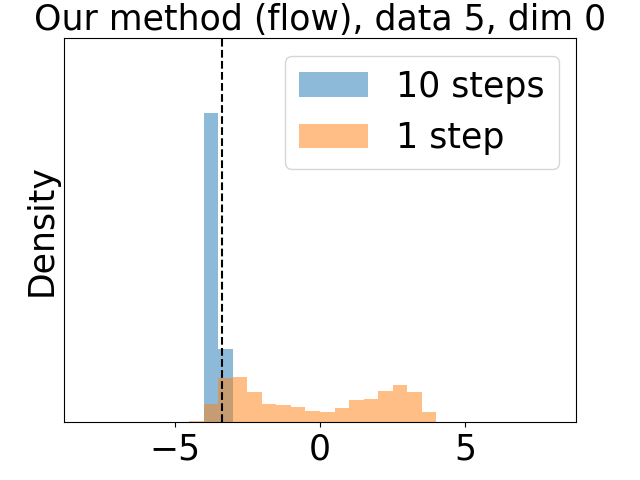}
    \includegraphics[width=0.24\columnwidth]{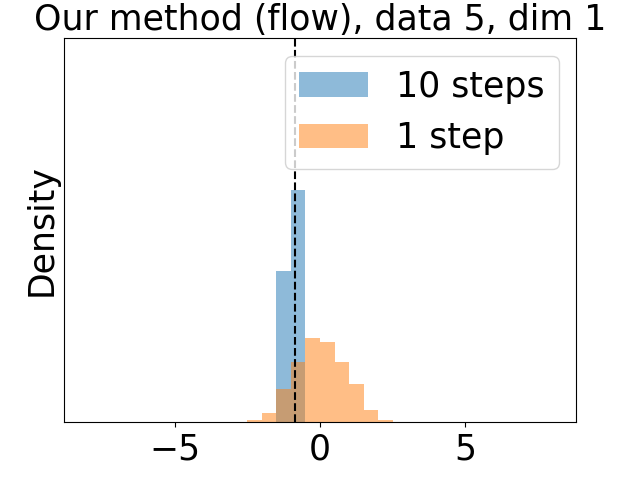}
    \includegraphics[width=0.24\columnwidth]{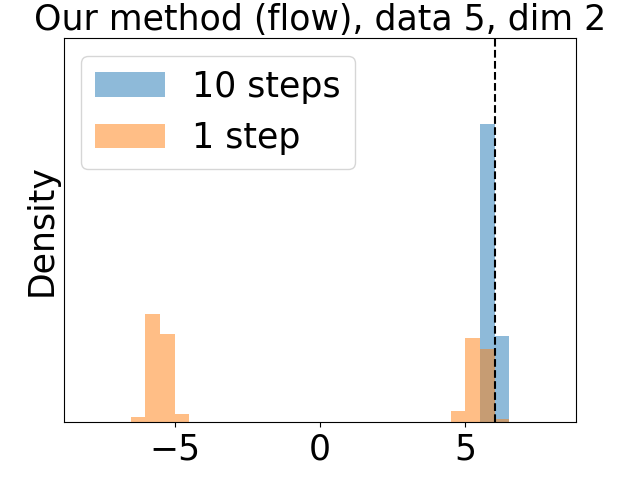}
    \includegraphics[width=0.24\columnwidth]{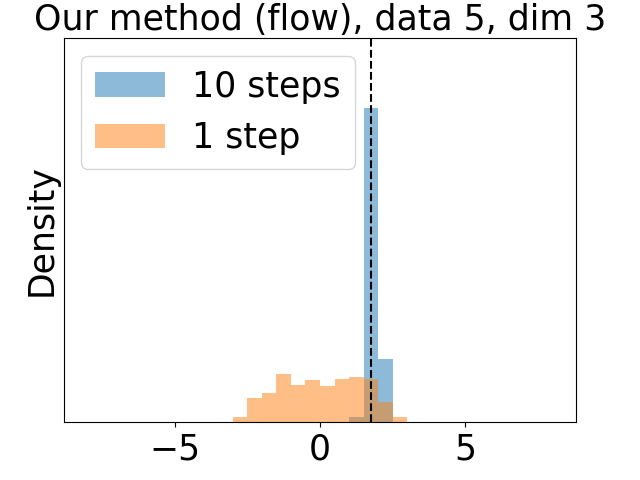}
    \caption{Histograms for the flow model.}
    \label{fig:flow_inference_example}
\end{figure}

\begin{figure}
    \centering
    \includegraphics[width=0.24\columnwidth]{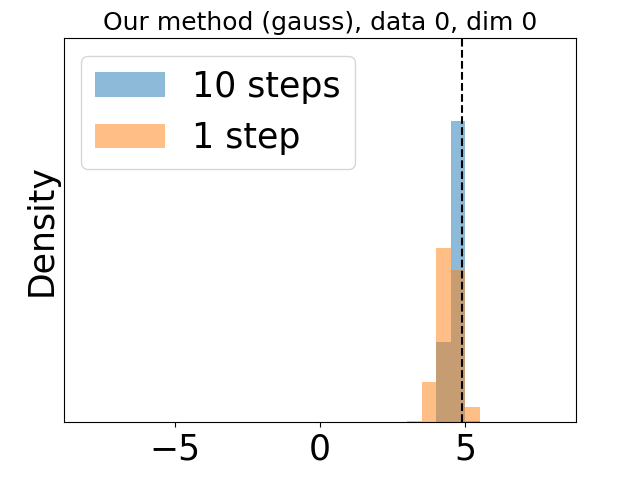}
    \includegraphics[width=0.24\columnwidth]{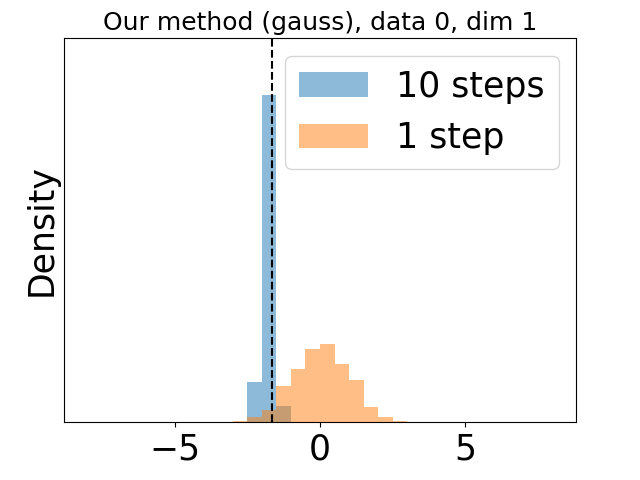}
    \includegraphics[width=0.24\columnwidth]{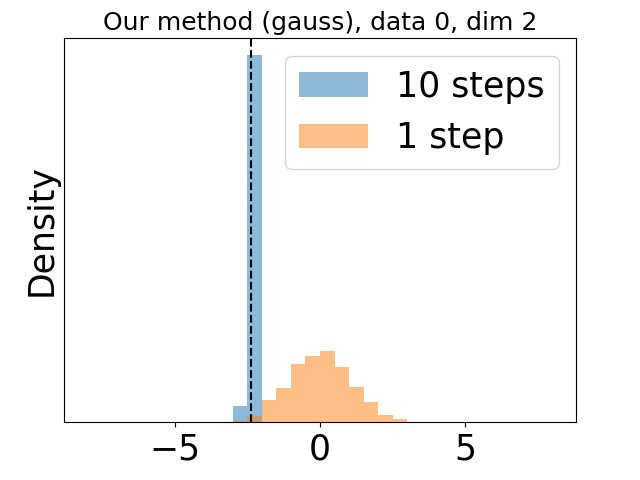}
    \includegraphics[width=0.24\columnwidth]{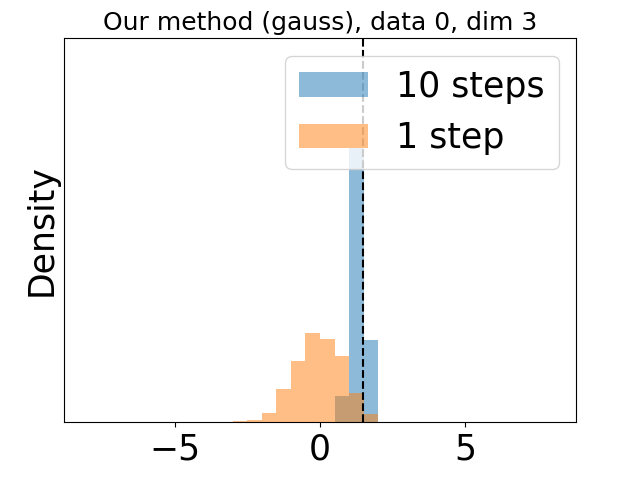}
    \includegraphics[width=0.24\columnwidth]{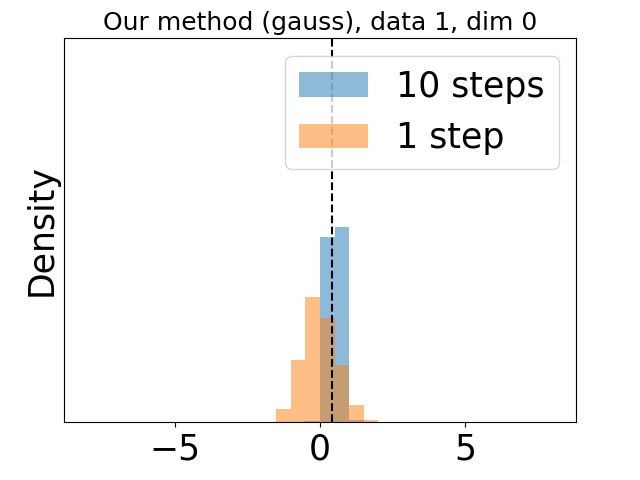}
    \includegraphics[width=0.24\columnwidth]{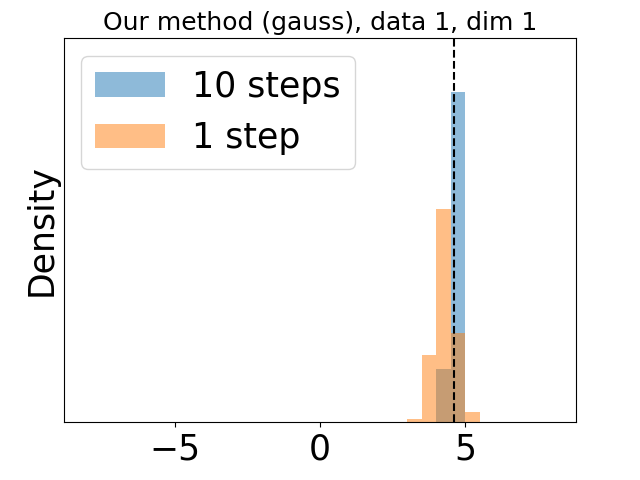}
    \includegraphics[width=0.24\columnwidth]{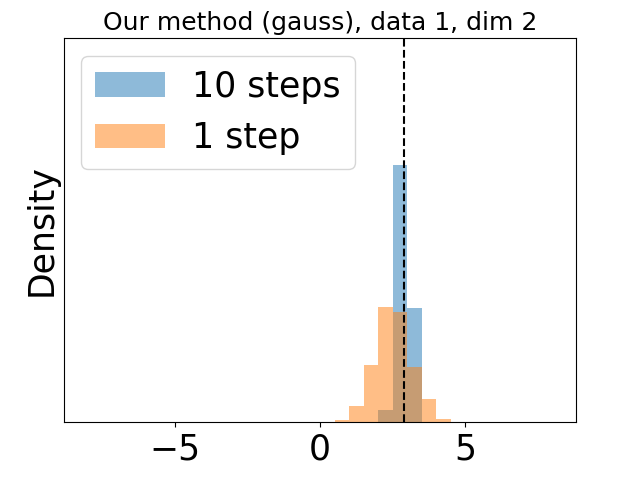}
    \includegraphics[width=0.24\columnwidth]{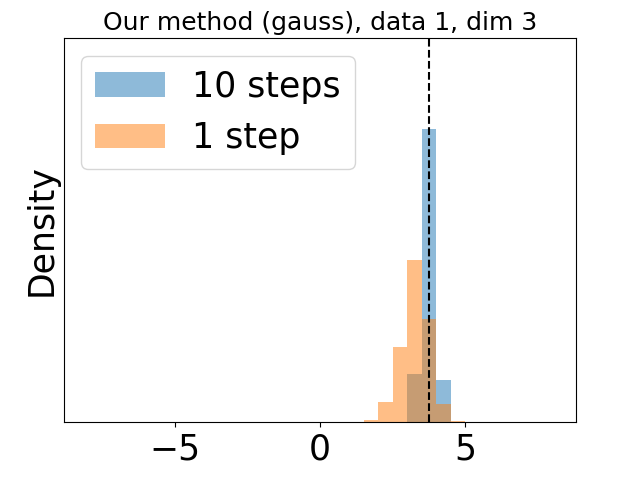}
    \includegraphics[width=0.24\columnwidth]{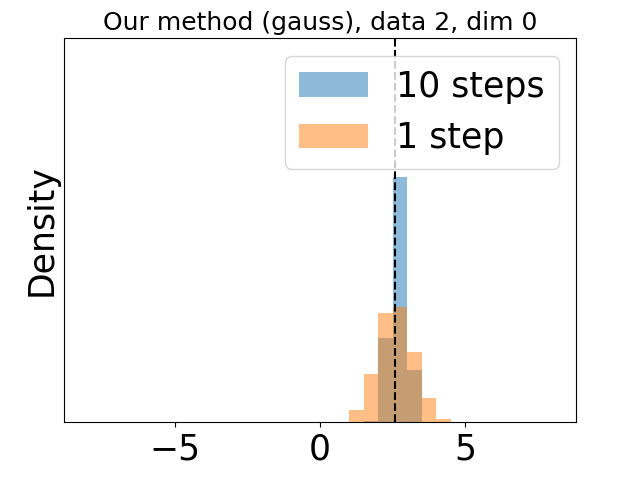}
    \includegraphics[width=0.24\columnwidth]{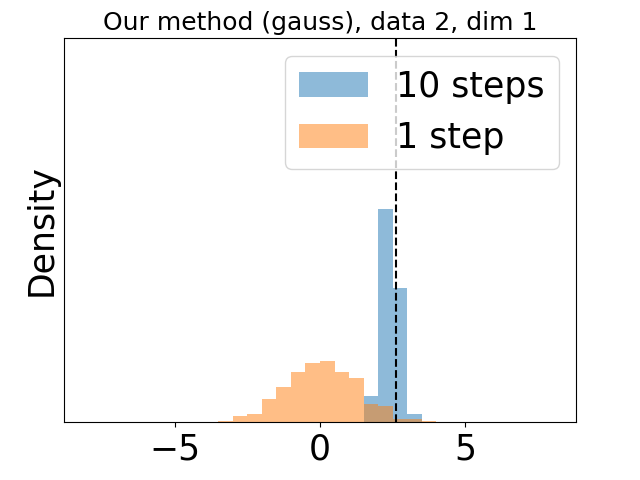}
    \includegraphics[width=0.24\columnwidth]{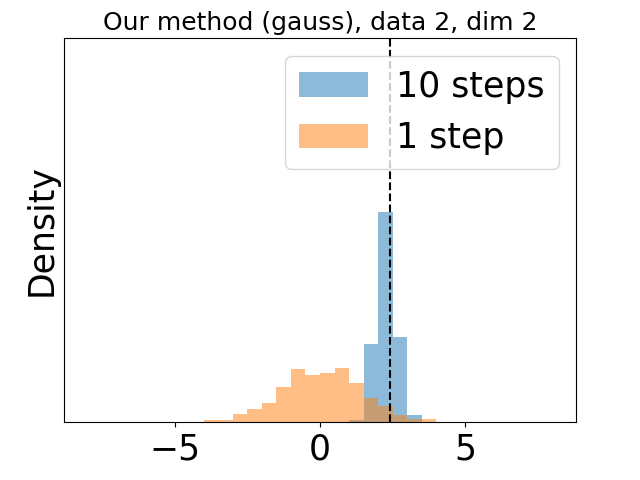}
    \includegraphics[width=0.24\columnwidth]{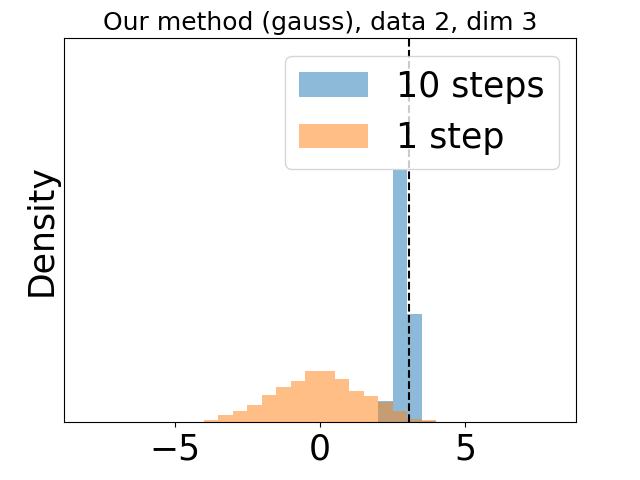}
    \includegraphics[width=0.24\columnwidth]{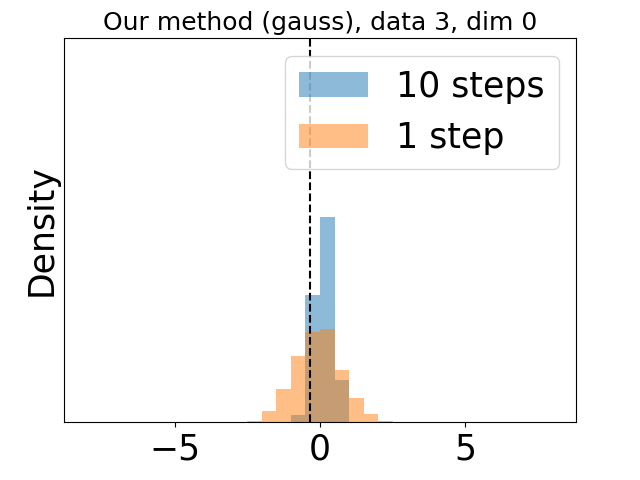}
    \includegraphics[width=0.24\columnwidth]{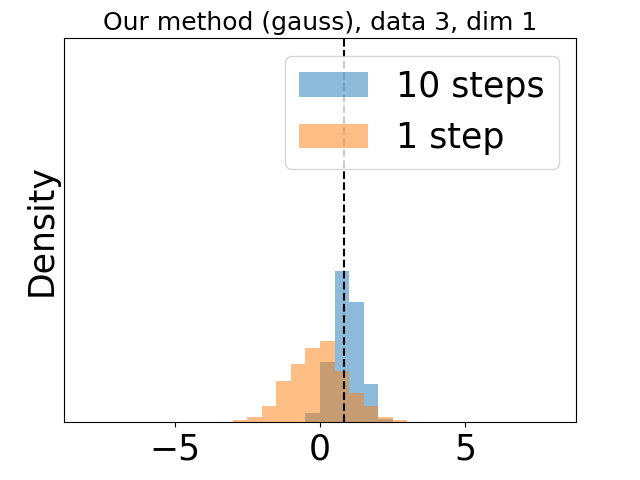}
    \includegraphics[width=0.24\columnwidth]{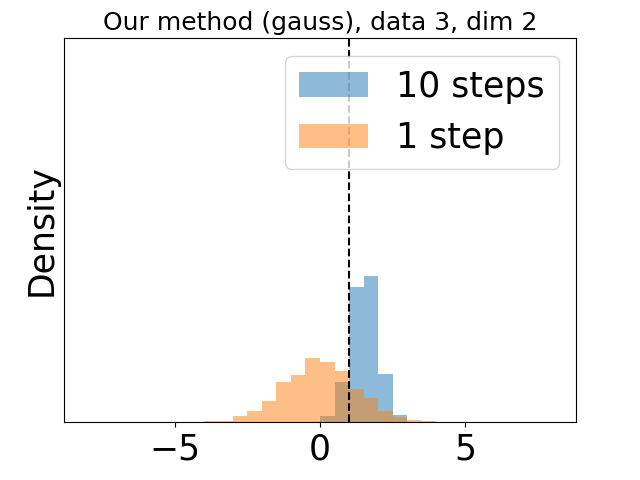}
    \includegraphics[width=0.24\columnwidth]{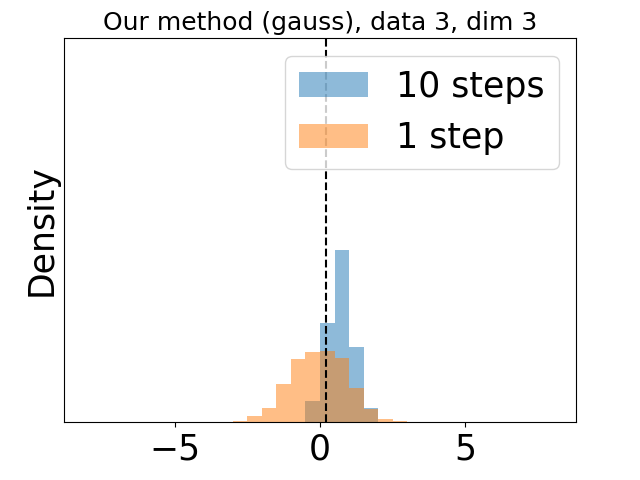}
    \includegraphics[width=0.24\columnwidth]{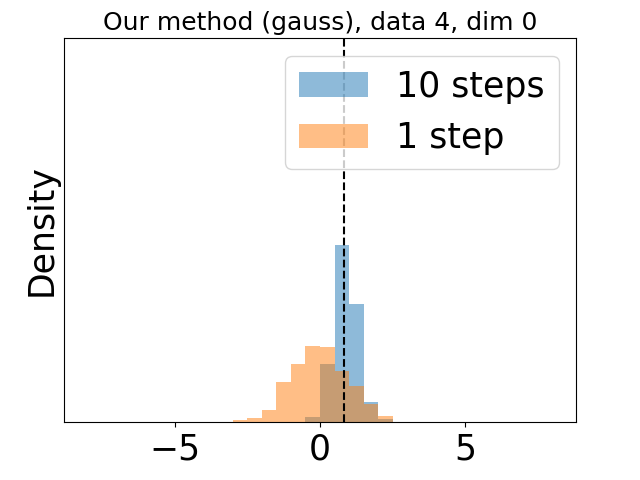}
    \includegraphics[width=0.24\columnwidth]{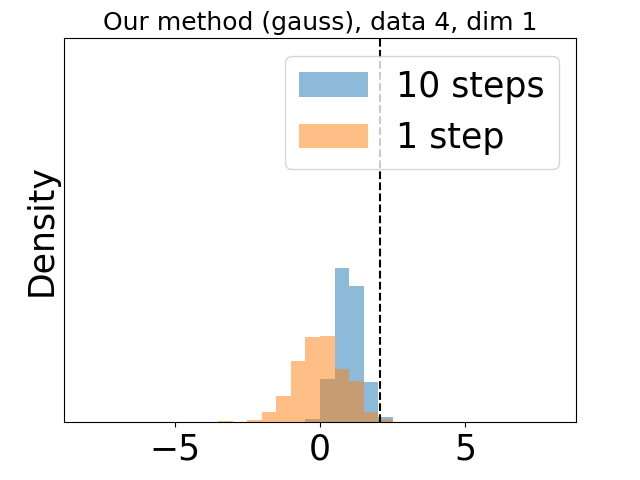}
    \includegraphics[width=0.24\columnwidth]{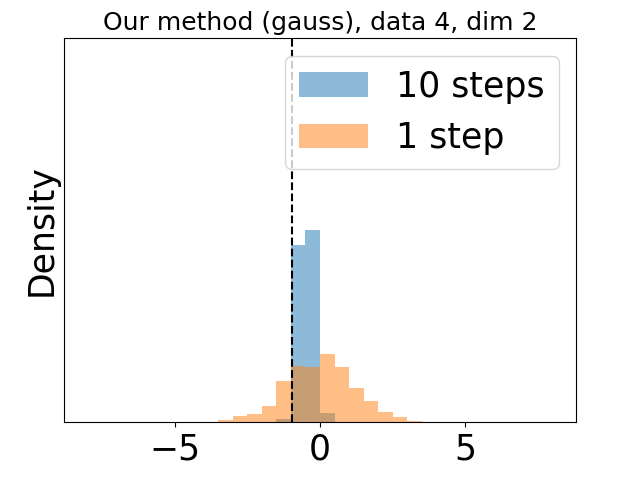}
    \includegraphics[width=0.24\columnwidth]{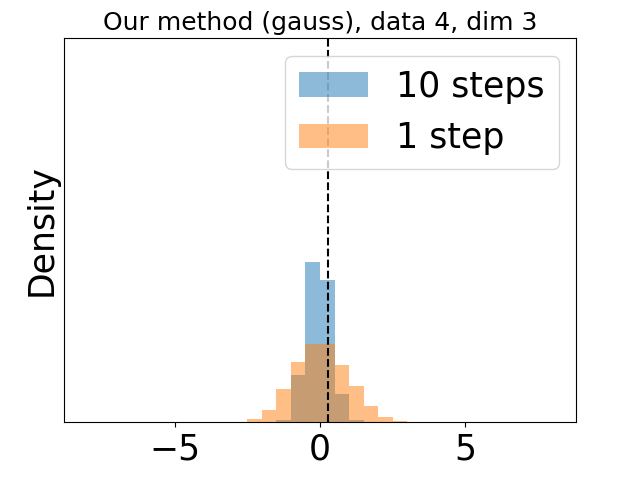}
    \includegraphics[width=0.24\columnwidth]{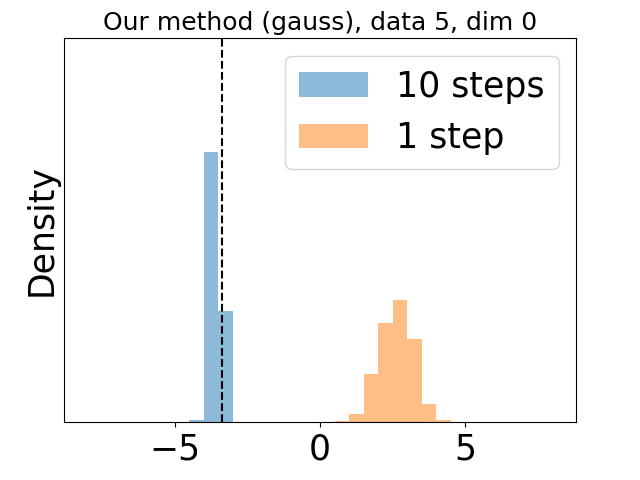}
    \includegraphics[width=0.24\columnwidth]{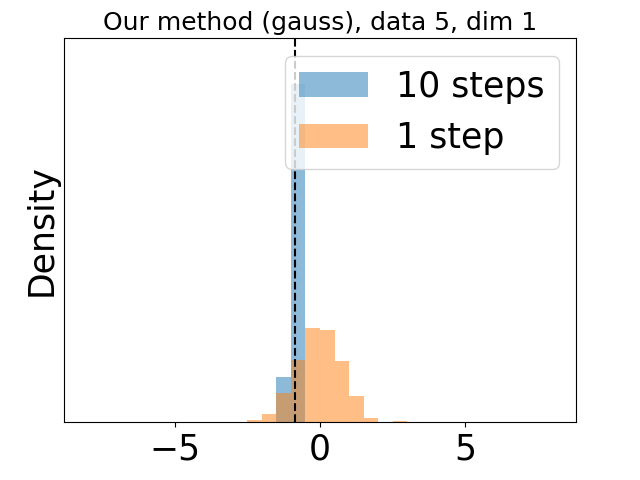}
    \includegraphics[width=0.24\columnwidth]{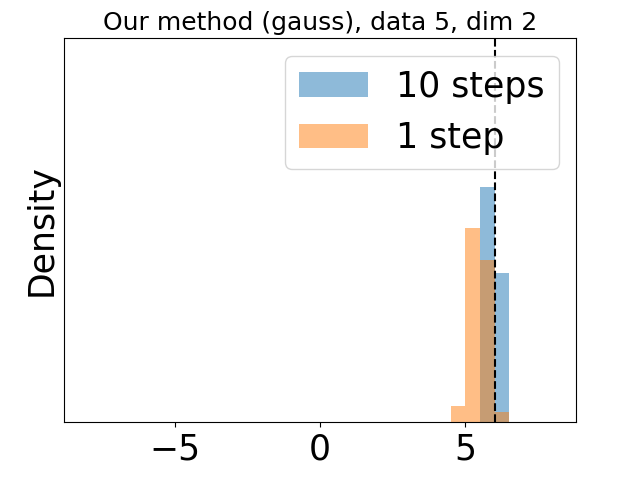}
    \includegraphics[width=0.24\columnwidth]{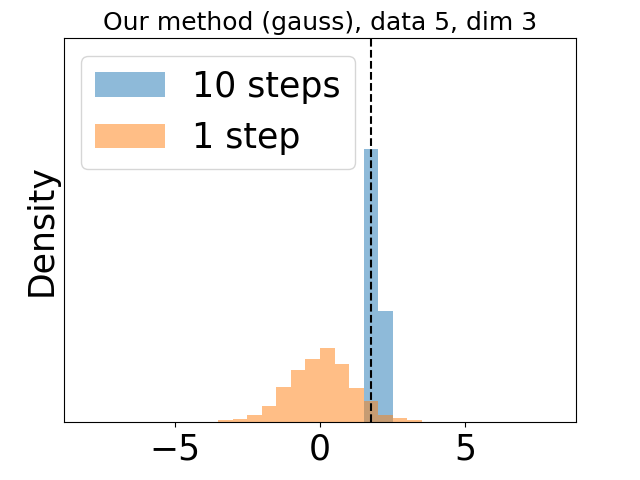}
    \caption{Histograms for the gauss model.}
    \label{fig:gauss_inference_example}
\end{figure}

\begin{figure}
    \centering
    \includegraphics[width=0.24\columnwidth]{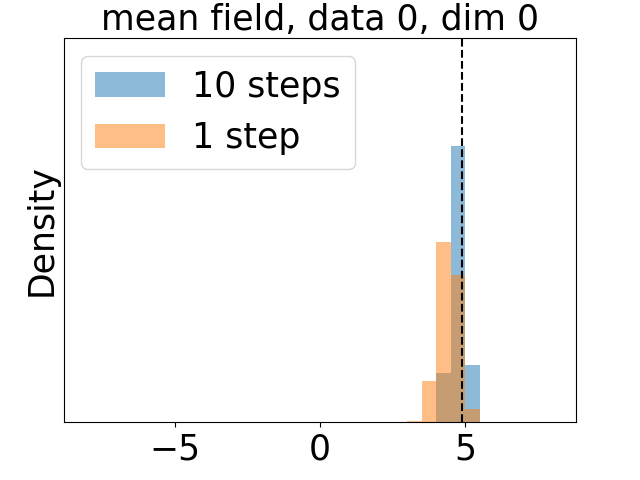}
    \includegraphics[width=0.24\columnwidth]{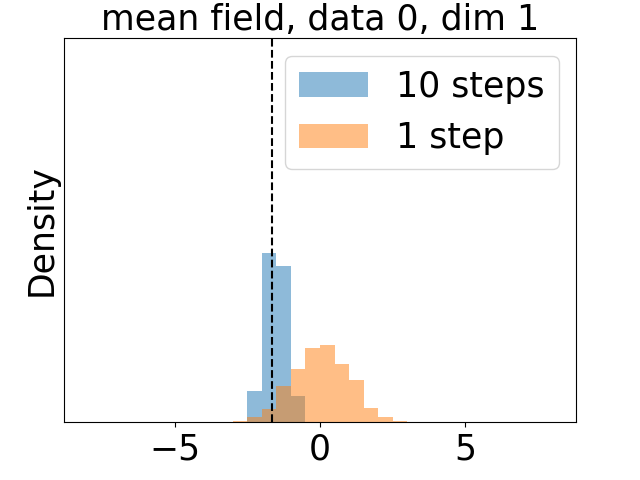}
    \includegraphics[width=0.24\columnwidth]{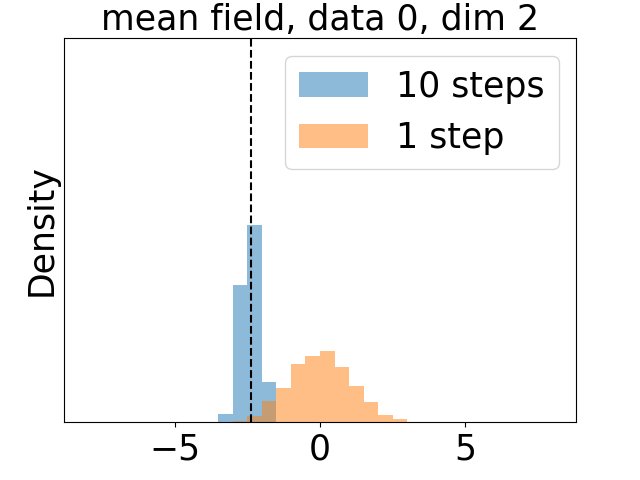}
    \includegraphics[width=0.24\columnwidth]{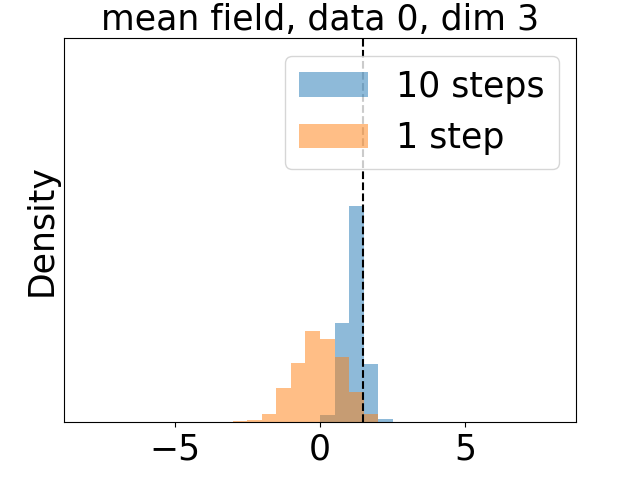}
    \includegraphics[width=0.24\columnwidth]{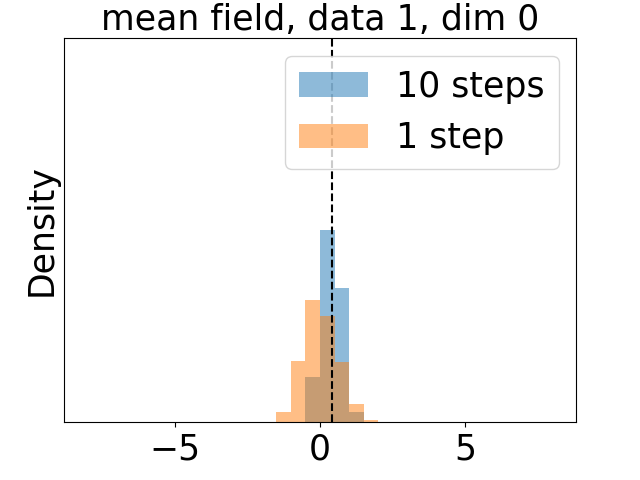}
    \includegraphics[width=0.24\columnwidth]{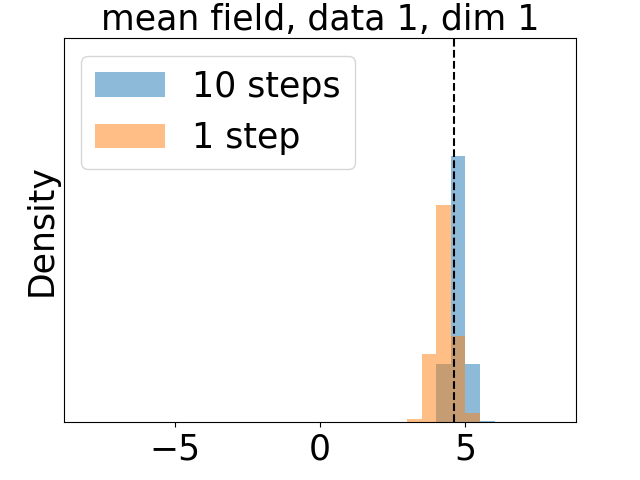}
    \includegraphics[width=0.24\columnwidth]{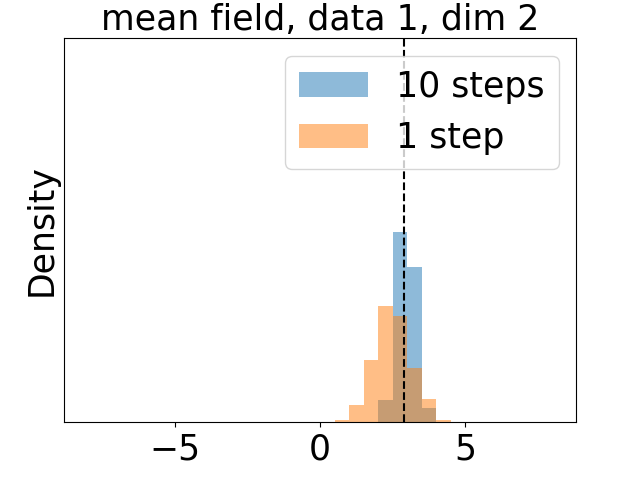}
    \includegraphics[width=0.24\columnwidth]{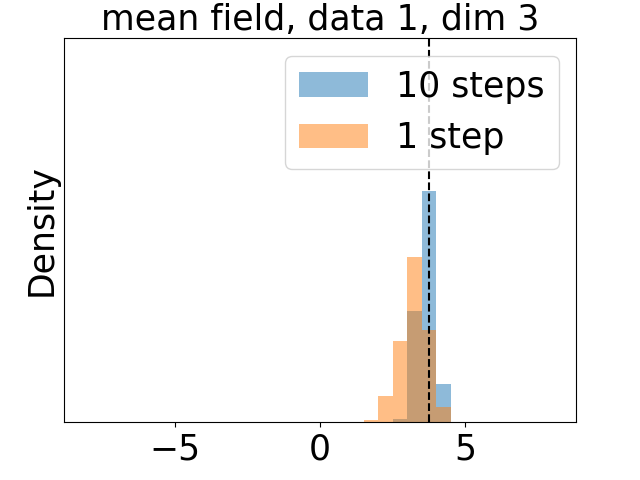}
    \includegraphics[width=0.24\columnwidth]{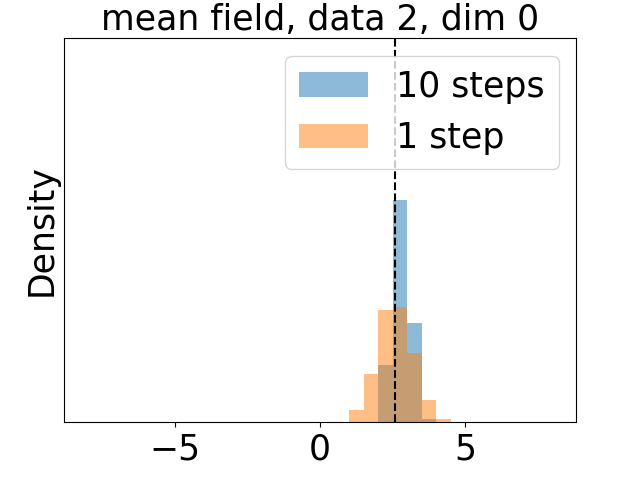}
    \includegraphics[width=0.24\columnwidth]{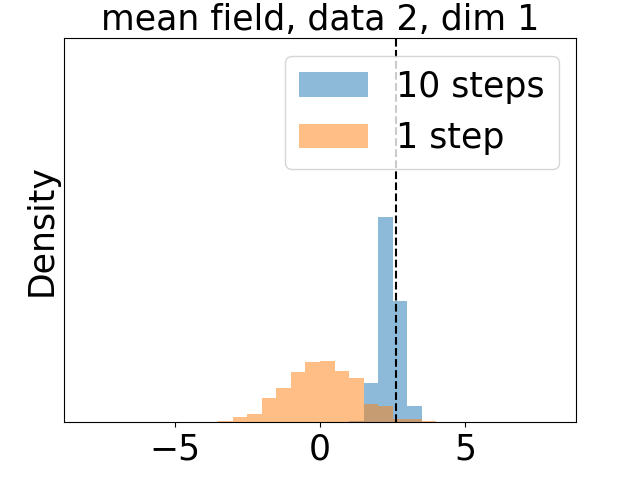}
    \includegraphics[width=0.24\columnwidth]{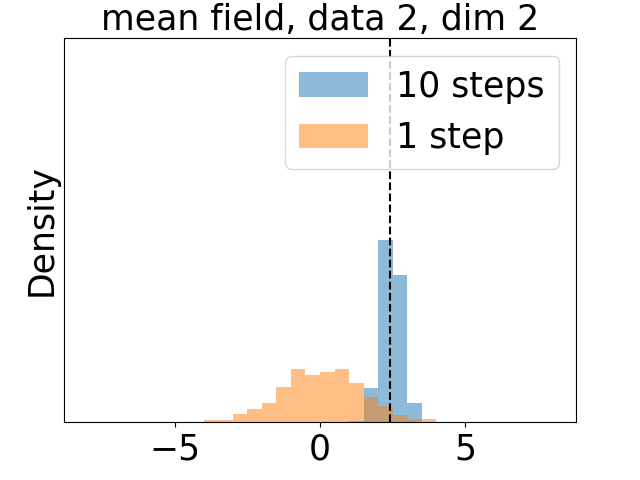}
    \includegraphics[width=0.24\columnwidth]{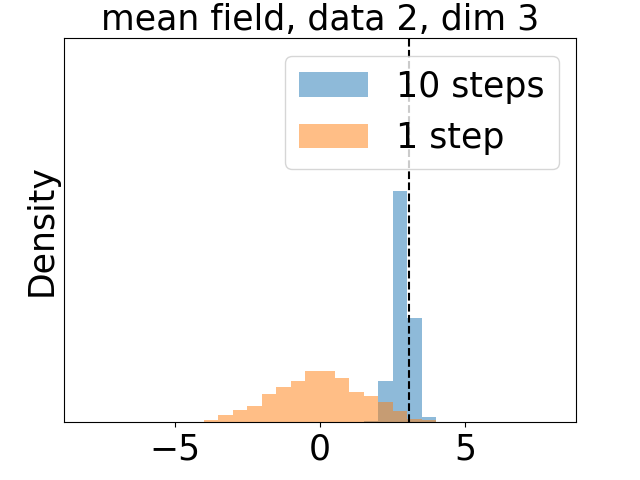}
    \includegraphics[width=0.24\columnwidth]{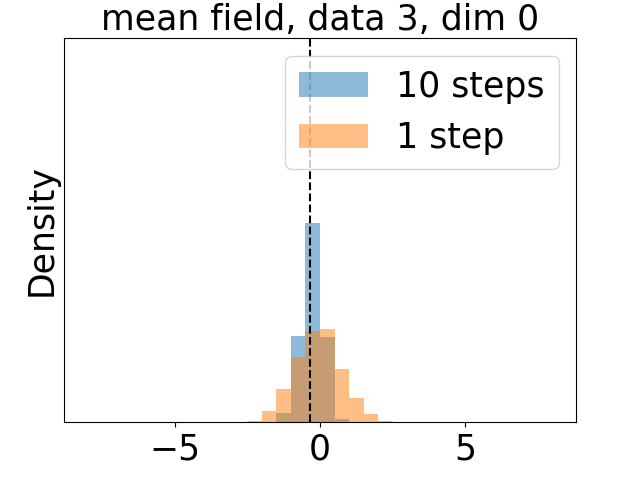}
    \includegraphics[width=0.24\columnwidth]{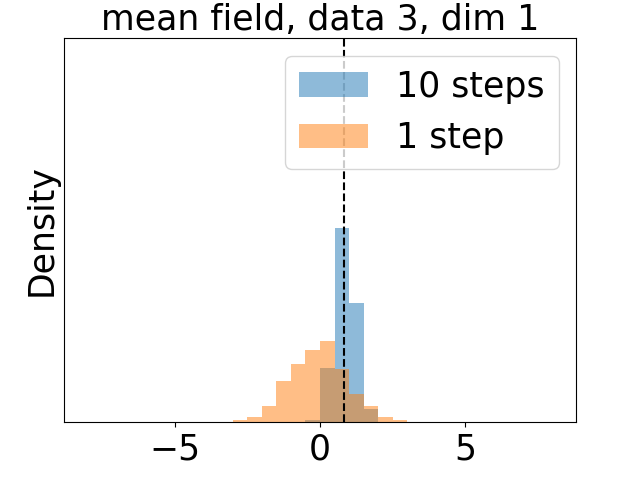}
    \includegraphics[width=0.24\columnwidth]{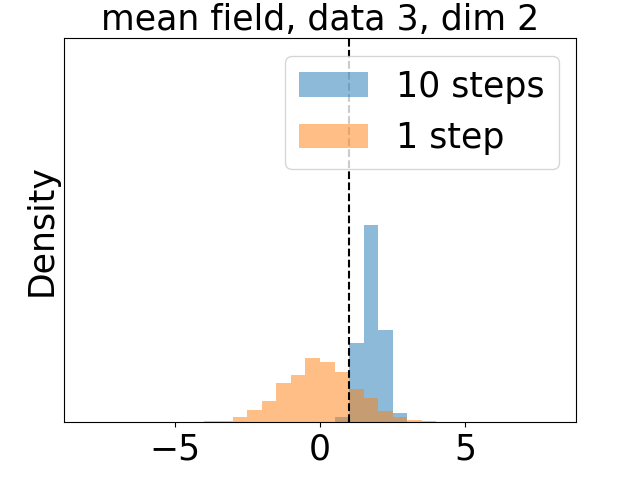}
    \includegraphics[width=0.24\columnwidth]{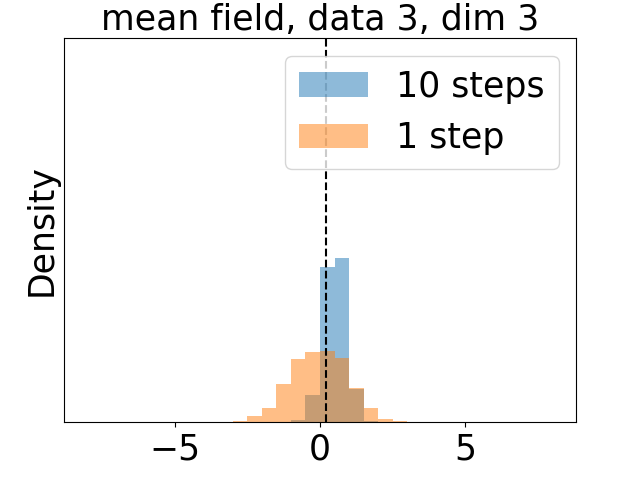}
    \includegraphics[width=0.24\columnwidth]{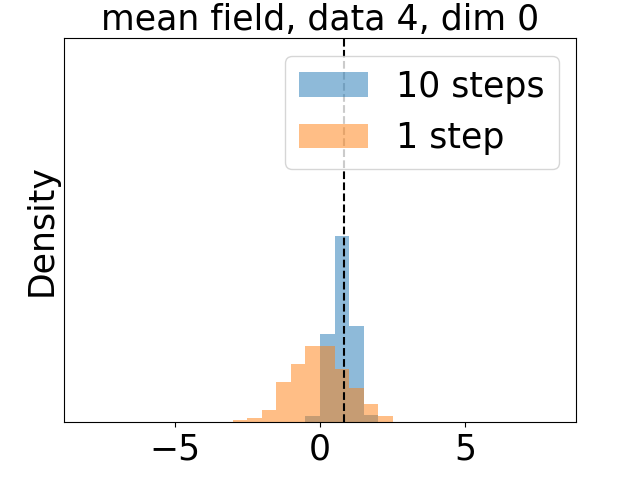}
    \includegraphics[width=0.24\columnwidth]{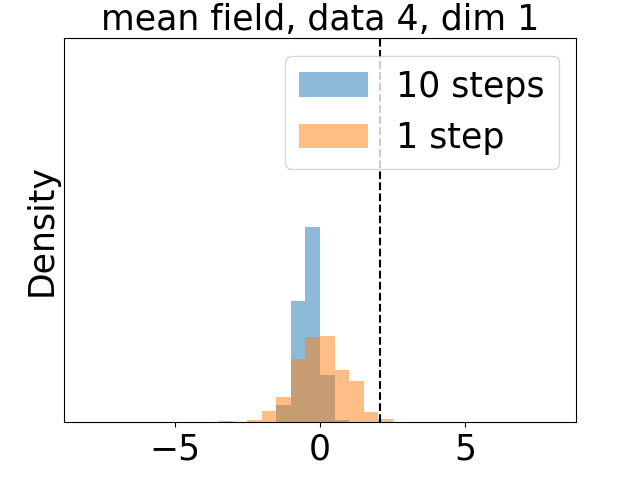}
    \includegraphics[width=0.24\columnwidth]{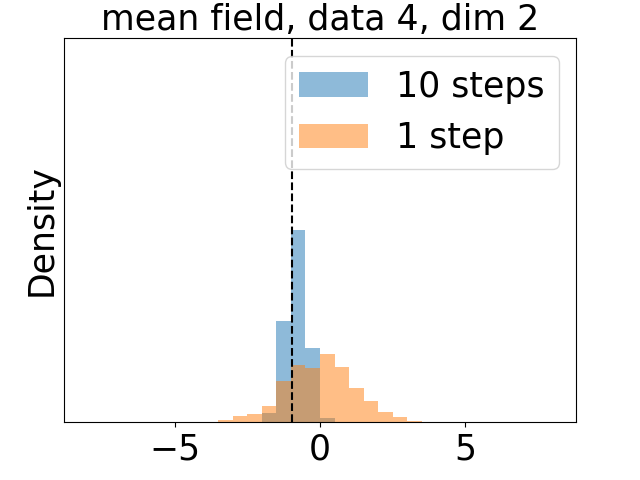}
    \includegraphics[width=0.24\columnwidth]{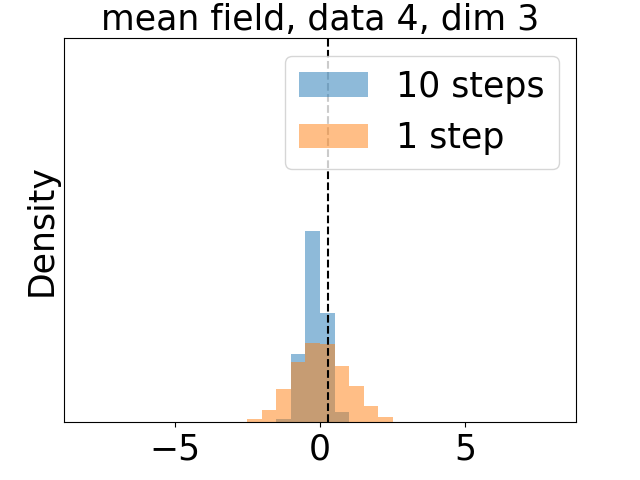}
    \includegraphics[width=0.24\columnwidth]{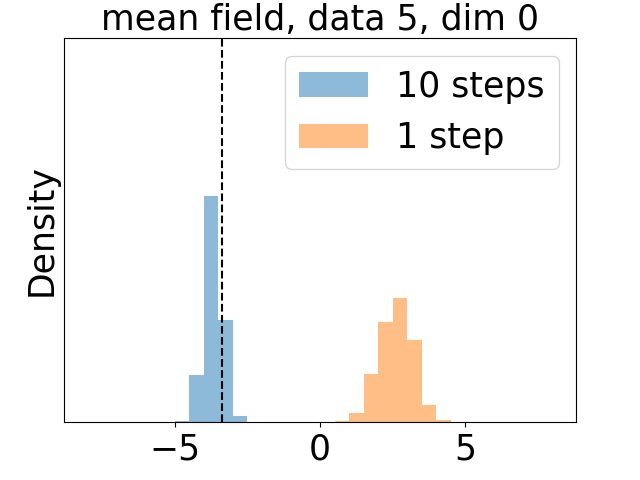}
    \includegraphics[width=0.24\columnwidth]{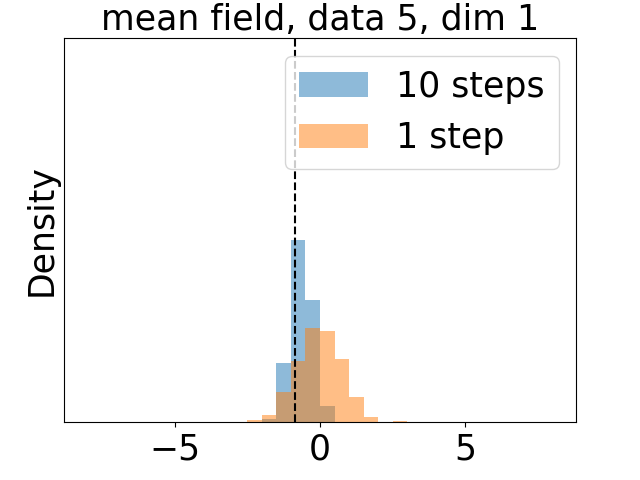}
    \includegraphics[width=0.24\columnwidth]{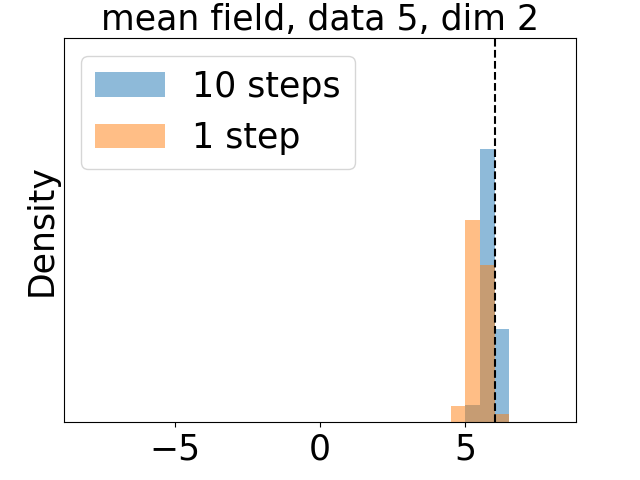}
    \includegraphics[width=0.24\columnwidth]{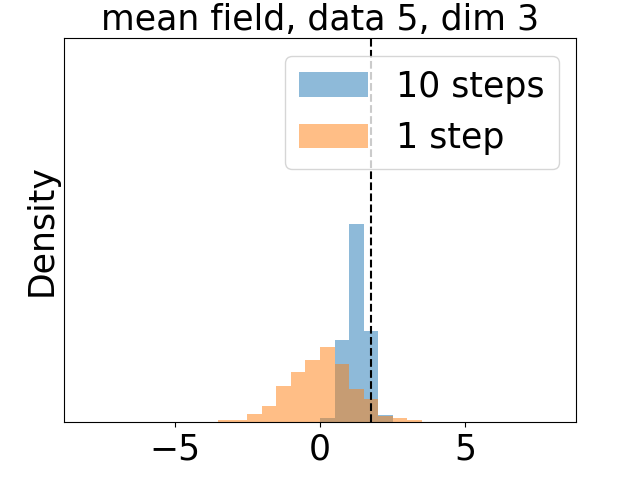}
    \caption{Histograms for the mean-field model.}
    \label{fig:MF_inference_example}
\end{figure}

\begin{figure}
    \centering
    \includegraphics[width=0.24\columnwidth]{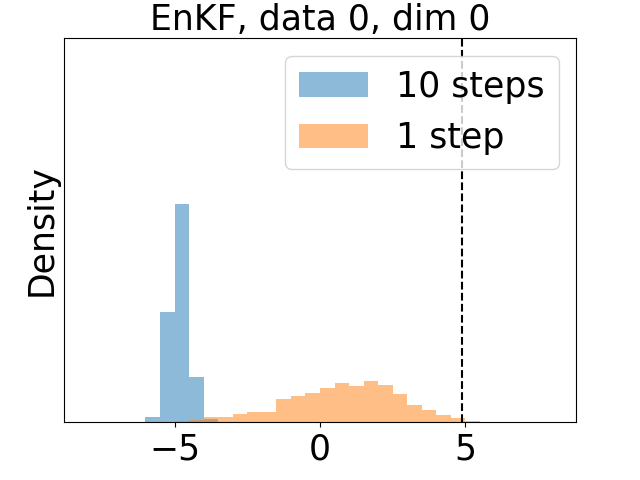}
    \includegraphics[width=0.24\columnwidth]{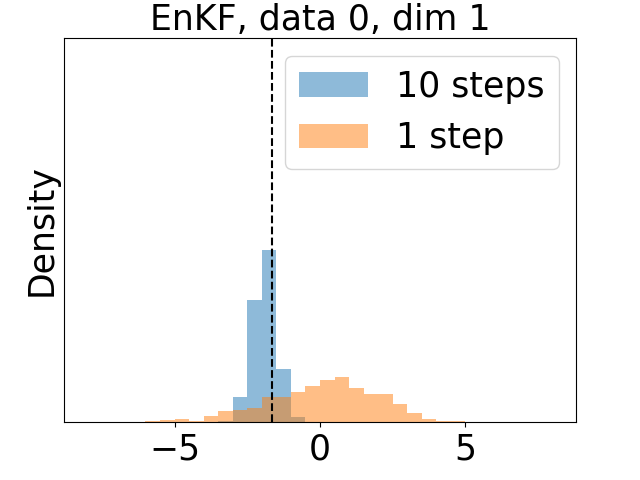}
    \includegraphics[width=0.24\columnwidth]{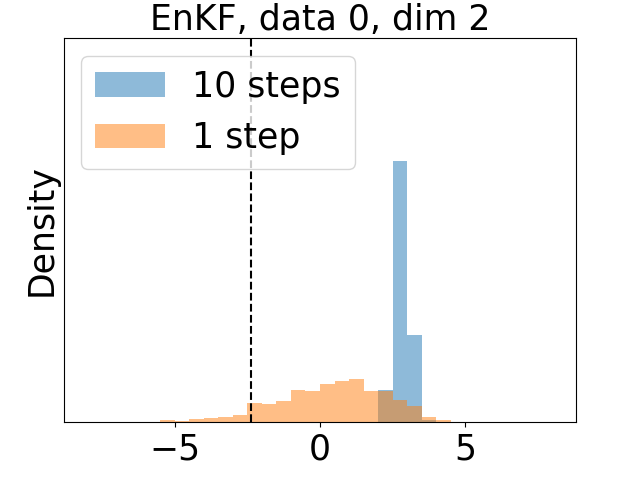}
    \includegraphics[width=0.24\columnwidth]{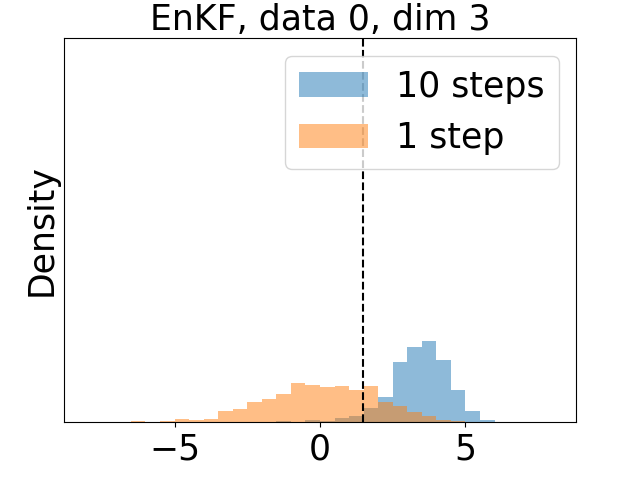}
    \includegraphics[width=0.24\columnwidth]{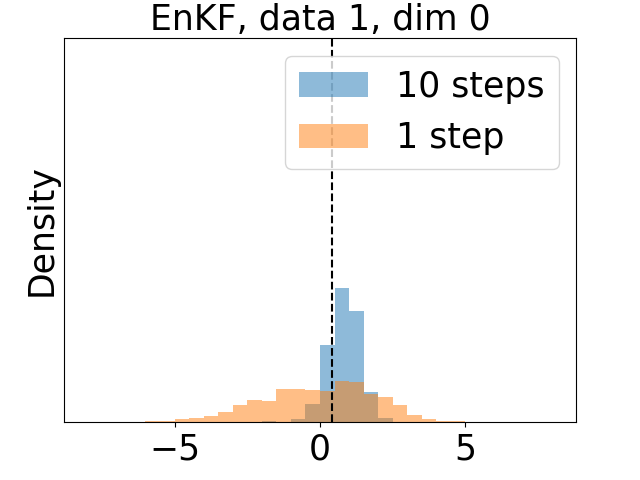}
    \includegraphics[width=0.24\columnwidth]{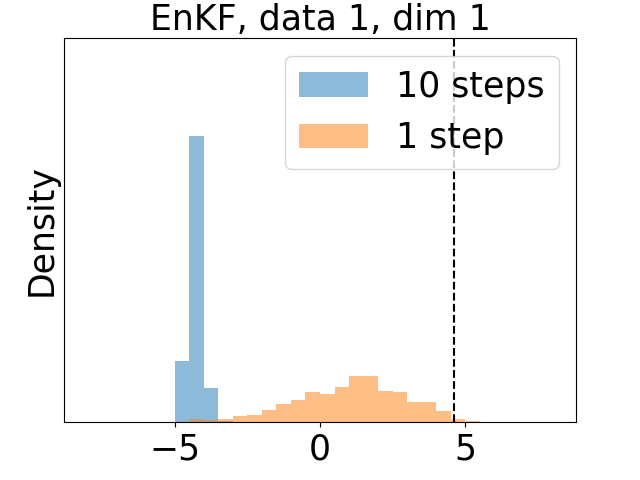}
    \includegraphics[width=0.24\columnwidth]{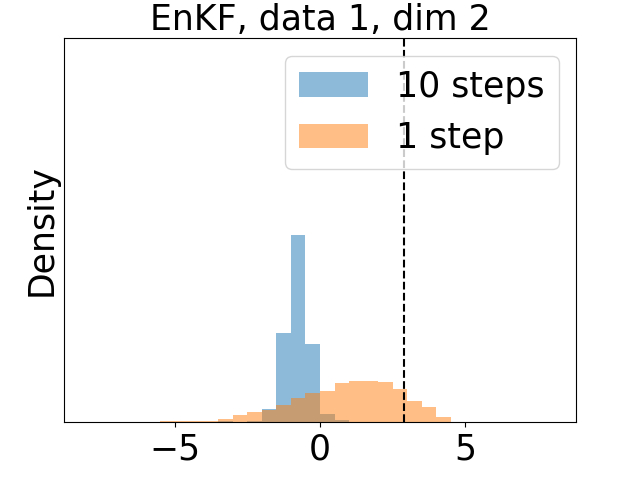}
    \includegraphics[width=0.24\columnwidth]{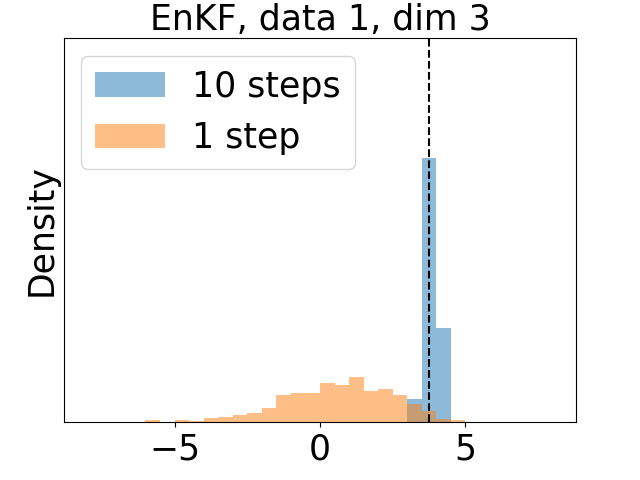}
    \includegraphics[width=0.24\columnwidth]{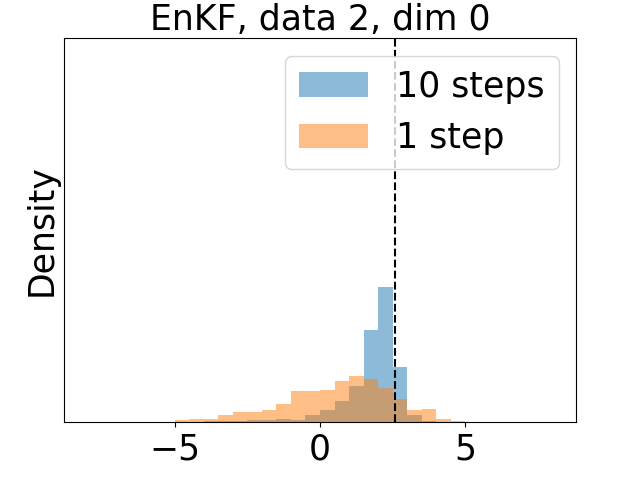}
    \includegraphics[width=0.24\columnwidth]{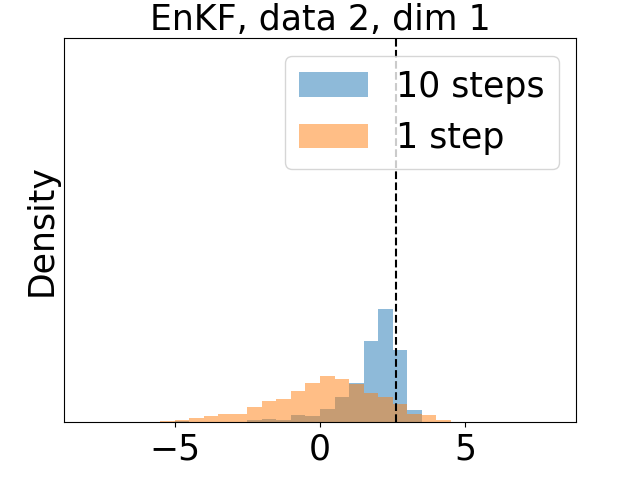}
    \includegraphics[width=0.24\columnwidth]{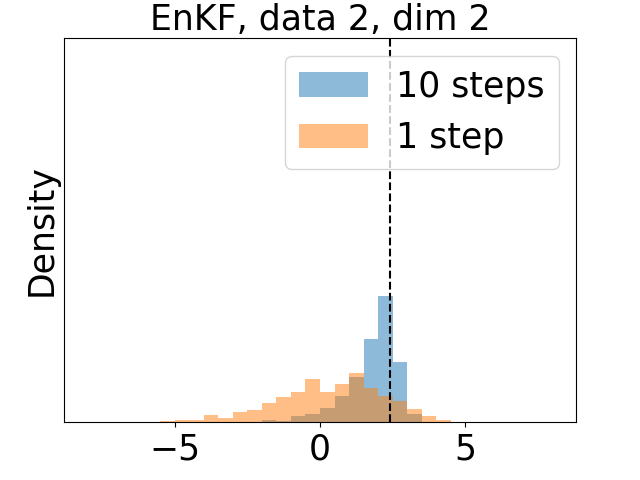}
    \includegraphics[width=0.24\columnwidth]{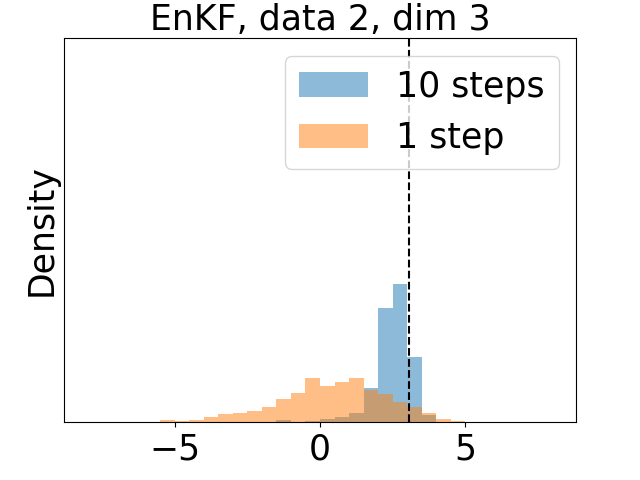}
    \includegraphics[width=0.24\columnwidth]{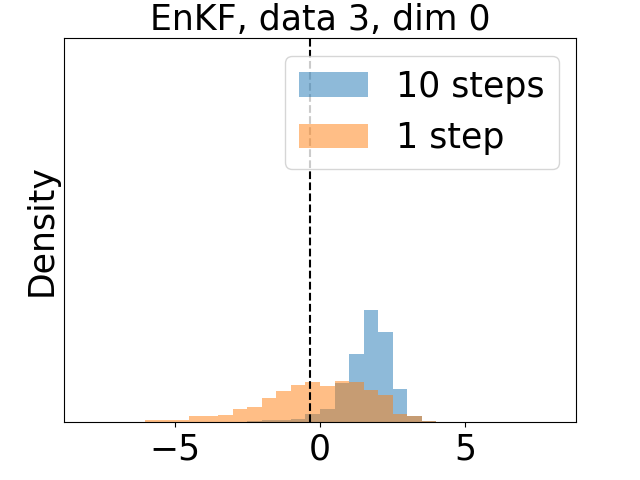}
    \includegraphics[width=0.24\columnwidth]{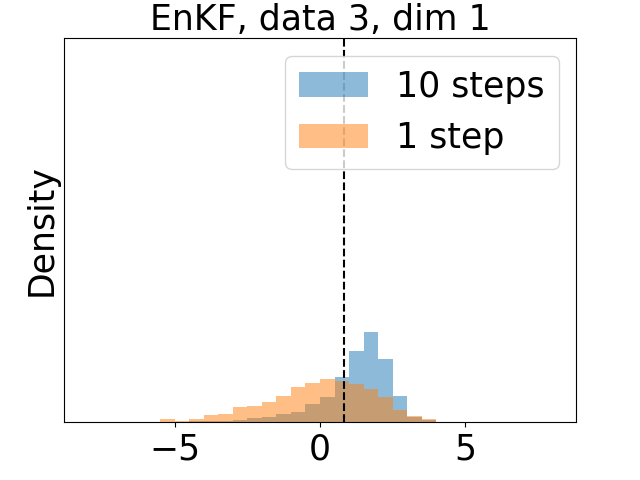}
    \includegraphics[width=0.24\columnwidth]{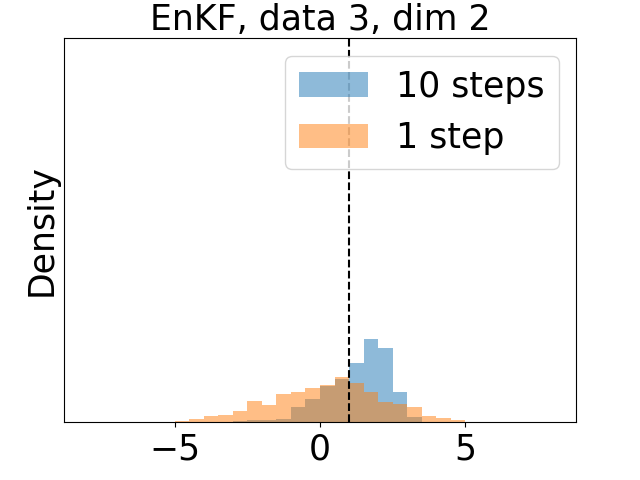}
    \includegraphics[width=0.24\columnwidth]{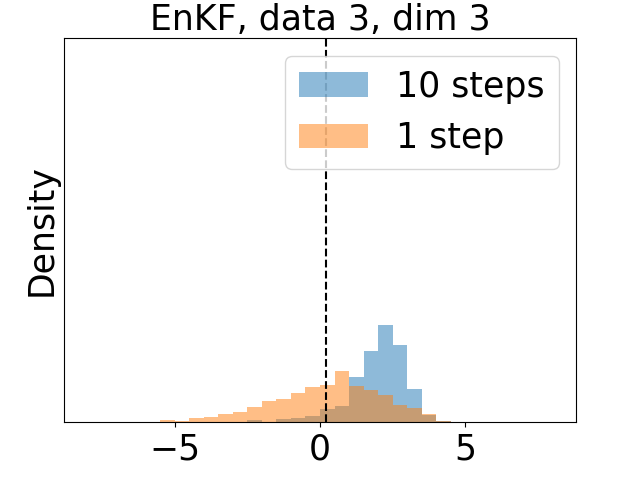}
    \includegraphics[width=0.24\columnwidth]{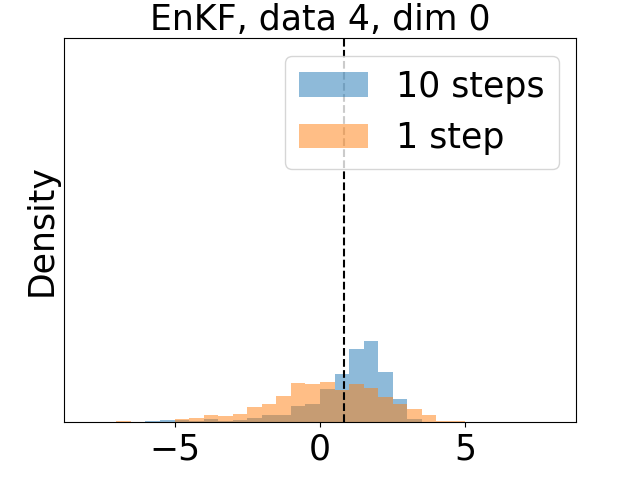}
    \includegraphics[width=0.24\columnwidth]{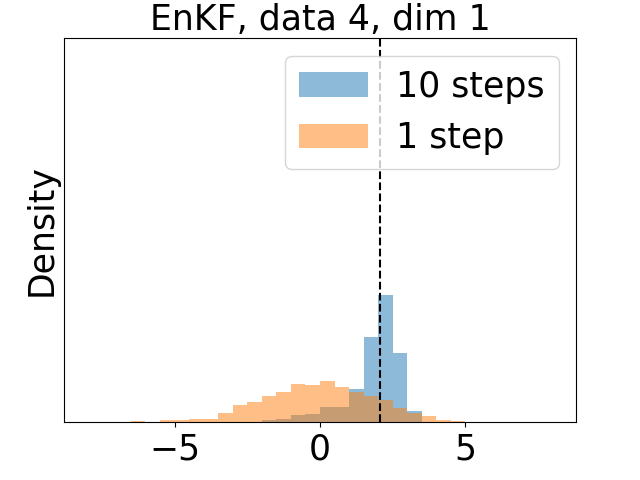}
    \includegraphics[width=0.24\columnwidth]{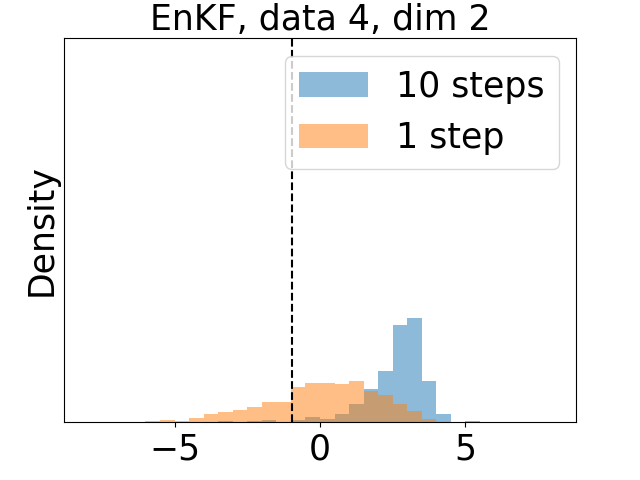}
    \includegraphics[width=0.24\columnwidth]{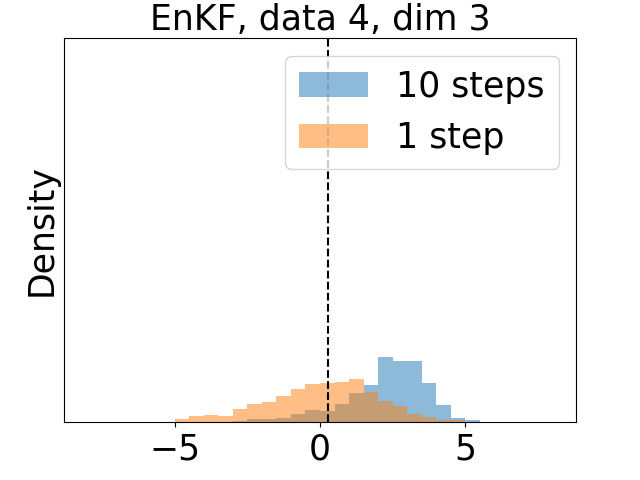}
    \includegraphics[width=0.24\columnwidth]{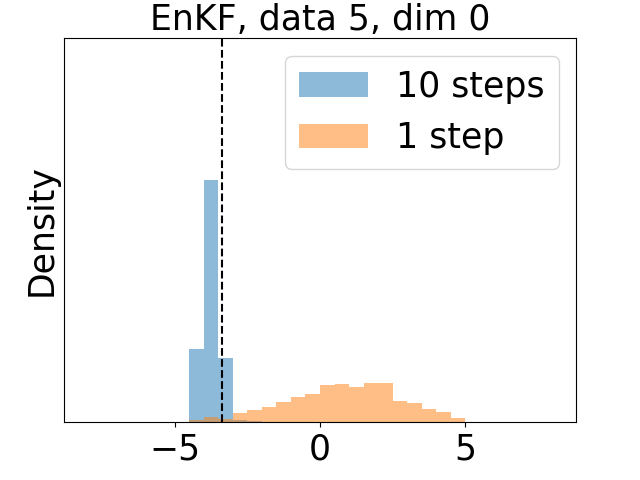}
    \includegraphics[width=0.24\columnwidth]{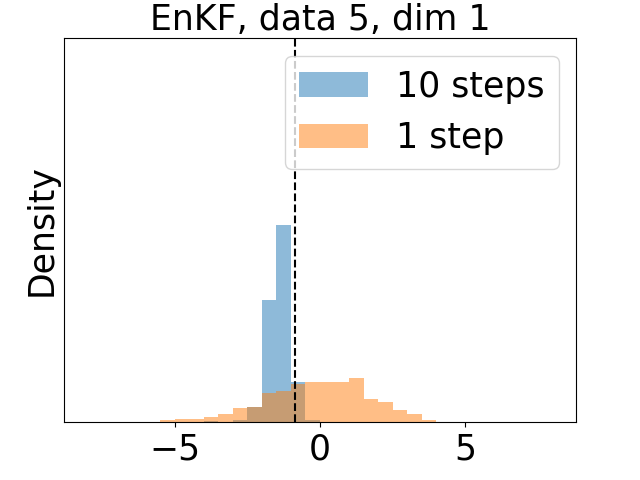}
    \includegraphics[width=0.24\columnwidth]{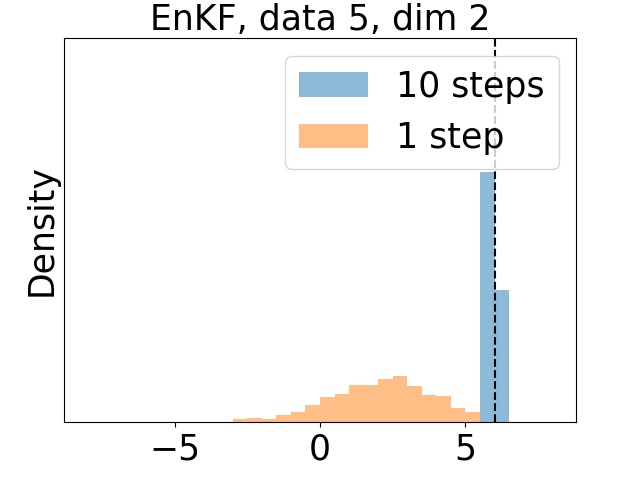}
    \includegraphics[width=0.24\columnwidth]{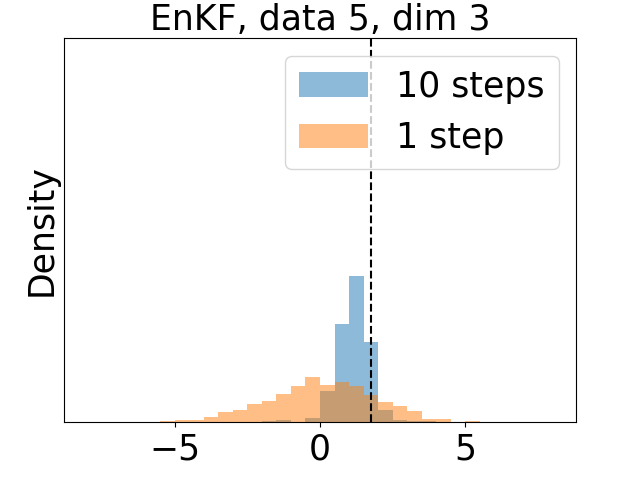}
    \caption{Histograms for the EnKF model.}
    \label{fig:EnKF_inference_example}
\end{figure}

\begin{figure}
    \centering
    \includegraphics[width=0.24\columnwidth]{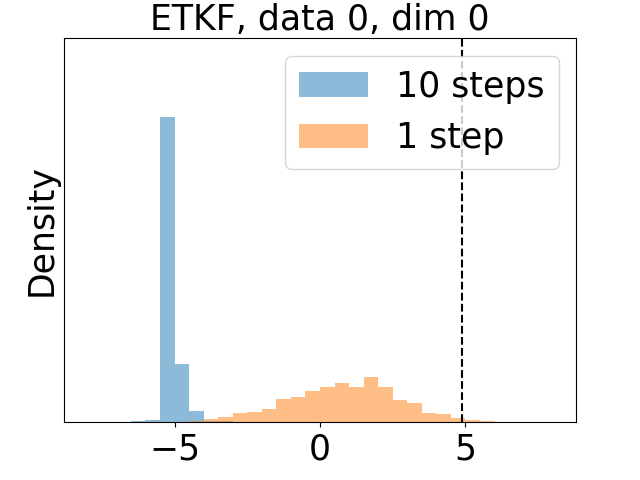}
    \includegraphics[width=0.24\columnwidth]{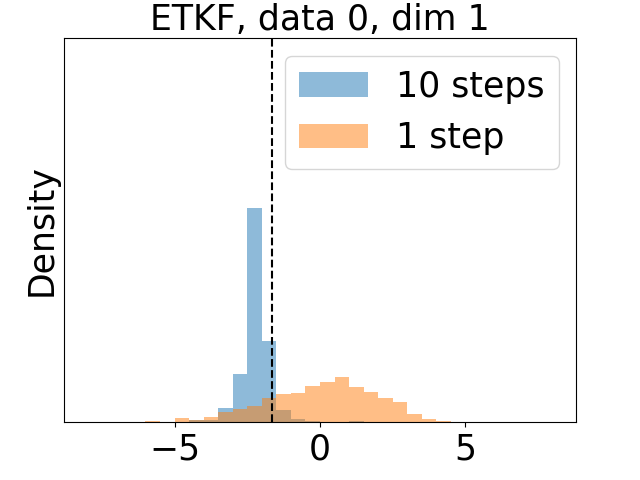}
    \includegraphics[width=0.24\columnwidth]{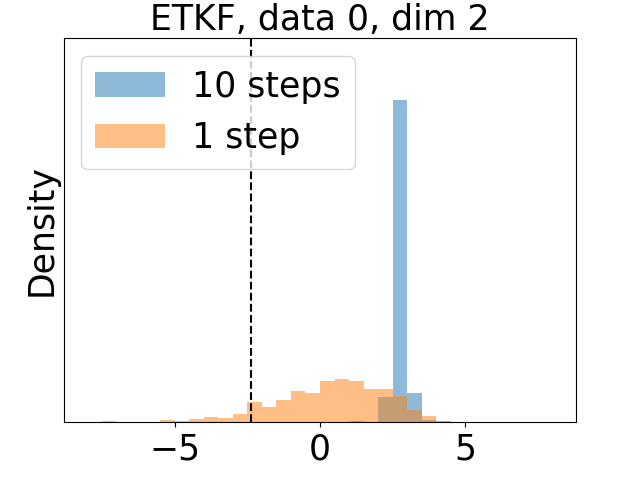}
    \includegraphics[width=0.24\columnwidth]{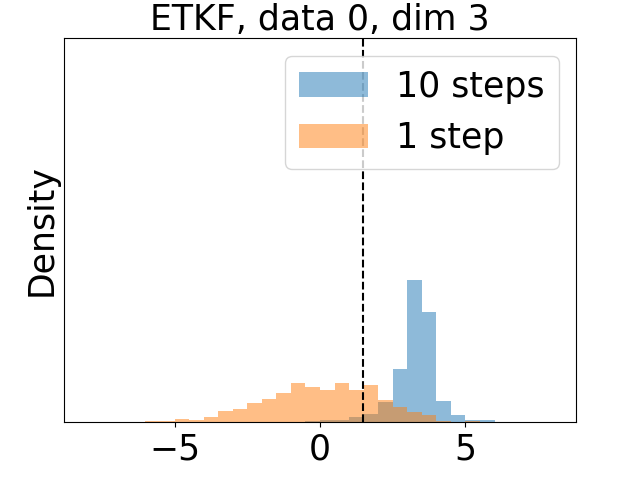}
    \includegraphics[width=0.24\columnwidth]{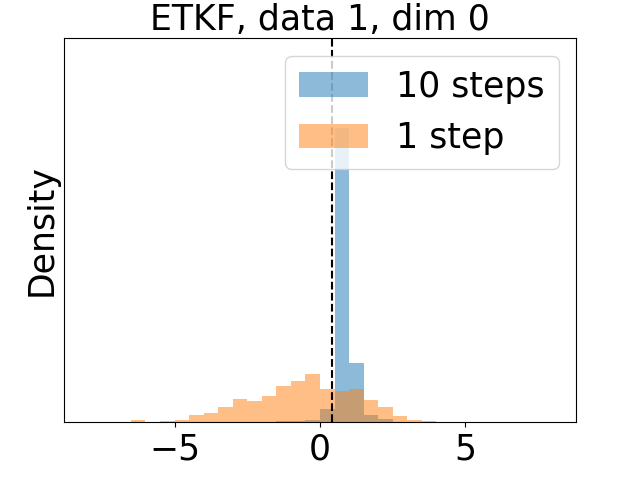}
    \includegraphics[width=0.24\columnwidth]{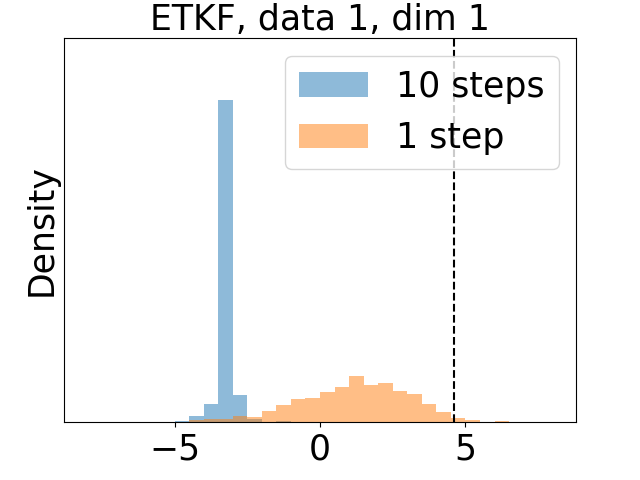}
    \includegraphics[width=0.24\columnwidth]{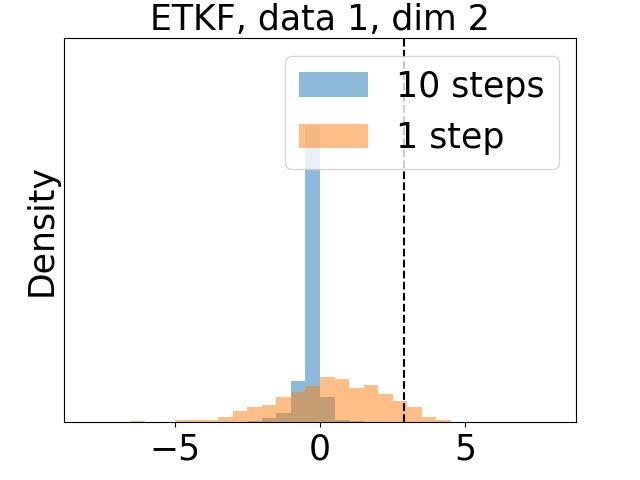}
    \includegraphics[width=0.24\columnwidth]{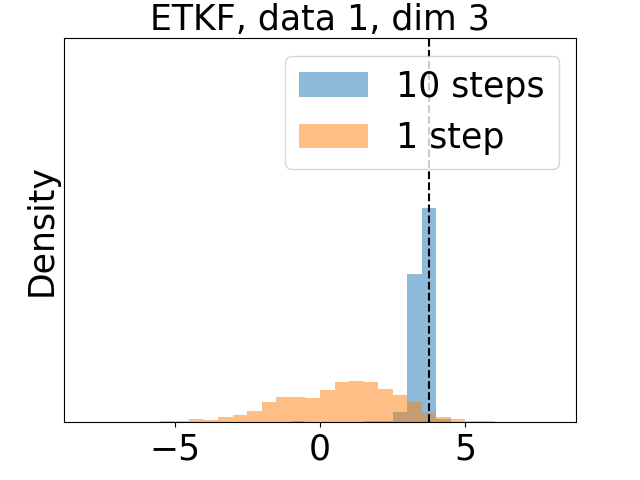}
    \includegraphics[width=0.24\columnwidth]{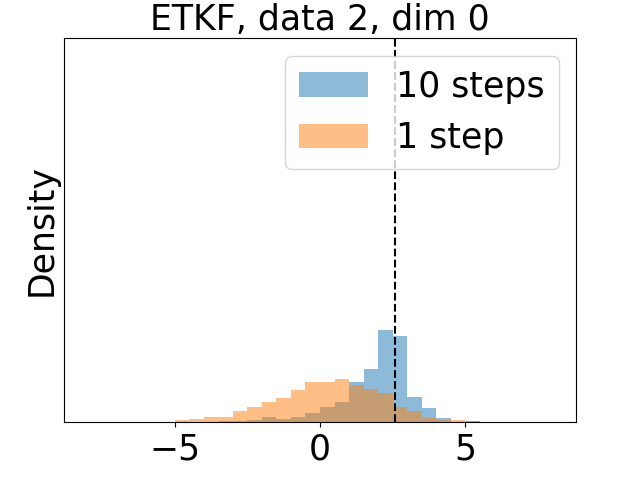}
    \includegraphics[width=0.24\columnwidth]{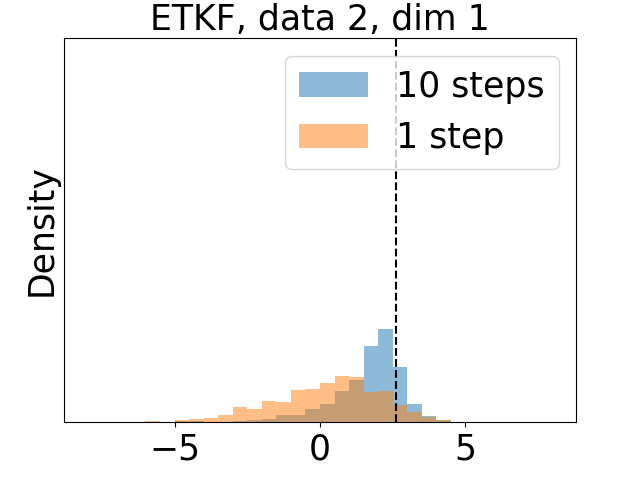}
    \includegraphics[width=0.24\columnwidth]{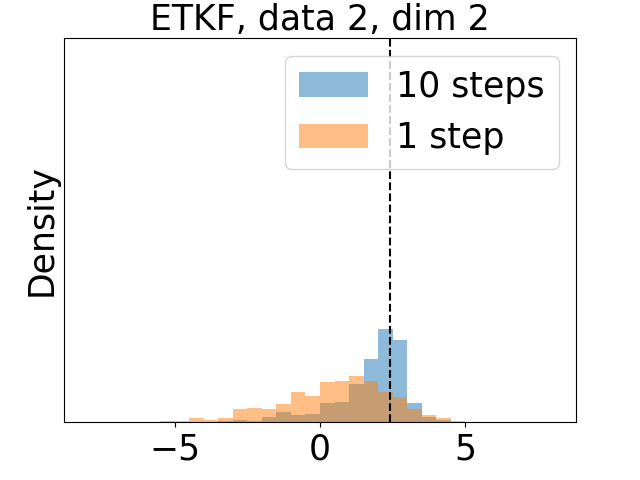}
    \includegraphics[width=0.24\columnwidth]{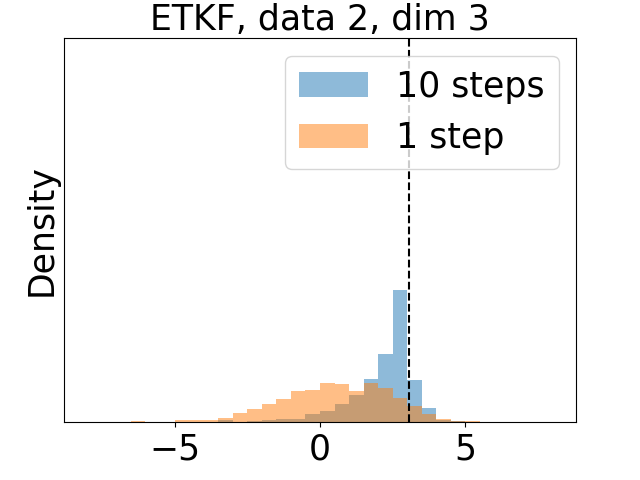}
    \includegraphics[width=0.24\columnwidth]{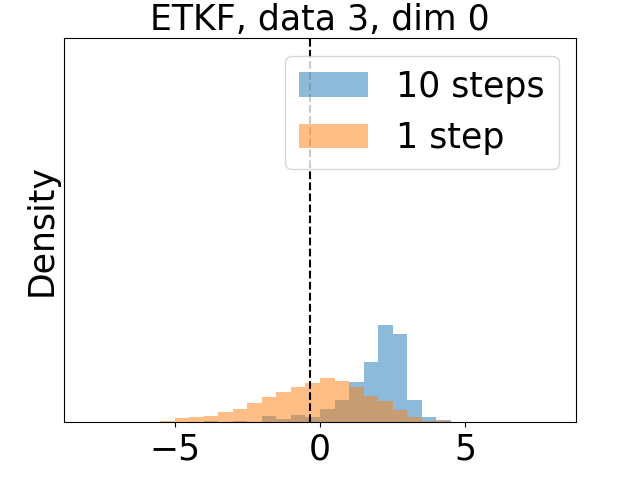}
    \includegraphics[width=0.24\columnwidth]{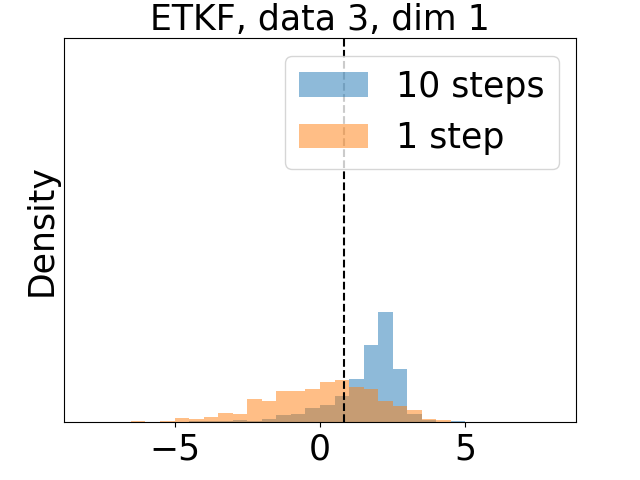}
    \includegraphics[width=0.24\columnwidth]{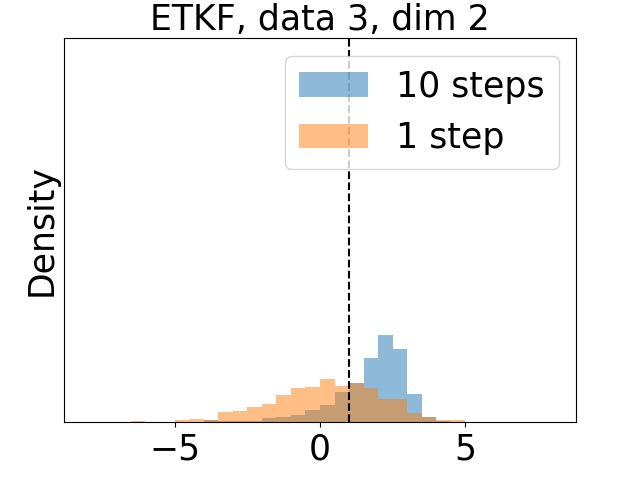}
    \includegraphics[width=0.24\columnwidth]{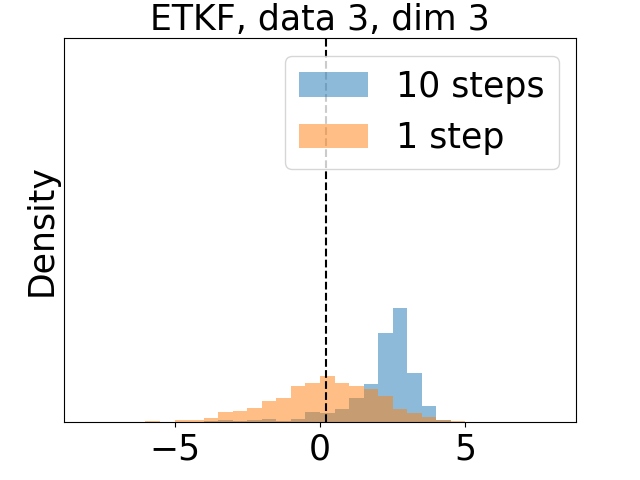}
    \includegraphics[width=0.24\columnwidth]{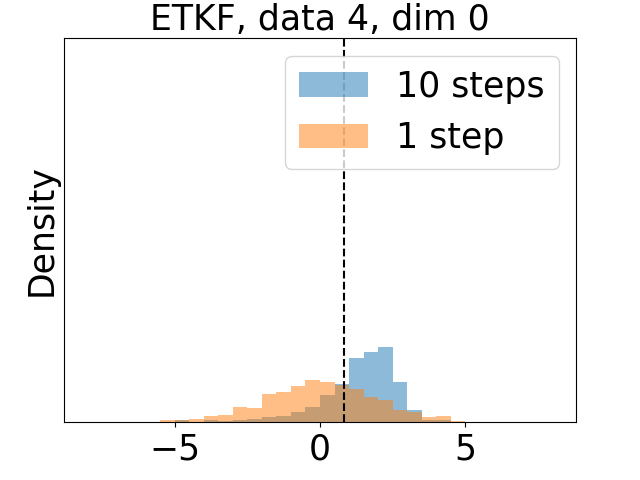}
    \includegraphics[width=0.24\columnwidth]{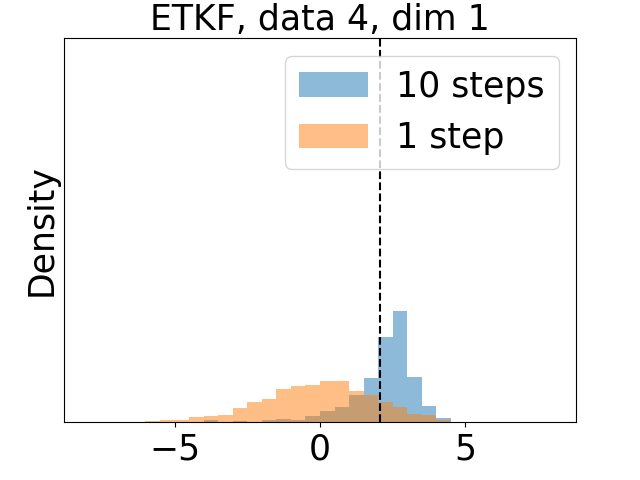}
    \includegraphics[width=0.24\columnwidth]{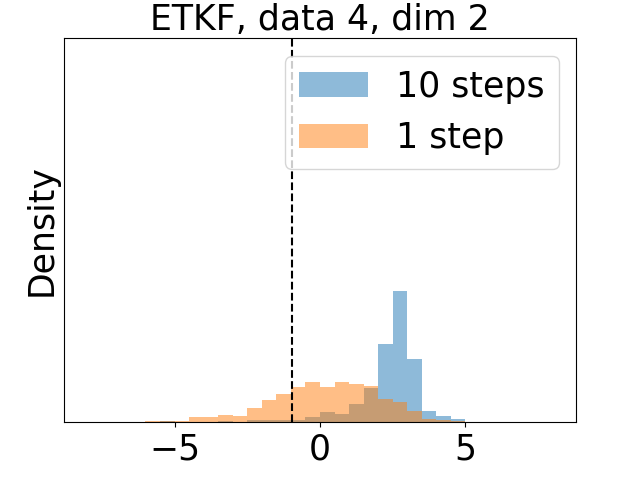}
    \includegraphics[width=0.24\columnwidth]{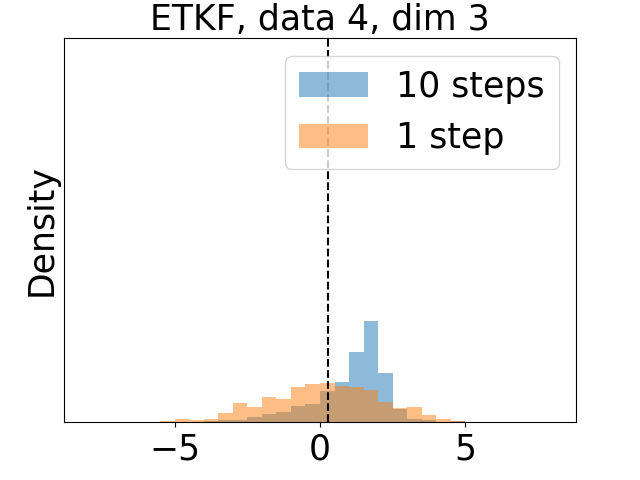}
    \includegraphics[width=0.24\columnwidth]{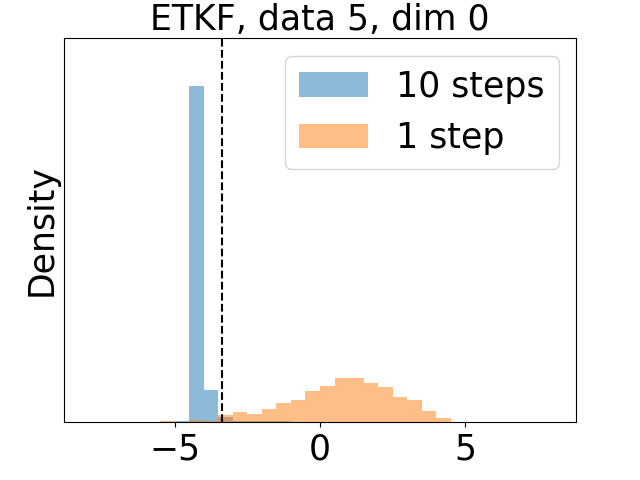}
    \includegraphics[width=0.24\columnwidth]{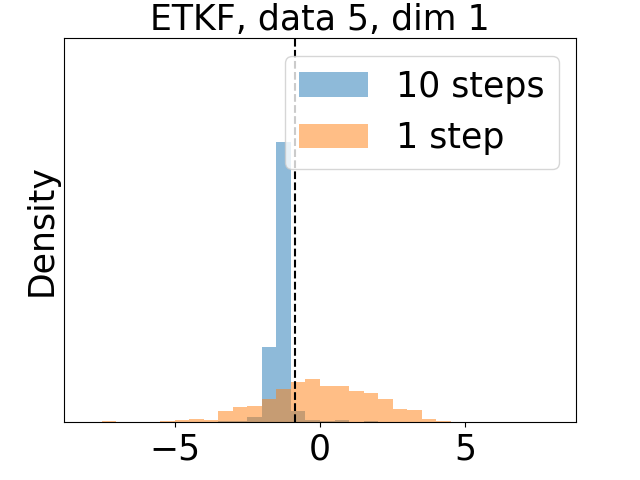}
    \includegraphics[width=0.24\columnwidth]{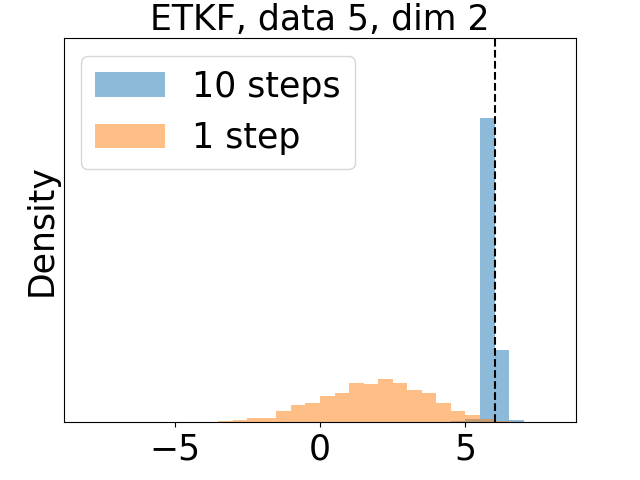}
    \includegraphics[width=0.24\columnwidth]{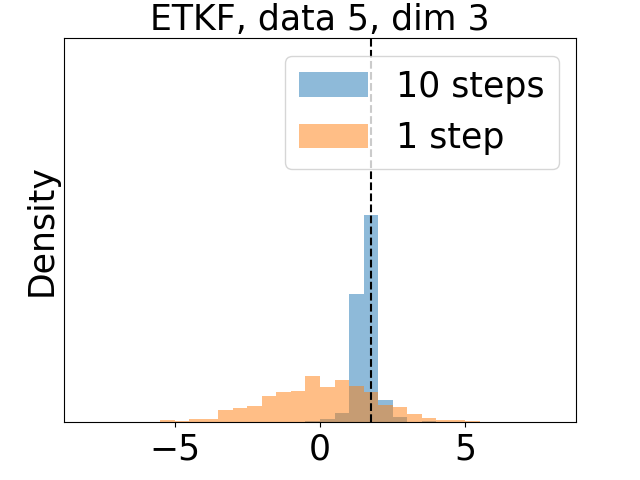}
    \caption{Histograms for the ETKF model.}
    \label{fig:ETKF_inference_example}
\end{figure}


\newpage
\input{checklist.tex}

\end{document}

%% file: checklist.tex
\clearpage

\section*{NeurIPS Paper Checklist}

\begin{enumerate}

\item {\bf Claims}
    \item[] Question: Do the main claims made in the abstract and introduction accurately reflect the paper's contributions and scope?
    \item[] Answer: \answerYes{} 
    \item[] Justification: the abstract and introduction state the paper’s actual contributions: a simulator-preserving amortized variational smoother, observation-only training, exactness in deterministic and linear-Gaussian limits, and experiments on Lorenz–96 and Kolmogorov flow. The related-work section also appropriately narrows scope by positioning PR-Smoother as a specific design point rather than as the first variational method for SSMs.

\item {\bf Limitations}
    \item[] Question: Does the paper discuss the limitations of the work performed by the authors?
    \item[] Answer: \answerYes{} 
    \item[] Justification: The paper has a dedicated Limitations section. It explicitly discusses the differentiability assumption on $(f,h)$, amortization gap under undertraining or distribution shift, identifiability limits in sparse-observation settings, and the fact that $\theta$ is treated as a point estimate.

\item {\bf Theory assumptions and proofs}
    \item[] Question: For each theoretical result, does the paper provide the full set of assumptions and a complete (and correct) proof?
    \item[] Answer: \answerYes{} 
    \item[] Justification: Proposition 3.1 states the required expressivity assumptions in the main text, gives a proof sketch there, and points to full derivations in Appendix A.

    \item {\bf Experimental result reproducibility}
    \item[] Question: Does the paper fully disclose all the information needed to reproduce the main experimental results of the paper to the extent that it affects the main claims and/or conclusions of the paper (regardless of whether the code and data are provided or not)?
    \item[] Answer: \answerYes{} 
    \item[] Justification: The paper gives a concrete reproduction path by specifying the generative setups, architectures, baselines, hyperparameters, hardware, and runtime, and it states that code for reproducing the experiments is provided as supplementary material. That meets the checklist’s “reasonable avenue for reproducibility” standard.

\item {\bf Open access to data and code}
    \item[] Question: Does the paper provide open access to the data and code, with sufficient instructions to faithfully reproduce the main experimental results, as described in supplemental material?
    \item[] Answer: \answerYes{} 
    \item[] Justification: We provide anonymized supplementary code with training, evaluation, and data-generation scripts, dependency specifications, and exact commands needed to reproduce the main experiments and baselines.

\item {\bf Experimental setting/details}
    \item[] Question: Does the paper specify all the training and test details (e.g., data splits, hyperparameters, how they were chosen, type of optimizer) necessary to understand the results?
    \item[] Answer: \answerYes{} 
    \item[] Justification: The appendix specifies data generation, observation operators, architectures, baseline settings, optimizers, learning-rate schedules, batch sizes, epochs, and runtime summaries for the reported experiments.

\item {\bf Experiment statistical significance}
    \item[] Question: Does the paper report error bars suitably and correctly defined or other appropriate information about the statistical significance of the experiments?
    \item[] Answer: \answerYes{} 
    \item[] Justification: All reported error values are explicitly defined in the text.

\item {\bf Experiments compute resources}
    \item[] Question: For each experiment, does the paper provide sufficient information on the computer resources (type of compute workers, memory, time of execution) needed to reproduce the experiments?
    \item[] Answer: \answerYes{} 
    \item[] Justification: the paper reports hardware type and memory (GH200 and RTX5090), per-model training times, and per-window inference runtimes for PR-Smoother and baselines, which is sufficient to understand the compute required for the reported experiments. NeurIPS also asks for this exact kind of information.

\item {\bf Code of ethics}
    \item[] Question: Does the research conducted in the paper conform, in every respect, with the NeurIPS Code of Ethics \url{https://neurips.cc/public/EthicsGuidelines}?
    \item[] Answer: \answerYes{} 
    \item[] Justification: We reviewed the NeurIPS Code of Ethics and ensured that the work conforms to it. The paper studies simulator-based benchmarks and does not involve human subjects, crowdsourcing, or sensitive real-world data.

\item {\bf Broader impacts}
    \item[] Question: Does the paper discuss both potential positive societal impacts and negative societal impacts of the work performed?
    \item[] Answer: \answerYes{} 
    \item[] Justification: Potential broader impacts are discussed in Sec. 6. The main positive impact is improved uncertainty-aware state estimation and parameter calibration in scientific data-assimilation workflows; the main negative impact is downstream over-trust in high-stakes monitoring or forecasting without domain-specific validation and calibration.

\item {\bf Safeguards}
    \item[] Question: Does the paper describe safeguards that have been put in place for responsible release of data or models that have a high risk for misuse (e.g., pre-trained language models, image generators, or scraped datasets)?
    \item[] Answer: \answerNA{} 
    \item[] Justification: The submission does not release a high-risk pretrained model, scraped dataset, or dual-use generative artifact. It studies a simulator-based inference method on synthetic physical benchmarks, so the safeguards question is not applicable here.

\item {\bf Licenses for existing assets}
    \item[] Question: Are the creators or original owners of assets (e.g., code, data, models), used in the paper, properly credited and are the license and terms of use explicitly mentioned and properly respected?
    \item[] Answer: \answerNA{} 
    \item[] Justification: This work does not rely on third-party code, data, or model assets requiring license reporting beyond standard software libraries.

\item {\bf New assets}
    \item[] Question: Are new assets introduced in the paper well documented and is the documentation provided alongside the assets?
    \item[] Answer: \answerYes{} 
    \item[] Justification: We release an anonymized supplementary code package documenting setup, dependencies, training, data generation, and limitations alongside the new code.

\item {\bf Crowdsourcing and research with human subjects}
    \item[] Question: For crowdsourcing experiments and research with human subjects, does the paper include the full text of instructions given to participants and screenshots, if applicable, as well as details about compensation (if any)? 
    \item[] Answer: \answerNA{} 
    \item[] Justification: The work does not use crowdsourcing or involve human participants. It is evaluated on simulator-generated physical benchmarks only.

\item {\bf Institutional review board (IRB) approvals or equivalent for research with human subjects}
    \item[] Question: Does the paper describe potential risks incurred by study participants, whether such risks were disclosed to the subjects, and whether Institutional Review Board (IRB) approvals (or an equivalent approval/review based on the requirements of your country or institution) were obtained?
    \item[] Answer: \answerNA{} 
    \item[] Justification: The work does not involve human participants, so IRB approval or an equivalent review is not applicable.

\item {\bf Declaration of LLM usage}
    \item[] Question: Does the paper describe the usage of LLMs if it is an important, original, or non-standard component of the core methods in this research? Note that if the LLM is used only for writing, editing, or formatting purposes and does \emph{not} impact the core methodology, scientific rigor, or originality of the research, declaration is not required.
    \item[] Answer: \answerNA{} 
    \item[] Justification: LLMs are not part of the core method.

\end{enumerate}

%% file: example_paper.bib
@ARTICLE{EKFNet,
  author={Xu, Liang and Niu, Ruixin},
  journal={IEEE Transactions on Signal Processing}, 
  title={EKFNet: Learning System Noise Covariance Parameters for Nonlinear Tracking}, 
  year={2024},
  volume={72},
  number={},
  pages={3139-3152},
  doi={10.1109/TSP.2024.3417350}}

@INPROCEEDINGS{KalmanNet,
  author={Revach, Guy and Shlezinger, Nir and van Sloun, Ruud J. G. and Eldar, Yonina C.},
  booktitle={ICASSP 2021 - 2021 IEEE International Conference on Acoustics, Speech and Signal Processing (ICASSP)}, 
  title={Kalmannet: Data-Driven Kalman Filtering}, 
  year={2021},
  volume={},
  number={},
  pages={3905-3909},
  doi={10.1109/ICASSP39728.2021.9413750}}

@INPROCEEDINGS{RTSNet,
  author={Ni, Xiaoyong and Revach, Guy and Shlezinger, Nir and van Sloun, Ruud J. G. and Eldar, Yonina C.},
  booktitle={ICASSP 2022 - 2022 IEEE International Conference on Acoustics, Speech and Signal Processing (ICASSP)}, 
  title={RTSNet: Deep Learning Aided Kalman Smoothing}, 
  year={2022},
  volume={},
  number={},
  pages={5902-5906},
  doi={10.1109/ICASSP43922.2022.9746487}}

@article{DVAE_review,
url = {http://dx.doi.org/10.1561/2200000089},
year = {2021},
volume = {15},
journal = {Foundations and Trends® in Machine Learning},
title = {Dynamical Variational Autoencoders: A Comprehensive Review},
doi = {10.1561/2200000089},
issn = {1935-8237},
number = {1-2},
pages = {1-175},
author = {Laurent Girin and Simon Leglaive and Xiaoyu Bie and Julien Diard and Thomas Hueber and Xavier Alameda-Pineda}
}

@inproceedings{NormalizingKalmanFilter,
 author = {de B\'{e}zenac, Emmanuel and Rangapuram, Syama Sundar and Benidis, Konstantinos and Bohlke-Schneider, Michael and Kurle, Richard and Stella, Lorenzo and Hasson, Hilaf and Gallinari, Patrick and Januschowski, Tim},
 booktitle = {Advances in Neural Information Processing Systems},
 editor = {H. Larochelle and M. Ranzato and R. Hadsell and M.F. Balcan and H. Lin},
 pages = {2995--3007},
 publisher = {Curran Associates, Inc.},
 title = {Normalizing Kalman Filters for Multivariate Time Series Analysis},
 url = {https://proceedings.neurips.cc/paper_files/paper/2020/file/1f47cef5e38c952f94c5d61726027439-Paper.pdf},
 volume = {33},
 year = {2020}
}

@article{4DVar1,
author = {Dimet, François-Xavier Le and Talagrand, Olivier},
title = {Variational algorithms for analysis and assimilation of meteorological observations: theoretical aspects},
journal = {Tellus A},
volume = {38A},
number = {2},
pages = {97-110},
doi = {https://doi.org/10.1111/j.1600-0870.1986.tb00459.x},
url = {https://onlinelibrary.wiley.com/doi/abs/10.1111/j.1600-0870.1986.tb00459.x},
eprint = {https://onlinelibrary.wiley.com/doi/pdf/10.1111/j.1600-0870.1986.tb00459.x},
year = {1986}
}

@article{4DVar2,
author = {Lewis, John M. and Derber, John C.},
title = {The use of adjoint equations to solve a variational adjustment problem with advective constraints},
journal = {Tellus A},
volume = {37A},
number = {4},
pages = {309-322},
doi = {https://doi.org/10.1111/j.1600-0870.1985.tb00430.x},
url = {https://onlinelibrary.wiley.com/doi/abs/10.1111/j.1600-0870.1985.tb00430.x},
eprint = {https://onlinelibrary.wiley.com/doi/pdf/10.1111/j.1600-0870.1985.tb00430.x},
year = {1985}
}

@article{EnKF_Evensen,
author = {Evensen, Geir},
title = {Sequential data assimilation with a nonlinear quasi-geostrophic model using Monte Carlo methods to forecast error statistics},
journal = {Journal of Geophysical Research: Oceans},
volume = {99},
number = {C5},
pages = {10143-10162},
doi = {https://doi.org/10.1029/94JC00572},
url = {https://agupubs.onlinelibrary.wiley.com/doi/abs/10.1029/94JC00572},
eprint = {https://agupubs.onlinelibrary.wiley.com/doi/pdf/10.1029/94JC00572},
year = {1994}
}

@article {ETKF,
      author = "Craig H. Bishop and Brian J. Etherton and Sharanya J. Majumdar",
      title = "Adaptive Sampling with the Ensemble Transform Kalman Filter. Part I: Theoretical Aspects",
      journal = "Monthly Weather Review",
      year = "2001",
      publisher = "American Meteorological Society",
      address = "Boston MA, USA",
      volume = "129",
      number = "3",
      doi = "10.1175/1520-0493(2001)129<0420:ASWTET>2.0.CO;2",
      pages=      "420 - 436",
      url = "https://journals.ametsoc.org/view/journals/mwre/129/3/1520-0493_2001_129_0420_aswtet_2.0.co_2.xml"
}

@article{IEnKS,
  TITLE = {{An iterative ensemble Kalman smoother}},
  AUTHOR = {Bocquet, Marc and Sakov, Pavel},
  URL = {https://inria.hal.science/hal-00918488},
  JOURNAL = {{Quarterly Journal of the Royal Meteorological Society}},
  PUBLISHER = {{Wiley}},
  VOLUME = {140},
  NUMBER = {682},
  PAGES = {1521-1535},
  YEAR = {2014},
  MONTH = Jul,
  DOI = {10.1002/qj.2236},
  HAL_ID = {hal-00918488},
  HAL_VERSION = {v1},
}

@inproceedings{PFF07,
author = {Fred Daum and Jim Huang},
title = {{Nonlinear filters with log-homotopy}},
volume = {6699},
booktitle = {Signal and Data Processing of Small Targets 2007},
editor = {Oliver E. Drummond and Richard D. Teichgraeber},
organization = {International Society for Optics and Photonics},
publisher = {SPIE},
pages = {669918},
year = {2007},
doi = {10.1117/12.725684},
URL = {https://doi.org/10.1117/12.725684}
}

@InProceedings{DBF,
  title = 	 {Deep {B}ayesian Filter for {B}ayes-Faithful Data Assimilation},
  author =       {Tarumi, Yuta and Fukuda, Keisuke and Maeda, Shin-Ichi},
  booktitle = 	 {Proceedings of the 42nd International Conference on Machine Learning},
  pages = 	 {59182--59209},
  year = 	 {2025},
  editor = 	 {Singh, Aarti and Fazel, Maryam and Hsu, Daniel and Lacoste-Julien, Simon and Berkenkamp, Felix and Maharaj, Tegan and Wagstaff, Kiri and Zhu, Jerry},
  volume = 	 {267},
  series = 	 {Proceedings of Machine Learning Research},
  month = 	 {13--19 Jul},
  publisher =    {PMLR},
  url = 	 {https://proceedings.mlr.press/v267/tarumi25a.html},}

@article{
SBI_review,
author = {Kyle Cranmer  and Johann Brehmer  and Gilles Louppe },
title = {The frontier of simulation-based inference},
journal = {Proceedings of the National Academy of Sciences},
volume = {117},
number = {48},
pages = {30055-30062},
year = {2020},
doi = {10.1073/pnas.1912789117},
URL = {https://www.pnas.org/doi/abs/10.1073/pnas.1912789117},
eprint = {https://www.pnas.org/doi/pdf/10.1073/pnas.1912789117}}

@article{Fablet21,
author = {Fablet, R. and Chapron, B. and Drumetz, L. and Mémin, E. and Pannekoucke, O. and Rousseau, F.},
title = {Learning Variational Data Assimilation Models and Solvers},
journal = {Journal of Advances in Modeling Earth Systems},
volume = {13},
number = {10},
pages = {e2021MS002572},
doi = {https://doi.org/10.1029/2021MS002572},
url = {https://agupubs.onlinelibrary.wiley.com/doi/abs/10.1029/2021MS002572},
eprint = {https://agupubs.onlinelibrary.wiley.com/doi/pdf/10.1029/2021MS002572},
note = {e2021MS002572 2021MS002572},
year = {2021}
}

@Article{Beauchamp23,
AUTHOR = {Beauchamp, M. and Febvre, Q. and Georgenthum, H. and Fablet, R.},
TITLE = {4DVarNet-SSH: end-to-end learning of variational interpolation schemes for nadir and wide-swath satellite altimetry},
JOURNAL = {Geoscientific Model Development},
VOLUME = {16},
YEAR = {2023},
NUMBER = {8},
PAGES = {2119--2147},
URL = {https://gmd.copernicus.org/articles/16/2119/2023/},
DOI = {10.5194/gmd-16-2119-2023}
}

@misc{DKF,
      title={Deep Kalman Filters}, 
      author={Rahul G. Krishnan and Uri Shalit and David Sontag},
      year={2015},
      eprint={1511.05121},
      archivePrefix={arXiv},
      primaryClass={stat.ML},
      url={https://arxiv.org/abs/1511.05121}, 
}

@Article{2018Bonavita,
AUTHOR = {Bonavita, M. and Lean, P. and Holm, E.},
TITLE = {Nonlinear effects in 4D-Var},
JOURNAL = {Nonlinear Processes in Geophysics},
VOLUME = {25},
YEAR = {2018},
NUMBER = {3},
PAGES = {713--729},
URL = {https://npg.copernicus.org/articles/25/713/2018/},
DOI = {10.5194/npg-25-713-2018}
}

@article{Carassi_DAreview,
author = {Carrassi, Alberto and Bocquet, Marc and Bertino, Laurent and Evensen, Geir},
title = {Data assimilation in the geosciences: An overview of methods, issues, and perspectives},
journal = {WIREs Climate Change},
volume = {9},
number = {5},
pages = {e535},
doi = {https://doi.org/10.1002/wcc.535},
url = {https://wires.onlinelibrary.wiley.com/doi/abs/10.1002/wcc.535},
eprint = {https://wires.onlinelibrary.wiley.com/doi/pdf/10.1002/wcc.535},
year = {2018}
}

@article{ERA5,
author = {Hersbach, Hans and Bell, Bill and Berrisford, Paul and Hirahara, Shoji and Horányi, András and Muñoz-Sabater, Joaquín and Nicolas, Julien and Peubey, Carole and Radu, Raluca and Schepers, Dinand and Simmons, Adrian and Soci, Cornel and Abdalla, Saleh and Abellan, Xavier and Balsamo, Gianpaolo and Bechtold, Peter and Biavati, Gionata and Bidlot, Jean and Bonavita, Massimo and De Chiara, Giovanna and Dahlgren, Per and Dee, Dick and Diamantakis, Michail and Dragani, Rossana and Flemming, Johannes and Forbes, Richard and Fuentes, Manuel and Geer, Alan and Haimberger, Leo and Healy, Sean and Hogan, Robin J. and Hólm, Elías and Janisková, Marta and Keeley, Sarah and Laloyaux, Patrick and Lopez, Philippe and Lupu, Cristina and Radnoti, Gabor and de Rosnay, Patricia and Rozum, Iryna and Vamborg, Freja and Villaume, Sebastien and Thépaut, Jean-Noël},
title = {The ERA5 global reanalysis},
journal = {Quarterly Journal of the Royal Meteorological Society},
volume = {146},
number = {730},
pages = {1999-2049},
doi = {https://doi.org/10.1002/qj.3803},
url = {https://rmets.onlinelibrary.wiley.com/doi/abs/10.1002/qj.3803},
eprint = {https://rmets.onlinelibrary.wiley.com/doi/pdf/10.1002/qj.3803},
year = {2020}
}

@article{
GraphCast,
author = {Remi Lam  and Alvaro Sanchez-Gonzalez  and Matthew Willson  and Peter Wirnsberger  and Meire Fortunato  and Ferran Alet  and Suman Ravuri  and Timo Ewalds  and Zach Eaton-Rosen  and Weihua Hu  and Alexander Merose  and Stephan Hoyer  and George Holland  and Oriol Vinyals  and Jacklynn Stott  and Alexander Pritzel  and Shakir Mohamed  and Peter Battaglia },
title = {Learning skillful medium-range global weather forecasting},
journal = {Science},
volume = {382},
number = {6677},
pages = {1416-1421},
year = {2023},
doi = {10.1126/science.adi2336},
URL = {https://www.science.org/doi/abs/10.1126/science.adi2336},
eprint = {https://www.science.org/doi/pdf/10.1126/science.adi2336},
}

@article{PanguWeather,
  author = {Bi, Kaifeng and Xie, Lingxi and Zhang, Hengheng and Chen, Xin and Gu, Xiaotao and Tian, Qi},
  title = {Accurate medium-range global weather forecasting with 3D neural networks},
  journal = {Nature},
  volume = {619},
  pages = {533--538},
  year = {2023},
  publisher = {Nature Publishing Group},
  doi = {10.1038/s41586-023-06185-3},
  url = {https://www.nature.com/articles/s41586-023-06185-3},
}

@misc{FourCastNet,
      title={FourCastNet: A Global Data-driven High-resolution Weather Model using Adaptive Fourier Neural Operators}, 
      author={Jaideep Pathak and Shashank Subramanian and Peter Harrington and Sanjeev Raja and Ashesh Chattopadhyay and Morteza Mardani and Thorsten Kurth and David Hall and Zongyi Li and Kamyar Azizzadenesheli and Pedram Hassanzadeh and Karthik Kashinath and Animashree Anandkumar},
      year={2022},
      eprint={2202.11214},
      archivePrefix={arXiv},
      primaryClass={physics.ao-ph},
      url={https://arxiv.org/abs/2202.11214}, 
}

@article{VenusDA,
title = "Development of an ensemble Kalman filter data assimilation system for the Venusian atmosphere",
author = "Norihiko Sugimoto and Akira Yamazaki and Toru Kouyama and Hiroki Kashimura and Takeshi Enomoto and Masahiro Takagi",
year = "2017",
month = dec,
day = "1",
doi = "10.1038/s41598-017-09461-1",
language = "English",
volume = "7",
journal = "Scientific reports",
issn = "2045-2322",
publisher = "Nature Research",
number = "1",
}

@article{SpaceWeather1,
author = {Lang, Matthew and Browne, Philip and van Leeuwen, Peter Jan and Owens, Mathew},
title = {Data Assimilation in the Solar Wind: Challenges and First Results},
journal = {Space Weather},
volume = {15},
number = {11},
pages = {1490-1510},
doi = {https://doi.org/10.1002/2017SW001681},
url = {https://agupubs.onlinelibrary.wiley.com/doi/abs/10.1002/2017SW001681},
eprint = {https://agupubs.onlinelibrary.wiley.com/doi/pdf/10.1002/2017SW001681},
year = {2017}
}

@ARTICLE{SpaceWeather2,
       author = {{Schunk}, R.~W. and {Scherliess}, L. and {Eccles}, V. and {Gardner}, L.~C. and {Sojka}, J.~J. and {Zhu}, L. and {Pi}, X. and {Mannucci}, A.~J. and {Komjathy}, A. and {Wang}, C. and {Rosen}, G.},
        title = "{Challenges in Specifying and Predicting Space Weather}",
      journal = {Space Weather},
         year = 2021,
        month = feb,
       volume = {19},
       number = {2},
          eid = {e2019SW002404},
        pages = {e2019SW002404},
          doi = {10.1029/2019SW002404},
       adsurl = {https://ui.adsabs.harvard.edu/abs/2021SpWea..1902404S}
}

@article {2010Bocquet,
      author = "Marc Bocquet and Carlos A. Pires and Lin Wu",
      title = "Beyond Gaussian Statistical Modeling in Geophysical Data Assimilation",
      journal = "Monthly Weather Review",
      year = "2010",
      publisher = "American Meteorological Society",
      address = "Boston MA, USA",
      volume = "138",
      number = "8",
      doi = "10.1175/2010MWR3164.1",
      pages=      "2997 - 3023",
      url = "https://journals.ametsoc.org/view/journals/mwre/138/8/2010mwr3164.1.xml"
}

@phdthesis{Lorenz96,
  author = {E.N. Lorenz},
  title = {Predictability: a problem partly solved},
  year = {1995},
  journal = {Seminar on Predictability, 4-8 September 1995},
  volume = {1},
  pages = {1-18},
  publisher = {ECMWF},
  address = {Shinfield Park, Reading},
  language = {eng},
}

@article{VarQC1,
author = {Anderson, Erik and Järvinen, Heikki},
title = {Variational quality control},
journal = {Quarterly Journal of the Royal Meteorological Society},
volume = {125},
number = {554},
pages = {697-722},
doi = {https://doi.org/10.1002/qj.49712555416},
url = {https://rmets.onlinelibrary.wiley.com/doi/abs/10.1002/qj.49712555416},
eprint = {https://rmets.onlinelibrary.wiley.com/doi/pdf/10.1002/qj.49712555416},
year = {1999}
}

@ARTICLE{VarQC2,
       author = {{Storto}, Andrea},
        title = "{Variational quality control of hydrographic profile data with non-Gaussian errors for global ocean variational data assimilation systems}",
      journal = {Ocean Modelling},
         year = 2016,
        month = aug,
       volume = {104},
        pages = {226-241},
          doi = {10.1016/j.ocemod.2016.06.011},
       adsurl = {https://ui.adsabs.harvard.edu/abs/2016OcMod.104..226S}
}

@Article{2019Geer_AllSky,
AUTHOR = {Geer, A. J. and Migliorini, S. and Matricardi, M.},
TITLE = {All-sky assimilation of infrared radiances sensitive to mid- and upper-tropospheric moisture and cloud},
JOURNAL = {Atmospheric Measurement Techniques},
VOLUME = {12},
YEAR = {2019},
NUMBER = {9},
PAGES = {4903--4929},
URL = {https://amt.copernicus.org/articles/12/4903/2019/},
DOI = {10.5194/amt-12-4903-2019}
}

@misc{FlowDAS,
      title={FlowDAS: A Stochastic Interpolant-based Framework for Data Assimilation}, 
      author={Siyi Chen and Yixuan Jia and Qing Qu and He Sun and Jeffrey A Fessler},
      year={2025},
      eprint={2501.16642},
      archivePrefix={arXiv},
      primaryClass={eess.SP},
      url={https://arxiv.org/abs/2501.16642}, 
}

@article{2013Reich,
author = {Reich, Sebastian},
title = {A Nonparametric Ensemble Transform Method for Bayesian Inference},
journal = {SIAM Journal on Scientific Computing},
volume = {35},
number = {4},
pages = {A2013-A2024},
year = {2013},
doi = {10.1137/130907367},
URL = {
        https://doi.org/10.1137/130907367
},
eprint = { 
        https://doi.org/10.1137/130907367
}
}

@article{2019vanLeeuwen_PFReview,
author = {van Leeuwen, Peter Jan and Künsch, Hans R. and Nerger, Lars and Potthast, Roland and Reich, Sebastian},
title = {Particle filters for high-dimensional geoscience applications: A review},
journal = {Quarterly Journal of the Royal Meteorological Society},
volume = {145},
number = {723},
pages = {2335-2365},
doi = {https://doi.org/10.1002/qj.3551},
url = {https://rmets.onlinelibrary.wiley.com/doi/abs/10.1002/qj.3551},
eprint = {https://rmets.onlinelibrary.wiley.com/doi/pdf/10.1002/qj.3551},
year = {2019}
}

@InProceedings{2021frerix,
  title = 	 {Variational Data Assimilation with a Learned Inverse Observation Operator},
  author =       {Frerix, Thomas and Kochkov, Dmitrii and Smith, Jamie and Cremers, Daniel and Brenner, Michael and Hoyer, Stephan},
  booktitle = 	 {Proceedings of the 38th International Conference on Machine Learning},
  pages = 	 {3449--3458},
  year = 	 {2021},
  editor = 	 {Meila, Marina and Zhang, Tong},
  volume = 	 {139},
  series = 	 {Proceedings of Machine Learning Research},
  month = 	 {18--24 Jul},
  publisher =    {PMLR},
  url = 	 {https://proceedings.mlr.press/v139/frerix21a.html},
}

@article{ABVI,
  author  = {Mathis Chagneux and Elisabeth Gassiat and Pierre Gloaguen and Sylvain Le Corff},
  title   = {Additive smoothing error in backward variational inference for general state-space models},
  journal = {Journal of Machine Learning Research},
  year    = {2024},
  volume  = {25},
  number  = {28},
  pages   = {1--33},
  url     = {http://jmlr.org/papers/v25/22-1392.html}
}

@inproceedings{SDA,
 author = {Rozet, Fran\c{c}ois and Louppe, Gilles},
 booktitle = {Advances in Neural Information Processing Systems},
 editor = {A. Oh and T. Naumann and A. Globerson and K. Saenko and M. Hardt and S. Levine},
 pages = {40521--40541},
 publisher = {Curran Associates, Inc.},
 title = {Score-based Data Assimilation},
 url = {https://proceedings.neurips.cc/paper_files/paper/2023/file/7f7fa581cc8a1970a4332920cdf87395-Paper-Conference.pdf},
 volume = {36},
 year = {2023}
}

@inproceedings{
RealNVP,
title={Density estimation using Real {NVP}},
author={Laurent Dinh and Jascha Sohl-Dickstein and Samy Bengio},
booktitle={International Conference on Learning Representations},
year={2017},
url={https://openreview.net/forum?id=HkpbnH9lx}
}

@misc{ECMWF_4DVar_difficulty,
  author = {Erik Andersson and Mike Fisher and E. Hólm and L. Isaksen and G. Radnóti and Y. Trémolet},
  title = {Will the 4D-Var approach be defeated by nonlinearity?},
  year = {2005},
  journal = {ECMWF Technical Memoranda},
  number = {479},
  pages = {26},
  publisher = {ECMWF},
  address = {Shinfield Park, Reading},
  url = {https://www.ecmwf.int/node/7768},
  doi = {10.21957/92th7r03d},
  language = {eng},
}

@inproceedings{FIVO,
 author = {Maddison, Chris J and Lawson, John and Tucker, George and Heess, Nicolas and Norouzi, Mohammad and Mnih, Andriy and Doucet, Arnaud and Teh, Yee},
 booktitle = {Advances in Neural Information Processing Systems},
 editor = {I. Guyon and U. Von Luxburg and S. Bengio and H. Wallach and R. Fergus and S. Vishwanathan and R. Garnett},
 pages = {},
 publisher = {Curran Associates, Inc.},
 title = {Filtering Variational Objectives},
 url = {https://proceedings.neurips.cc/paper_files/paper/2017/file/fa84632d742f2729dc32ce8cb5d49733-Paper.pdf},
 volume = {30},
 year = {2017}
}

@InProceedings{VSMC,
  title = 	 {Variational Sequential Monte Carlo},
  author = 	 {Naesseth, Christian and Linderman, Scott and Ranganath, Rajesh and Blei, David},
  booktitle = 	 {Proceedings of the Twenty-First International Conference on Artificial Intelligence and Statistics},
  pages = 	 {968--977},
  year = 	 {2018},
  editor = 	 {Storkey, Amos and Perez-Cruz, Fernando},
  volume = 	 {84},
  series = 	 {Proceedings of Machine Learning Research},
  month = 	 {09--11 Apr},
  publisher =    {PMLR},
  url = 	 {https://proceedings.mlr.press/v84/naesseth18a.html},
}

@article{Courts2023,
  author  = {Jarrad Courts and Adrian G. Wills and Thomas B. Sch{\"o}n and Brett Ninness},
  title   = {Variational system identification for nonlinear state-space models},
  journal = {Automatica},
  volume  = {147},
  pages   = {110687},
  year    = {2023},
  doi     = {10.1016/j.automatica.2022.110687}
}

@article{Dutra2025,
  author  = {Dimas Abreu Archanjo Dutra},
  title   = {Parameterizations for large-scale variational system identification using unconstrained optimization},
  journal = {Automatica},
  volume  = {173},
  pages   = {112086},
  year    = {2025},
  doi     = {10.1016/j.automatica.2024.112086}
}

@InProceedings{ryder18,
  title = 	 {Black-Box Variational Inference for Stochastic Differential Equations},
  author =       {Ryder, Tom and Golightly, Andrew and McGough, A. Stephen and Prangle, Dennis},
  booktitle = 	 {Proceedings of the 35th International Conference on Machine Learning},
  pages = 	 {4423--4432},
  year = 	 {2018},
  editor = 	 {Dy, Jennifer and Krause, Andreas},
  volume = 	 {80},
  series = 	 {Proceedings of Machine Learning Research},
  month = 	 {10--15 Jul},
  publisher =    {PMLR},
  url = 	 {https://proceedings.mlr.press/v80/ryder18a.html}
}

@InProceedings{ryder21,
  title = 	 {The neural moving average model for scalable variational inference of state space models},
  author =       {Ryder, Thomas and Prangle, Dennis and Golightly, Andrew and Matthews, Isaac},
  booktitle = 	 {Proceedings of the Thirty-Seventh Conference on Uncertainty in Artificial Intelligence},
  pages = 	 {12--22},
  year = 	 {2021},
  editor = 	 {de Campos, Cassio and Maathuis, Marloes H.},
  volume = 	 {161},
  series = 	 {Proceedings of Machine Learning Research},
  month = 	 {27--30 Jul},
  publisher =    {PMLR},
  url = 	 {https://proceedings.mlr.press/v161/ryder21a.html}
}

@misc{AFSF,
      title={Amortized Filtering and Smoothing with Conditional Normalizing Flows}, 
      author={Tiangang Cui and Xiaodong Feng and Chenlong Pei and Xiaoliang Wan and Tao Zhou},
      year={2026},
      eprint={2604.07169},
      archivePrefix={arXiv},
      primaryClass={stat.ML},
      url={https://arxiv.org/abs/2604.07169}, 
}

@inproceedings{Frigola2014,
 author = {Frigola, Roger and Chen, Yutian and Rasmussen, Carl E.},
 booktitle = {Advances in Neural Information Processing Systems},
 editor = {Z. Ghahramani and M. Welling and C. Cortes and N. Lawrence and K.Q. Weinberger},
 pages = {},
 publisher = {Curran Associates, Inc.},
 title = {Variational Gaussian Process State-Space Models},
 url = {https://proceedings.neurips.cc/paper_files/paper/2014/file/875b60cf697a9e09e12c0f07b982e431-Paper.pdf},
 volume = {27},
 year = {2014}
}

@InProceedings{Naesseth2018,
  title = 	 {Variational Sequential Monte Carlo},
  author = 	 {Naesseth, Christian and Linderman, Scott and Ranganath, Rajesh and Blei, David},
  booktitle = 	 {Proceedings of the Twenty-First International Conference on Artificial Intelligence and Statistics},
  pages = 	 {968--977},
  year = 	 {2018},
  editor = 	 {Storkey, Amos and Perez-Cruz, Fernando},
  volume = 	 {84},
  series = 	 {Proceedings of Machine Learning Research},
  month = 	 {09--11 Apr},
  publisher =    {PMLR},
  url = 	 {https://proceedings.mlr.press/v84/naesseth18a.html}
}

@InProceedings{Hirt2019,
  title = 	 {Scalable Bayesian Learning for State Space Models using Variational Inference with SMC Samplers},
  author =       {Hirt, Marcel and Dellaportas, Petros},
  booktitle = 	 {Proceedings of the Twenty-Second International Conference on Artificial Intelligence and Statistics},
  pages = 	 {76--86},
  year = 	 {2019},
  editor = 	 {Chaudhuri, Kamalika and Sugiyama, Masashi},
  volume = 	 {89},
  series = 	 {Proceedings of Machine Learning Research},
  month = 	 {16--18 Apr},
  publisher =    {PMLR},
  url = 	 {https://proceedings.mlr.press/v89/hirt19a.html}
}

@inproceedings{Krishnan2017, 
author = {Krishnan, Rahul G. and Shalit, Uri and Sontag, David}, 
title = {Structured inference networks for nonlinear state space models}, 
year = {2017}, 
publisher = {AAAI Press}, 
booktitle = {Proceedings of the Thirty-First AAAI Conference on Artificial Intelligence}, 
pages = {2101–2109}, 
numpages = {9}, 
location = {San Francisco, California, USA}, 
series = {AAAI'17} 
}

@inproceedings{KVAE,
 author = {Fraccaro, Marco and Kamronn, Simon and Paquet, Ulrich and Winther, Ole},
 booktitle = {Advances in Neural Information Processing Systems},
 editor = {I. Guyon and U. Von Luxburg and S. Bengio and H. Wallach and R. Fergus and S. Vishwanathan and R. Garnett},
 pages = {},
 publisher = {Curran Associates, Inc.},
 title = {A Disentangled Recognition and Nonlinear Dynamics Model for Unsupervised Learning},
 url = {https://proceedings.neurips.cc/paper_files/paper/2017/file/7b7a53e239400a13bd6be6c91c4f6c4e-Paper.pdf},
 volume = {30},
 year = {2017}
}

@InProceedings{Doerr2018,
  title = 	 {Probabilistic Recurrent State-Space Models},
  author =       {Doerr, Andreas and Daniel, Christian and Schiegg, Martin and Duy, Nguyen-Tuong and Schaal, Stefan and Toussaint, Marc and Sebastian, Trimpe},
  booktitle = 	 {Proceedings of the 35th International Conference on Machine Learning},
  pages = 	 {1280--1289},
  year = 	 {2018},
  editor = 	 {Dy, Jennifer and Krause, Andreas},
  volume = 	 {80},
  series = 	 {Proceedings of Machine Learning Research},
  month = 	 {10--15 Jul},
  publisher =    {PMLR},
  url = 	 {https://proceedings.mlr.press/v80/doerr18a.html}
}

@inproceedings{
DVBF,
title={Deep Variational Bayes Filters: Unsupervised Learning of State Space Models from Raw Data},
author={Maximilian Karl and Maximilian Soelch and Justin Bayer and Patrick van der Smagt},
booktitle={International Conference on Learning Representations},
year={2017},
url={https://openreview.net/forum?id=HyTqHL5xg}
}

@InProceedings{RKN,
  title = 	 {Recurrent Kalman Networks: Factorized Inference in High-Dimensional Deep Feature Spaces},
  author =       {Becker, Philipp and Pandya, Harit and Gebhardt, Gregor and Zhao, Cheng and Taylor, C. James and Neumann, Gerhard},
  booktitle = 	 {Proceedings of the 36th International Conference on Machine Learning},
  pages = 	 {544--552},
  year = 	 {2019},
  editor = 	 {Chaudhuri, Kamalika and Salakhutdinov, Ruslan},
  volume = 	 {97},
  series = 	 {Proceedings of Machine Learning Research},
  month = 	 {09--15 Jun},
  publisher =    {PMLR},
  url = 	 {https://proceedings.mlr.press/v97/becker19a.html}
}

@inproceedings{Campbell2021,
 author = {Campbell, Andrew and Shi, Yuyang and Rainforth, Thomas and Doucet, Arnaud},
 booktitle = {Advances in Neural Information Processing Systems},
 editor = {M. Ranzato and A. Beygelzimer and Y. Dauphin and P.S. Liang and J. Wortman Vaughan},
 pages = {18633--18645},
 publisher = {Curran Associates, Inc.},
 title = {Online Variational Filtering and Parameter Learning},
 url = {https://proceedings.neurips.cc/paper_files/paper/2021/file/9a6a1aaafe73c572b7374828b03a1881-Paper.pdf},
 volume = {34},
 year = {2021}
}

@article{Peyron2021,
author = {Peyron, Mathis and Fillion, Anthony and Gürol, Selime and Marchais, Victor and Gratton, Serge and Boudier, Pierre and Goret, Gael},
title = {Latent space data assimilation by using deep learning},
journal = {Quarterly Journal of the Royal Meteorological Society},
volume = {147},
number = {740},
pages = {3759-3777},
doi = {https://doi.org/10.1002/qj.4153},
url = {https://rmets.onlinelibrary.wiley.com/doi/abs/10.1002/qj.4153},
eprint = {https://rmets.onlinelibrary.wiley.com/doi/pdf/10.1002/qj.4153},
year = {2021}
}

@article {Fan2025LDA,
      author = "Hang Fan and Yubao Liu and Yuewei Liu and Zhaoyang Huo and Baojun Chen and Yu Qin",
      title = "A Novel Latent Space Data Assimilation Framework with Autoencoder-Observation to Latent Space (AE-O2L) Network. Part II: Observation and Background Assimilation with Interpretability",
      journal = "Monthly Weather Review",
      year = "2025",
      publisher = "American Meteorological Society",
      address = "Boston MA, USA",
      volume = "153",
      number = "8",
      doi = "10.1175/MWR-D-24-0058.1",
      pages=      "1349 - 1363",
      url = "https://journals.ametsoc.org/view/journals/mwre/153/8/MWR-D-24-0058.1.xml"
}

@inproceedings{SIXO,
 author = {Lawson, Dieterich and Ravent\'{o}s, Allan and Warrington, Andrew and Linderman, Scott},
 booktitle = {Advances in Neural Information Processing Systems},
 editor = {S. Koyejo and S. Mohamed and A. Agarwal and D. Belgrave and K. Cho and A. Oh},
 pages = {38844--38858},
 publisher = {Curran Associates, Inc.},
 title = {SIXO: Smoothing Inference with Twisted Objectives},
 url = {https://proceedings.neurips.cc/paper_files/paper/2022/file/fddc79681b2df2734c01444f9bc2a17e-Paper-Conference.pdf},
 volume = {35},
 year = {2022}
}

@article{Ramgraber23b,
title = {Ensemble transport smoothing. Part II: Nonlinear updates},
journal = {Journal of Computational Physics: X},
volume = {17},
pages = {100133},
year = {2023},
issn = {2590-0552},
doi = {https://doi.org/10.1016/j.jcpx.2023.100133},
url = {https://www.sciencedirect.com/science/article/pii/S2590055223000112},
author = {Maximilian Ramgraber and Ricardo Baptista and Dennis McLaughlin and Youssef Marzouk},
}

@article{Ramgraber23a,
title = {Ensemble transport smoothing. Part I: Unified framework},
journal = {Journal of Computational Physics: X},
volume = {17},
pages = {100134},
year = {2023},
issn = {2590-0552},
doi = {https://doi.org/10.1016/j.jcpx.2023.100134},
url = {https://www.sciencedirect.com/science/article/pii/S2590055223000124},
author = {Maximilian Ramgraber and Ricardo Baptista and Dennis McLaughlin and Youssef Marzouk},
}

@InProceedings{VIFLE,
  title = 	 {Variational Inference for Sequential Data with Future Likelihood Estimates},
  author =       {Kim, Geon-Hyeong and Jang, Youngsoo and Yang, Hongseok and Kim, Kee-Eung},
  booktitle = 	 {Proceedings of the 37th International Conference on Machine Learning},
  pages = 	 {5296--5305},
  year = 	 {2020},
  editor = 	 {III, Hal Daumé and Singh, Aarti},
  volume = 	 {119},
  series = 	 {Proceedings of Machine Learning Research},
  month = 	 {13--18 Jul},
  publisher =    {PMLR},
  url = 	 {https://proceedings.mlr.press/v119/kim20d.html},
}

@misc{LEVDA,
      title={{LEVDA}: Latent Ensemble Variational Data Assimilation via Differentiable Dynamics}, 
      author={Phillip Si and Peng Chen},
      year={2026},
      eprint={2602.19406},
      archivePrefix={arXiv},
      primaryClass={cs.LG},
      url={https://arxiv.org/abs/2602.19406}, 
}
